\documentclass{article}

    \PassOptionsToPackage{numbers, sort&compress}{natbib}
    \usepackage[final]{neurips_2019}

\usepackage[table]{xcolor}
\usepackage[utf8]{inputenc} 
\usepackage[T1]{fontenc}    
\usepackage{hyperref}       
\usepackage{url}            
\usepackage{booktabs}       
\usepackage{amsfonts}       
\usepackage{nicefrac}       
\usepackage{microtype}      

\usepackage{amsfonts}
\usepackage{amsmath}
\usepackage{enumitem} 
\usepackage[pdftex]{graphicx}
\usepackage{epstopdf}

\usepackage{amsthm}
\usepackage{amssymb}
\usepackage{subcaption}

\usepackage{amsthm}
\usepackage{hyperref}
\usepackage{bm}

\usepackage{multirow}
\usepackage{ctable} 
\usepackage{floatrow}
\usepackage{float}
\restylefloat{table} 
\allowdisplaybreaks 

\usepackage{adjustbox}
\usepackage{wrapfig}

   \newcommand{\myComment}[1]{}  
\usepackage{appendix}

\def\diag{\text{diag}}

\def\norm#1{\left\|#1\right\|}

\def\<{\left<} \def\>{\right>}

\def\sym{\mathrm{sym}}
\newtheorem{theorem}{Theorem}

\newtheorem{corollary}{Corollary}

\newtheorem{lemma}{Lemma}
\DeclareMathOperator*{\argmin}{arg\,min}
\DeclareMathOperator*{\argmax}{arg\,max}

 \newenvironment{talign*}
 {\csname align*\endcsname}
 {\endalign}
 
 \def\y{\mathbf y}  
\def\u{\mathbf u} \def\v{\mathbf v}

\newcommand{\abs}[1]{\left|#1\right|}
\newcommand{\bchi}{\boldsymbol{\chi}}

\def\p{p_{\mathrm{in}}}
\def\q{p_{\mathrm{out}}}

\def\diag{\mathrm{diag}}
\def\one{\mathbf{1}}

\def\to{\rightarrow}
\def\R{\mathbb{R}}
\usepackage[ruled,linesnumbered]{algorithm2e}

\newcommand{\fillcaption}[1]{ 
\textbf{Figure \arabic{figure}:} #1 
\addtocounter {figure} {1} 
} 

\newif\ifpaper
\papertrue

\title{ 
Generalized Matrix Means for Semi-Supervised Learning with Multilayer Graphs }

\author{%
  Pedro Mercado$^1$, Francesco Tudisco$^2$ and Matthias Hein$^1$  \\
  $^1$University of T\"{u}bingen, Germany \\
  $^2$Gran Sasso Science Institute, Italy
}

\begin{document}

\maketitle

\begin{abstract}
We study the task of semi-supervised learning on multilayer graphs by taking into account both labeled and unlabeled observations together with the information encoded by each individual graph layer. We propose a regularizer based on the generalized matrix mean, which is a one-parameter family of matrix means that includes the arithmetic, geometric and harmonic means as particular cases. We analyze it in expectation under a Multilayer Stochastic Block Model and verify numerically that it outperforms state of the art methods. Moreover, we introduce a matrix-free numerical scheme based on contour integral quadratures and Krylov subspace solvers that scales to large sparse multilayer graphs.
\end{abstract}

\section{Introduction}
The task of graph-based Semi-Supervised Learning (SSL) is to build a classifier that takes into account both labeled and unlabeled observations, together with the information encoded by a given graph\cite{Subramanya:2014:GSL:2765801,Chapelle:2010:SL}.
A common and successful approach is to take a suitable loss function on the labeled nodes and a regularizer which provides information encoded by the graph~\cite{Zhou:2003:LLG:2981345.2981386,Zhu:2003:SLU:3041838.3041953,Belkin:2004,Yang:2016:RSL:3045390.3045396,kipf2017semi}. Whereas this task is well studied, traditionally these methods assume that the graph is composed by interactions of one single kind, i.e.\ only one graph is available. 

For the case where multiple graphs, or equivalently, multiple layers are available, the challenge is to 
boost the classification performance by merging the information encoded in each graph. 
The arguably most popular approach for this task consists of finding some form
of convex combination of graph matrices, where more informative graphs receive a larger weight~\cite{Tsuda:2005:FPC:1093772.1181531,zhou2007spectral,Argyriou:2005,Nie:2016:PAM:3060832.3060884,Karasuyama:2013,Kato:2009,pmlr-v89-viswanathan19a,DBLP:conf/pakdd/YeA18}. %

Note that a convex combination of graph matrices can be seen as a weighted arithmetic mean of graph matrices. 
In the context of multilayer graph clustering, previous studies~\cite{Mercado:2018:powerMean,Mercado:2016:Geometric,Mercado:2019:powerMean} have shown that weighted arithmetic means are suboptimal under certain benchmark generative graph models, whereas other matrix means, such as the geometric~\cite{Mercado:2016:Geometric} and harmonic means~\cite{Mercado:2018:powerMean}, are able to discover clustering structures that the arithmetic means overlook.

In this paper we study the task of semi-supervised learning with multilayer graphs with a novel regularizer based on the power mean Laplacian. The power mean Laplacian is a one-parameter family of Laplacian matrix means that includes as special cases the arithmetic, geometric and harmonic mean of Laplacian matrices.%
We show that in expectation under a Multilayer Stochastic Block Model, our approach 
provably correctly classifies unlabeled nodes in settings where state of the art approaches fail. In particular, a limit case of our method is provably robust against noise, yielding good classification performance as long as one layer is informative and remaining layers are potentially just noise. %
We verify the analysis in expectation with extensive experiments with random graphs, showing that our approach compares favorably with state of the art methods, yielding a good classification performance on several relevant settings where state of the art approaches fail.

Moreover, our approach scales to large datasets: even though the computation of the power mean Laplacian is in general prohibitive for large graphs, we present a matrix-free numerical scheme  based on integral quadratures methods and Krylov subspace solvers which allows us to apply the power mean Laplacian regularizer to large sparse graphs.
Finally, we perform numerical experiments on real world datasets and verify that our approach is competitive to state of the art approaches.

\section{The Power Mean Laplacian}\label{sec:power-mean-laplacian}
In this section we introduce our multilayer graph regularizer based on the power mean Laplacian.
We define a multilayer graph $\mathbb{G}$ with $T$ layers as the set $\mathbb{G}=\{G^{(1)},\ldots,G^{(T)}\}$, with each graph layer defined as $G^{(t)}=(V,W^{(t)})$, where
$V=\{v_1,\ldots,v_n\}$ is the node set and
$W^{(t)}\in\R^{n\times n}_+$ is the corresponding adjacency matrix, 
which we assume symmetric and  nonnegative.
We further denote the layers' normalized Laplacians as $L^{(t)}_\sym = I - (D^{(t)})^{-1/2}W(D^{(t)})^{-1/2}$, where $D^{(t)}$ is the degree diagonal matrix with $(D^{(t)})_{ii}=\sum_{j=1}^n W^{(t)}_{ij}$.

The scalar power mean is a one-parameter family of scalar means defined as
\begin{itemize}[topsep=-3pt,leftmargin=*]\setlength\itemsep{-3pt}
 \item[] \centering $m_p(x_1,\ldots,x_T)= ( \frac{1}{T} \sum_{i=1}^T x_i^p )^{1/p}$
\end{itemize}
where $x_1,\ldots,x_T$ are nonnegative scalars and $p$ is a real parameter. Particular choices of $p$ yield specific means such as the arithmetic, geometric and harmonic means, as illustrated in Table~\ref{table:powerMeans}.

The \textbf{Power Mean Laplacian}, introduced in~\cite{Mercado:2018:powerMean}, is 
a matrix extension of the scalar power mean applied to the Laplacians of a multilayer graph and proposed as a more robust way to blend the information encoded across the layers. It is defined as
\begin{itemize}[topsep=-3pt,leftmargin=*]\setlength\itemsep{-3pt}
 \item[] \centering $L_p = \left( \frac{1}{T} \sum_{i=1}^T  (L^{(i)}_\sym) ^p \right)^{1/p}$
\end{itemize}
where $A^{1/p}$ is the unique positive definite solution of the matrix equation $X^p=A$.
For the case $p\leq 0$ a small diagonal shift $\varepsilon>0$ is added to each Laplacian, i.e.\ 
we replace $L^{(i)}_\sym$ with
$L^{(i)}_\sym+\varepsilon$, to ensure that $L_p$ is well defined as suggested in~\cite{bhagwat_subramanian_1978}. 
In what follows all the proofs hold for an arbitrary shift.  Following~\cite{Mercado:2018:powerMean}, we set $\varepsilon=\log_{10}(1+\abs{p})+10^{-6}$ for $p\leq 0$ in the numerical experiments.
\begin{table}[t]
\centering
\small
\begin{tabular}{cccccc}
name & minimum & harmonic mean & geometric mean& arithmetic mean & maximum
\\\hline 
$p$  & $p\to -\infty$ & $p=-1$ & $p\to 0$ & $p=1$ & $p\to \infty$
\\\hline 
$m_p(a,b)$ & $\min \{a,b\}$ & $2\big(\frac 1 a + \frac 1 b\big)^{-1}$ & $\sqrt{ab}$ & $(a+b)/2$ & $\max \{a,b\}$
\\
\noalign{\vskip 1pt}    
\hline
\end{tabular}
\caption{Particular cases of scalar power means}
\label{table:powerMeans}
\vspace{-5pt}
\end{table}

\section{Multilayer Semi-Supervised Learning with the Power Mean Laplacian}
In this paper we consider the following optimization problem for the task of semi-supervised learning in multilayer graphs: Given $k$ classes $r=1,\dots,k$ and membership vectors $Y^{(r)}\in\R^{n}$ 
defined by
$Y^{(r)}_i=1$ if node $v_i$ belongs to class $r$ and $Y^{(r)}_i=0$ otherwise, we let 
\begin{align}\label{local-global-optimization-problem-rephrased-power-mean-laplacian}
  f^{(r)} = \argmin_{f\in\R^n} \| f-Y^{(r)}\|^2 + \lambda f^T L_p f \, .
\end{align}
The final class assignment for an unlabeled node $v_i$ is
$y_i = \argmax \{ f^{(1)}_i, \ldots, f^{(k)}_i \}$.
Note that the solution $f$ of~\eqref{local-global-optimization-problem-rephrased-power-mean-laplacian}, for a particular class $r$, is such that $(I + \lambda L_p)f = Y^{(r)}$.
Equation~\eqref{local-global-optimization-problem-rephrased-power-mean-laplacian} has
two terms: the first term is a loss function based on the labeled nodes whereas the second term is a regularization term based on the power mean Laplacian $L_p$, which accounts for the multilayer graph structure.
It is worth noting that the Local-Global approach of~\cite{Zhou:2003:LLG:2981345.2981386} 
is a particular case of our approach when only one layer ($T=1$) is considered. Moreover, not that when $p=1$ we obtain a regularizer term based on the arithmetic mean of Laplacians $L_1=\frac{1}{T}\sum_{i=1}^T L_\sym^{(i)}$.
In the following section we analyze our proposed approach~\eqref{local-global-optimization-problem-rephrased-power-mean-laplacian} under the Multilayer Stochastic Block Model.

\section{Multilayer Stochastic Block Model}\label{Section:SBM}
In this section we provide an analysis of semi-supervised learning for multilayer graphs with the power mean Laplacian as a regularizer under the Multilayer Stochastic Block Model (\textbf{MSBM}).
The MSBM is a generative model for graphs showing certain prescribed clusters/classes structures via a set of membership parameters $\p^{(t)}$ and $\q^{(t)}$, $t=1,\dots,T$. These parameters designate the edge probabilities: given nodes $v_i$ and $v_j$ the probability of observing an edge between them on layer $t$ is $\p^{(t)}$ (resp. $\q^{(t)}$), if $v_i$ and $v_j$ belong to the same (resp. different) cluster/class. %
Note that, unlike the Labeled Stochastic Block Model~\cite{heimlicher2012community}, the MSBM allows multiple edges between the same pairs of nodes across the layers.
For SSL with one layer under the SBM we refer the reader to~\cite{Saade_2018,Kanade:2016,Mossel:2016:LAB:2840728.2840749}.

We present an analysis in expectation. We consider $k$ clusters/classes $\mathcal{C}_1,\ldots,\mathcal{C}_k$ of equal size $\abs{\mathcal{C}}=n/k$. 
%
%
We denote with calligraphic letters the layers of a multilayer graph in expectation $E(\mathbb{G})=\{E(G^{(1)},\ldots,E(G^{(T)})\}$,
%
%
i.e. $\mathcal{W}^{(t)}$ is the expected adjacency matrix of the $t^{th}$-layer. We assume that our multilayer graphs are non-weighted, i.e. edges are zero or one, and hence we have $\mathcal{W}^{(t)}_{ij}=\p^{(t)}$, (resp.\ $\mathcal{W}^{(t)}_{ij}=\q^{(t)}$) for nodes $v_i,v_j$ belonging to the same (resp. different) cluster/class.

In order to grasp how different methods classify the nodes  in  multilayer  graphs  following  the  MSBM  we analyze two different settings.  In the first setting (Section \ref{sec:complementaryInformation}) all layers have the same class structure and we study the 
conditions for different regularizers $L_p$ to correctly predict class labels. We further show that our approach is robust against the presence of noise layers, in the sense that it achieves a small classification error when at least one layer is informative and the remaining layers are potentially just noise. In this setting we distinguish the case where each class has the same amount of initial labels  and the case where different classes have different number of labels.   In  the  second  setting (Section \ref{sec:independentInformation})  we consider  the case where each layer taken alone would lead to a large classification error whereas considering all the layers together can lead to a small classification error. 

\subsection{Complementary Information Layers}\label{sec:complementaryInformation}
A common assumption in multilayer semi-supervised learning is that at least one layer encodes relevant information in the label prediction task. 
The next theorem  
discusses the classification error of the expected  power mean Laplacian regularizer in this setting.

\begin{theorem}\label{theorem:generalization-equallySized-equallyLabelled}
 Let $E(\mathbb{G})$ be the expected multilayer graph with $T$ layers following the multilayer SBM with
 $k$ 
 classes 
 $\mathcal{C}_1, \ldots, \mathcal{C}_k$ of equal size and parameters $\left(\p^{(t)},\q^{(t)}\right)_{t=1}^{T}$. 
 Assume the same number of labeled nodes are available per class.
 Then, the solution of \eqref{local-global-optimization-problem-rephrased-power-mean-laplacian} yields zero test error if and only if 
\begin{align}\label{eq:condition_p}
  m_p(\boldsymbol{\rho_\epsilon}) < 1+\epsilon\, ,
\end{align}
 where $(\boldsymbol{\rho_\epsilon})_t = 1-(\p^{(t)} - \q^{(t)})/(\p^{(t)} + (k-1)\q^{(t)})+\epsilon$, and $t=1,\ldots,T$.
 \end{theorem}
This theorem shows that  the power mean Laplacian regularizer allows to correctly classify the nodes if $p$ is such that condition \eqref{eq:condition_p} holds. In order to better understand how this condition changes when $p$ varies, we analyze in the next corollary the limit cases $p\to\pm \infty$.
 \begin{corollary}\label{corollary:limit_cases}
   Let $E(\mathbb{G})$ be an expected multilayer graph as in Theorem~\ref{theorem:generalization-equallySized-equallyLabelled}.
 Then,
 \begin{itemize}[topsep=-3pt,leftmargin=*]\setlength\itemsep{-3pt}
 \item For $p\to\infty$, the test error is zero if and only if $\q^{(t)}<\p^{(t)}$ for \textit{all} $t=1,\ldots,T$.
 \item For $p\!\to\! -\infty$,  the test error is zero if and only there exists a $t\!\in\!\{1,\ldots,T\}$ such that $\q^{(t)}<\p^{(t)}$.
\end{itemize}
\end{corollary}
This corollary implies that the limit case $p\to\infty$ requires that \textit{all layers} convey information regarding the clustering/class structure of the multilayer graph, whereas the case $p\to-\infty$ requires that \textit{at least one layer} encodes clustering/class information,
and hence it is clear that conditions for the limit $p\to-\infty$ are less restrictive than the conditions for the limit case $p\to \infty$.
The next Corollary shows that the smaller the power parameter $p$ is, the less restrictive are the conditions to yield a zero test error. 
\begin{corollary}\label{corollary:contention}
 Let $E(\mathbb{G})$ be an expected multilayer graph as in Theorem~\ref{theorem:generalization-equallySized-equallyLabelled}.
 Let $p\leq q$. If $\mathcal{L}_q$ yields zero test error, then $\mathcal{L}_p$ yields a zero test error.
\end{corollary}
The previous results show the effectivity of  the power mean Laplacian regularizer in expectation. 
We now present a numerical evaluation  based on Theorem~\ref{theorem:generalization-equallySized-equallyLabelled} and Corollaries~\ref{corollary:limit_cases} and~\ref{corollary:contention} on random graphs sampled from the SBM. 
The corresponding results are presented in Fig.~\ref{figure:SBM:2layers} for classification with regularizers $L_{-10},L_{-1},L_{0},L_{1},L_{10}$ and $\lambda=1$.
\begin{figure}[t]
 \centering
 \textbf{Classification Error\\}
\includegraphics[width=0.3\linewidth, clip,trim=200 50 150 480]{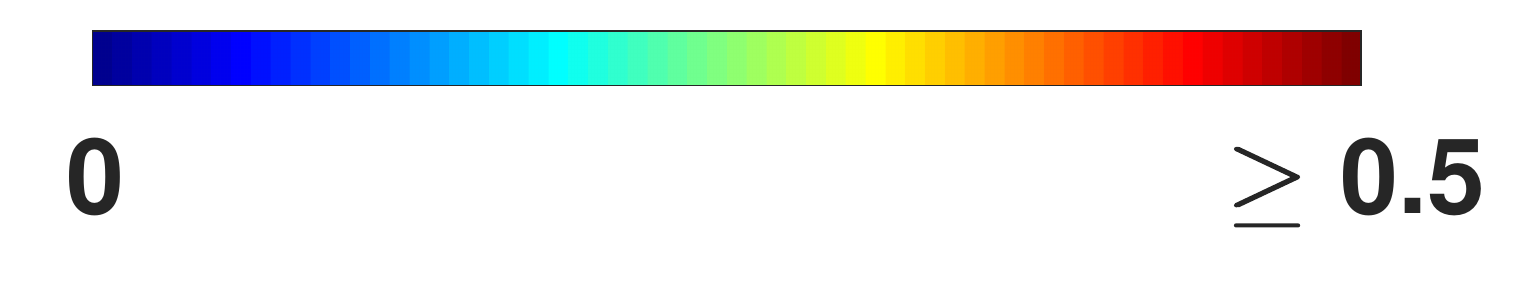}
\vspace{-5pt}
\\
 \hfill
 \begin{subfigure}[]{0.18\linewidth}
 \includegraphics[width=1\linewidth, clip,trim=120 30 170 35]{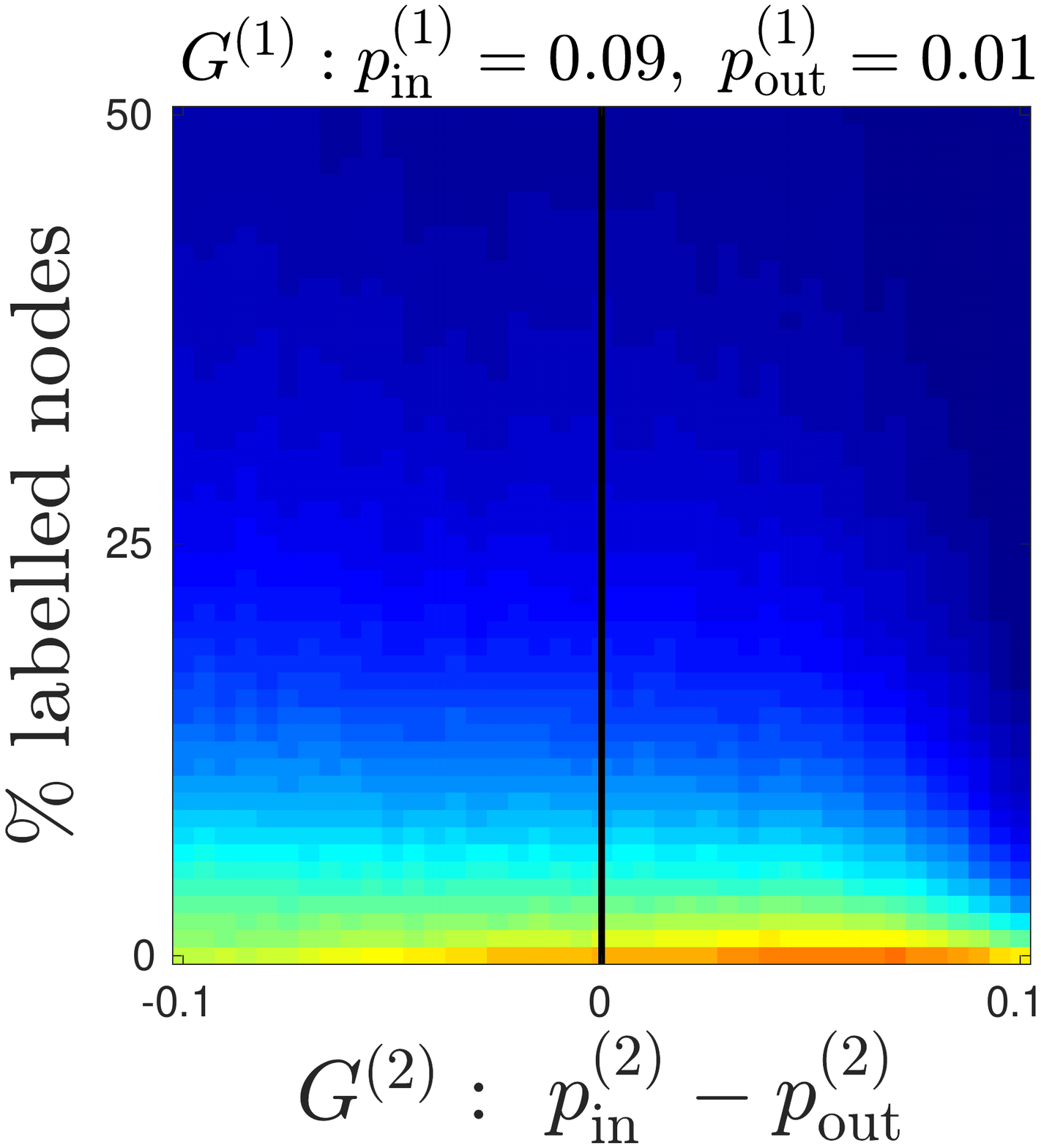}
 \caption{$L_{-10}$}
 \label{figure:SBM:2layers:p=-10}
 \end{subfigure}
 \hfill 
 \begin{subfigure}[]{0.18\linewidth}
 \includegraphics[width=1\linewidth, clip,trim=120 30 170 35]{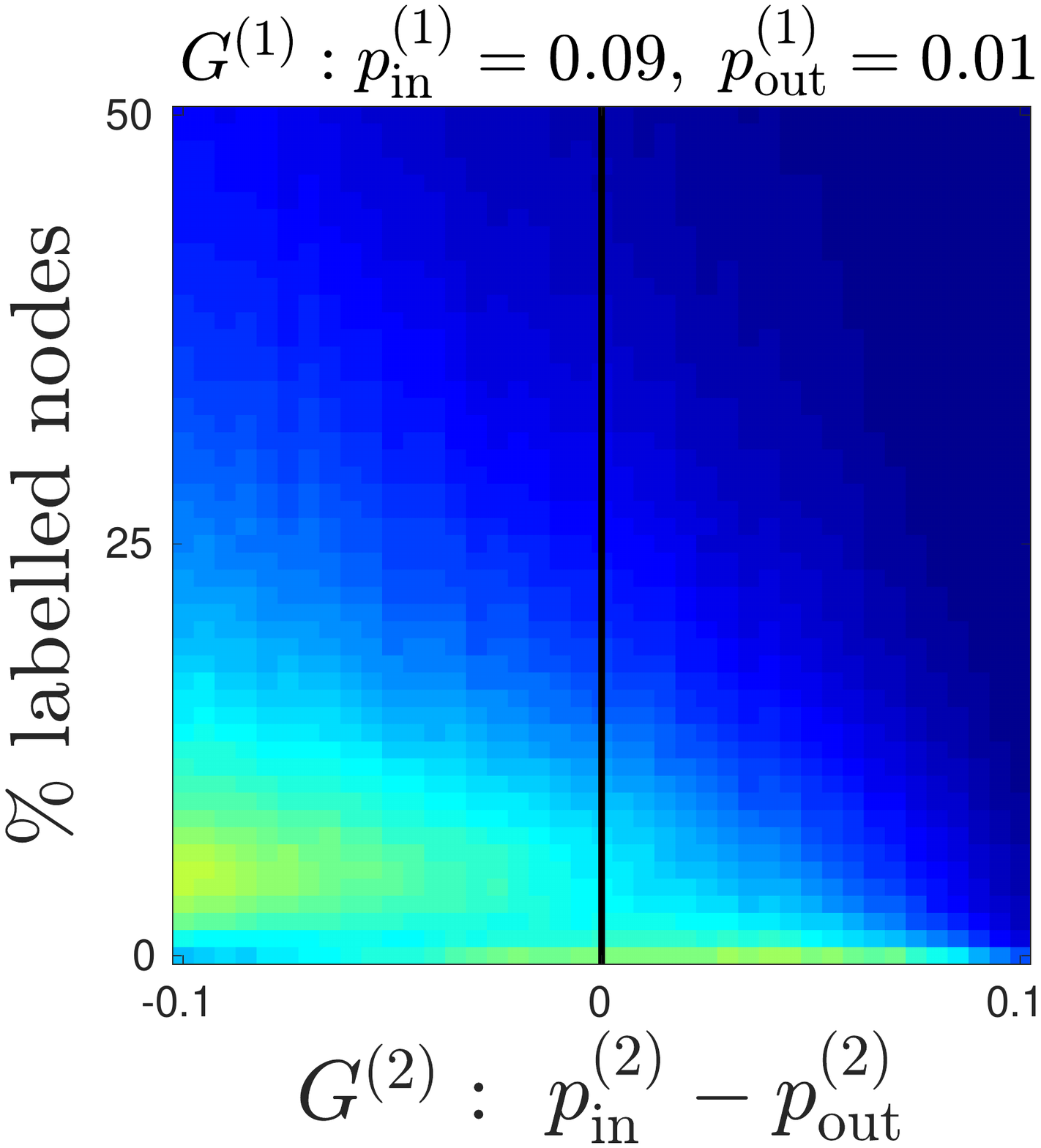}
 \caption{$L_{-1}$}
 \label{figure:SBM:2layers:p=-1}
 \end{subfigure}
 \hfill 
 \begin{subfigure}[]{0.18\linewidth}
 \includegraphics[width=1\linewidth, clip,trim=120 30 170 35]{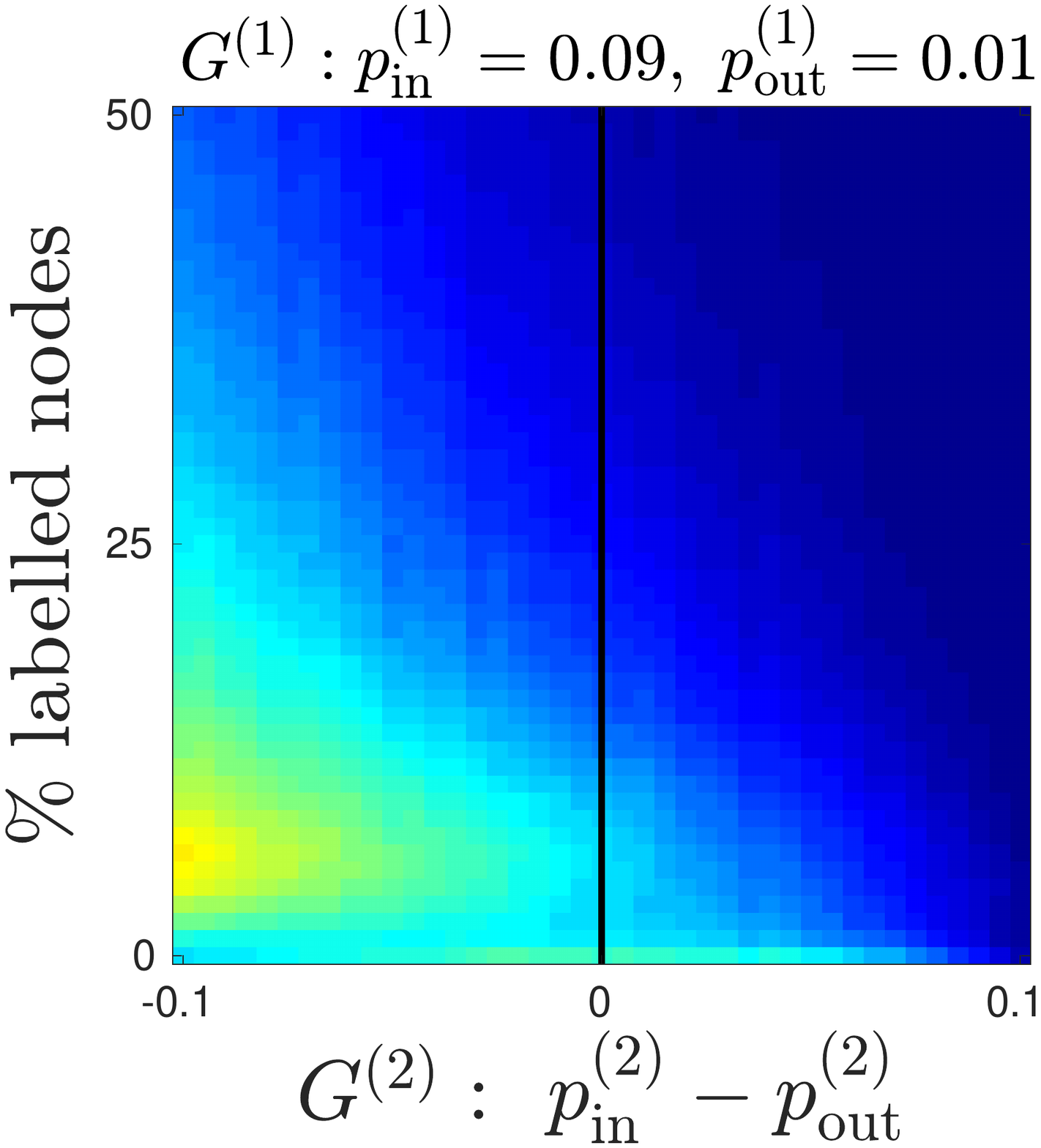}
 \caption{$L_{0}$}
 \label{figure:SBM:2layers:p=0}
 \end{subfigure}
 \hfill 
 \begin{subfigure}[]{0.18\linewidth}
 \includegraphics[width=1\linewidth, clip,trim=120 30 170 35]{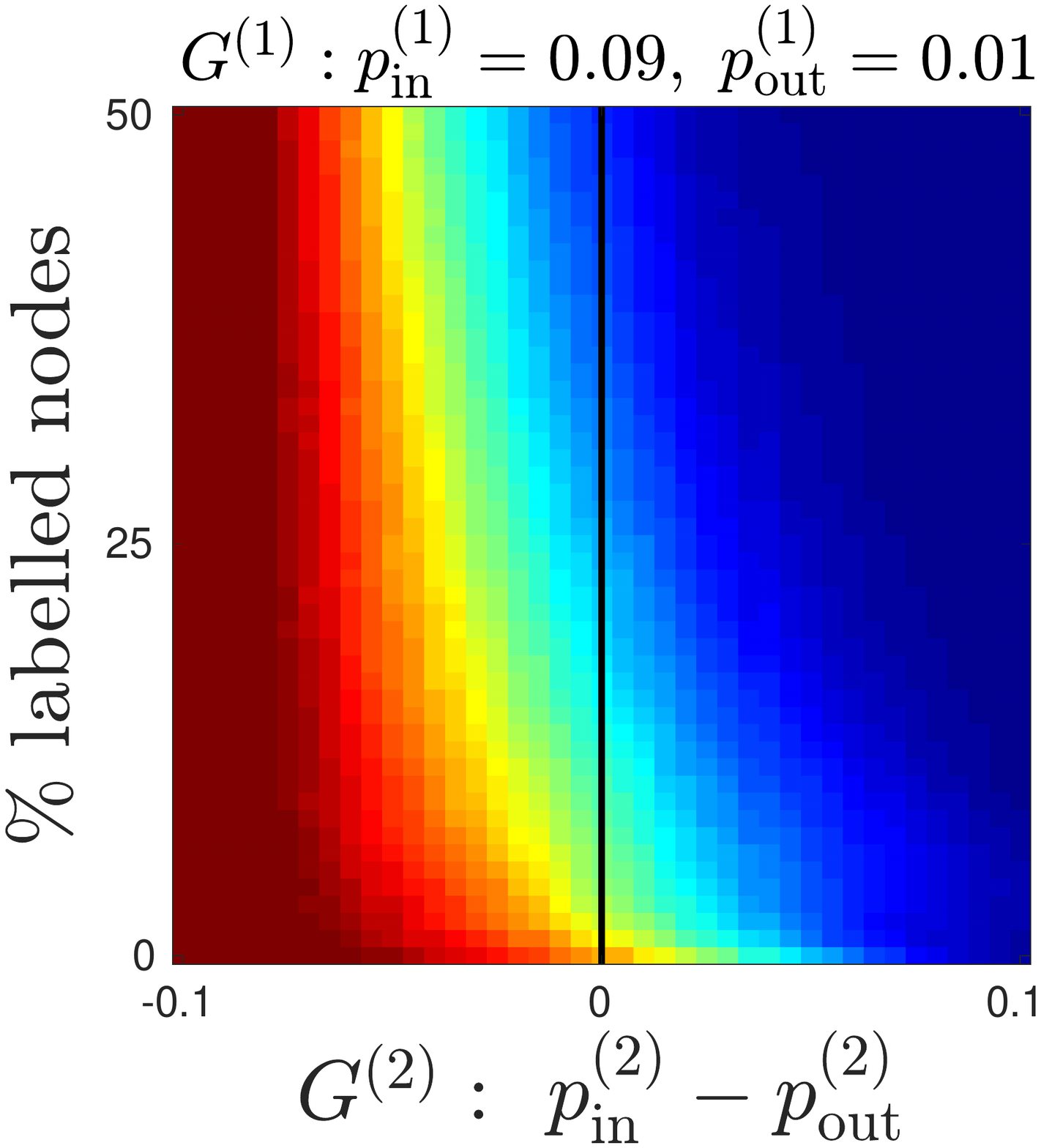}
 \caption{$L_{1}$}
 \label{figure:SBM:2layers:p=1}
 \end{subfigure}
 \hfill 
 \begin{subfigure}[]{0.18\linewidth}
 \includegraphics[width=1\linewidth, clip,trim=120 30 170 35]{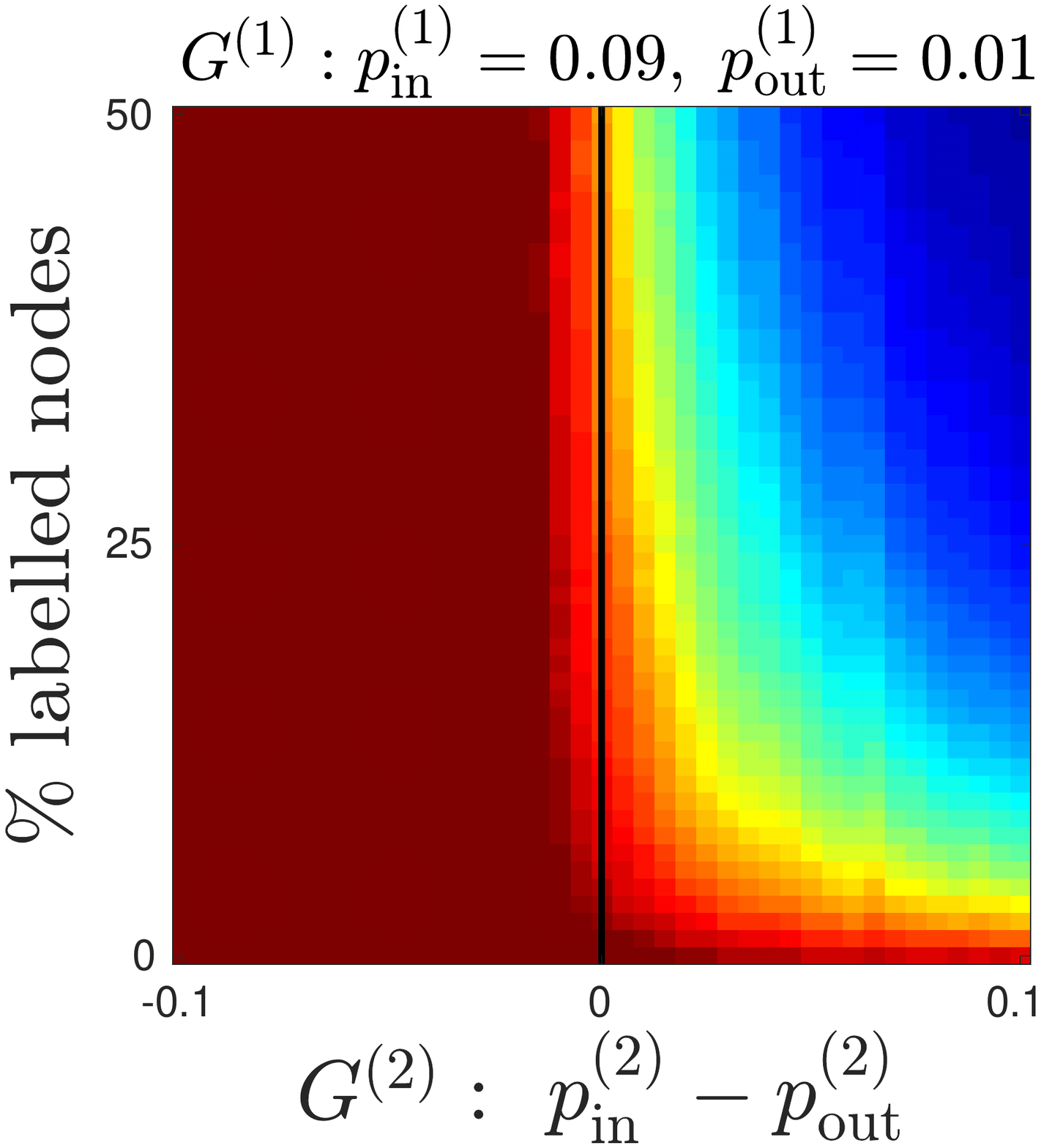}
 \caption{$L_{10}$}
 \label{figure:SBM:2layers:p=10}
 \end{subfigure}
 \hfill

\hfill
 \begin{subfigure}[]{0.18\linewidth}
 \includegraphics[width=1\linewidth, clip,trim=120 30 170 35]{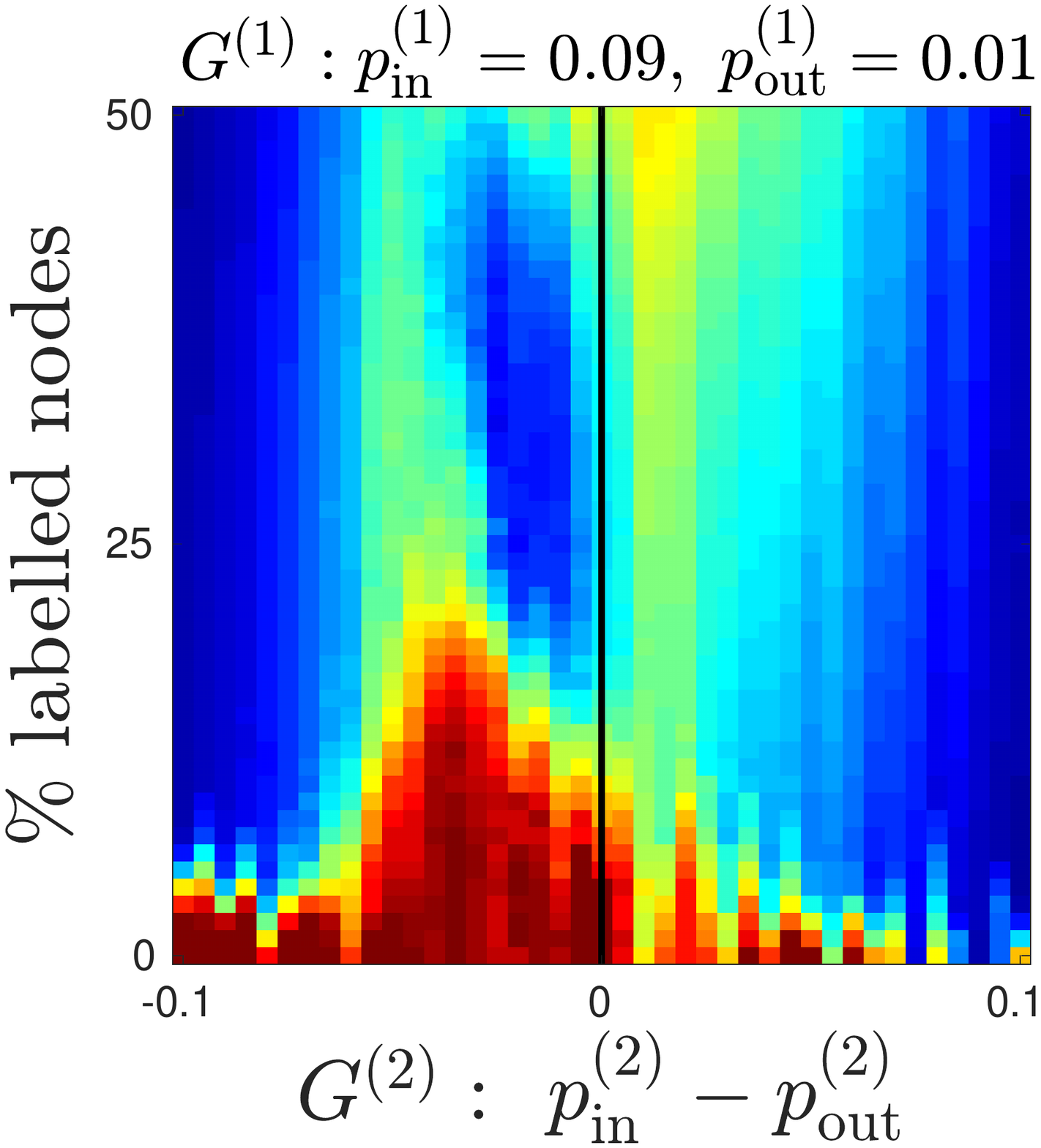}
 \caption{SMACD}
 \end{subfigure}
 \hfill 
 \begin{subfigure}[]{0.18\linewidth}
\includegraphics[width=1\linewidth, clip,trim=120 30 170 35]{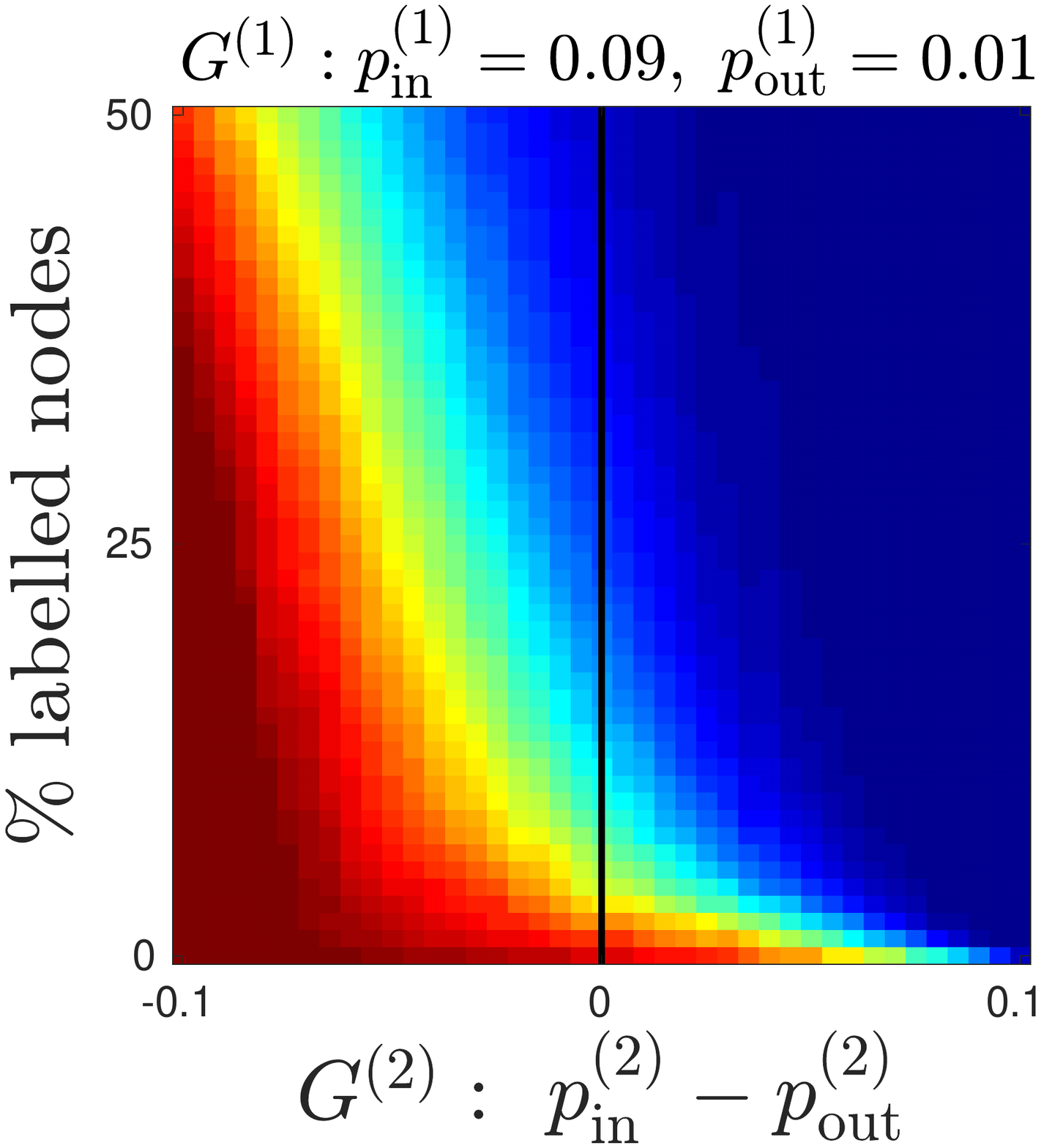}
 \caption{AGML}
 \end{subfigure}
 \hfill 
 \begin{subfigure}[]{0.18\linewidth}
 \includegraphics[width=1\linewidth, clip,trim=120 30 170 35]{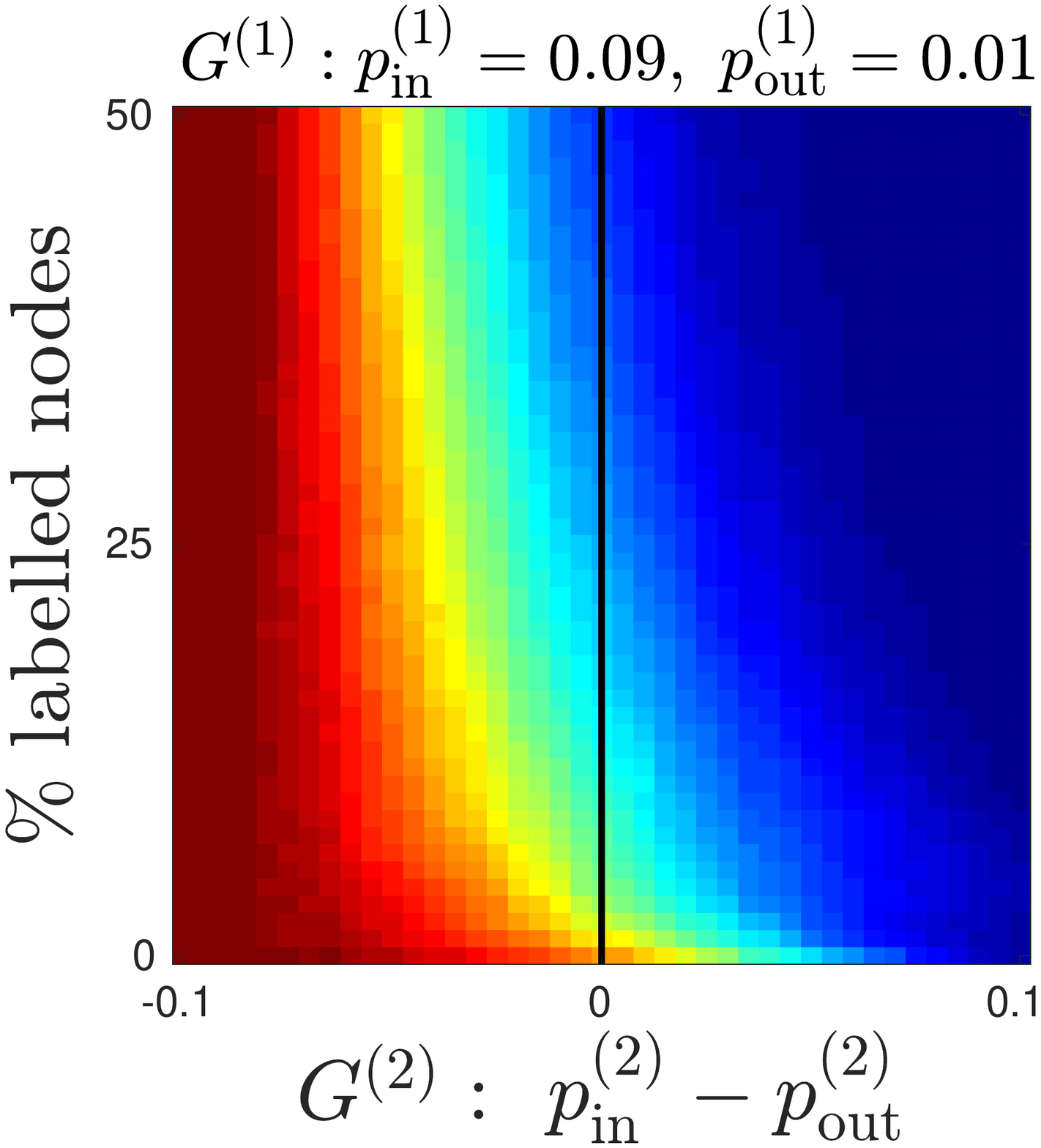}
 \caption{TLMV}
 \end{subfigure}
 \hfill 
 \begin{subfigure}[]{0.18\linewidth}
 \includegraphics[width=1\linewidth, clip,trim=120 30 170 35]{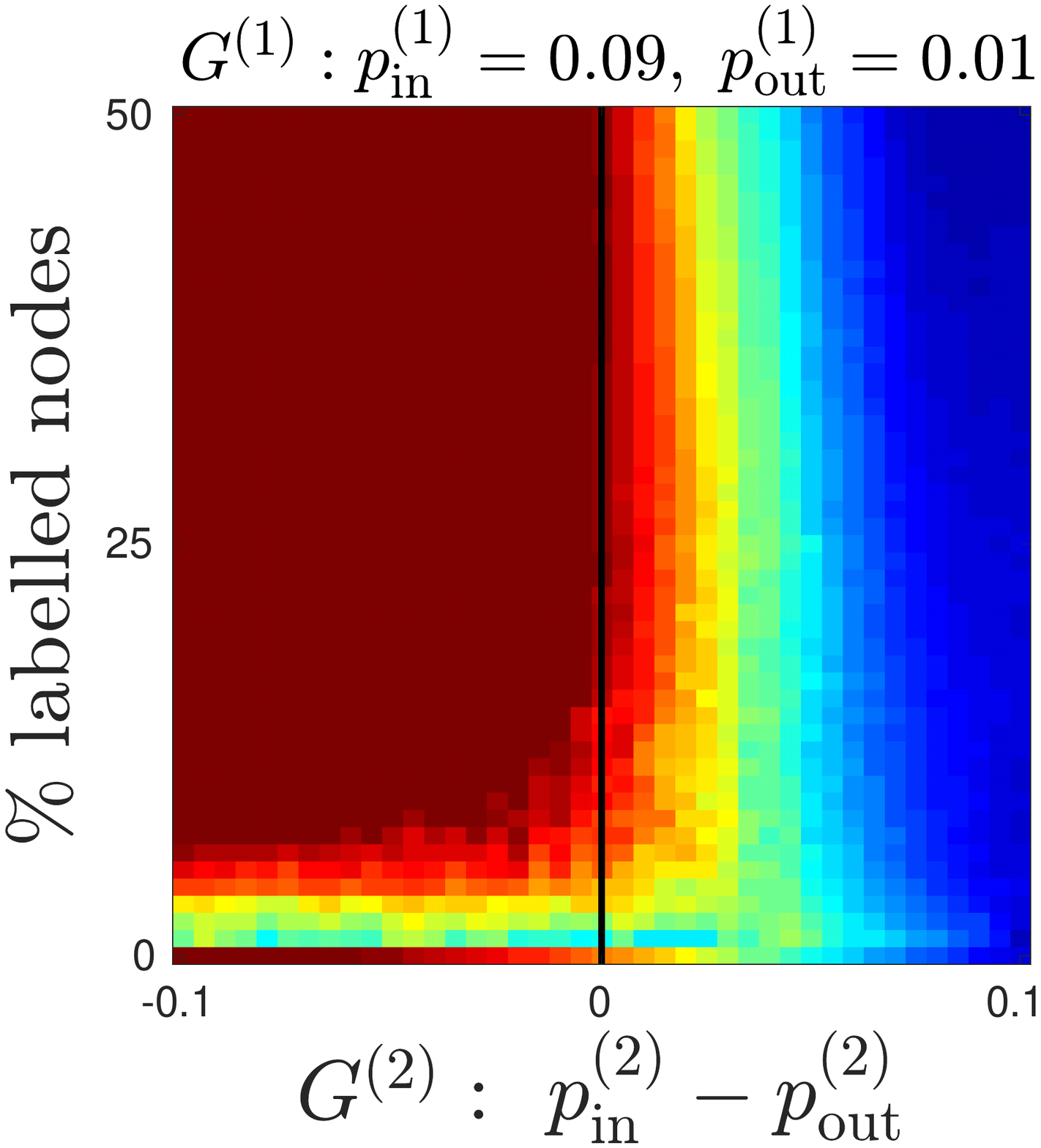}
 \caption{SGMI}
 \end{subfigure}
 \hfill 
 \begin{subfigure}[]{0.18\linewidth}
 \includegraphics[width=1\linewidth, clip,trim=120 30 170 35]{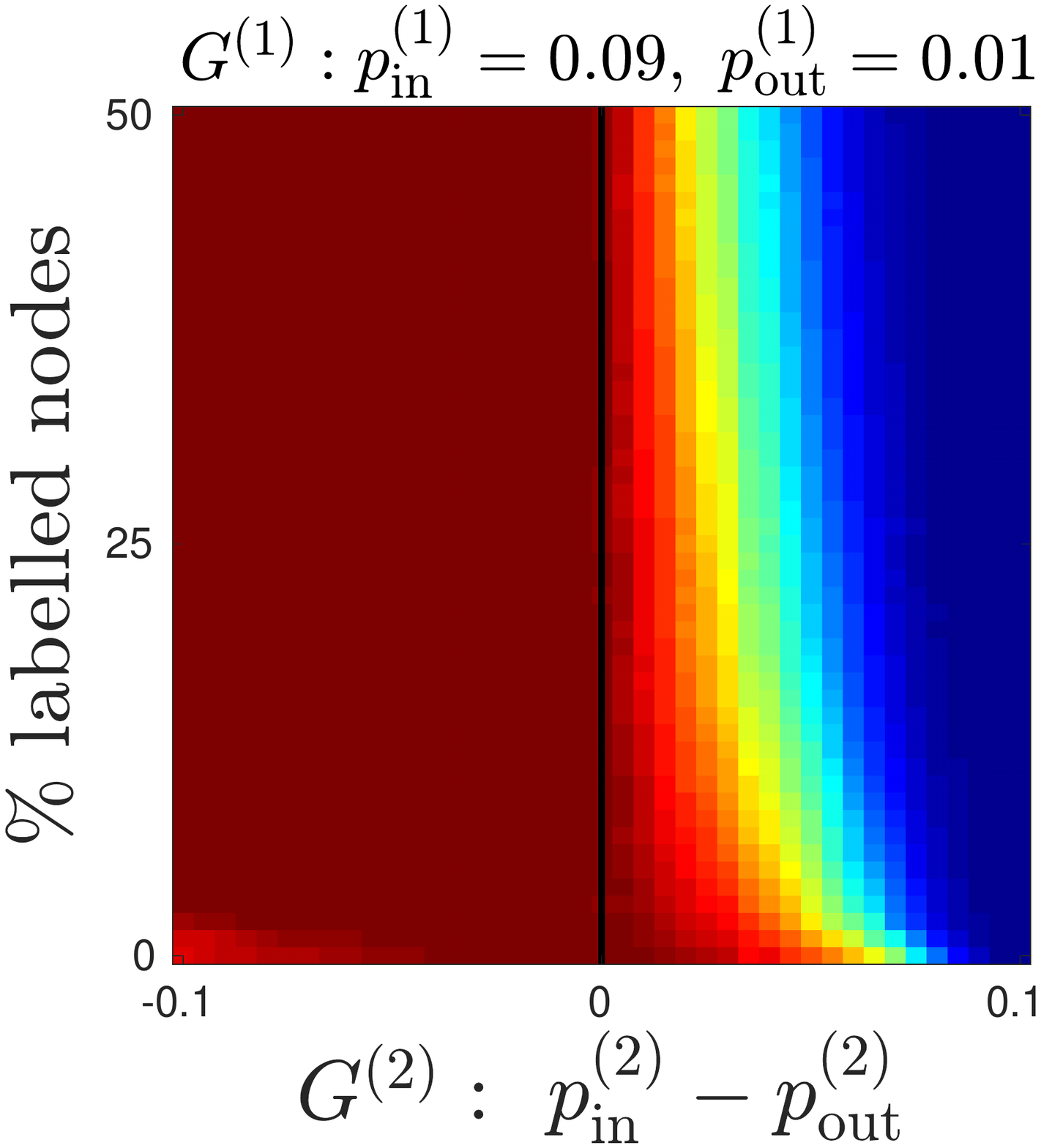}
 \caption{TSS}
 \end{subfigure}
 \hfill
 \caption{Average classification error under the Stochastic Block Model computed from 100 runs. 
 \textbf{Top Row:} Particular cases with the power mean Laplacian.
 \textbf{Bottom Row:} State of the art models.
 \vspace{-20pt}
 }
 \label{figure:SBM:2layers}
\end{figure}
We first describe the setting we consider:
we generate random multilayer graphs with two layers ($T=2$) and two classes ($k=2$) each composed by 100 nodes ($\abs{\mathcal{C}}=100$). For each parameter configuration $(\p^{(1)},\q^{(1)},\p^{(2)},\q^{(2)})$ we generate 10 random multilayer graphs and 10 random samples of labeled nodes, yielding a total of 100 runs per parameter configuration, and report the average test error.
Our goal is to evaluate the classification performance under different SBM parameters and different amounts of labeled nodes. To this end, we fix the first layer $G^{(1)}$ to be informative of the class structure ($\p^{(1)}-\q^{(1)}=0.08$), i.e. one can achieve a low classification error by taking this layer alone, provided sufficiently many labeled nodes are given.
The second layer will go from non-informative (noisy) configurations ($\p^{(2)}<\q^{(2)}$, left half of $x$-axis) to informative configurations~($\p^{(2)}>\q^{(2)}$, right half of $x$-axis),
with $\p^{(t)}+\q^{(t)}=0.1$ for both layers.
Moreover, we consider different amounts of labeled nodes: going from $1\%$ to $50\%$ ($y$-axis). The corresponding results are presented in Figs.~\ref{figure:SBM:2layers:p=-10},\ref{figure:SBM:2layers:p=-1},\ref{figure:SBM:2layers:p=0},\ref{figure:SBM:2layers:p=1}, and~\ref{figure:SBM:2layers:p=10}.

In general one can expect a low classification error when both layers $G^{(1)}$ and $G^{(2)}$ are informative (right half of $x$-axis). We can see that this is the case for all power mean Laplacian regularizers here considered (see top row of Fig.~\ref{figure:SBM:2layers}). In particular, we can see in Fig.~\ref{figure:SBM:2layers:p=10} that $L_{10}$ performs well only when \textbf{both} layers are informative and completely fails when the second layer is not informative, regardless of the amount of labeled nodes. 
On the other side we can see in Fig.~\ref{figure:SBM:2layers:p=-10} that $L_{-10}$ achieves in general a low classification error, regardless of the configuration of the second layer $G^{(2)}$,
i.e. when $G^{(1)}$ \textbf{or} $G^{(2)}$ are informative. Moreover, we can see that overall the areas with low classification error (dark blue) increase when the parameter $p$ decreases, verifying the result from Corollary~\ref{corollary:contention}.
In the bottom row of Fig.~\ref{figure:SBM:2layers} we present the performance of state of the art methods. We can observe that most of them present a classification performance that resembles the one of the power mean Laplacian regularizer $L_1$. In general their classification performance 
drops when the level of noise increases, i.e. for non-informative configurations of the second layer $G^{(2)}$, and they are outperformed by the power mean Laplacian regularizer for small values of $p$.
\begin{figure*}[t]%
%
  \hfill
 \begin{subfigure}[]{0.23\linewidth}
 \includegraphics[width=1\linewidth, clip,trim=120 30 170 35]{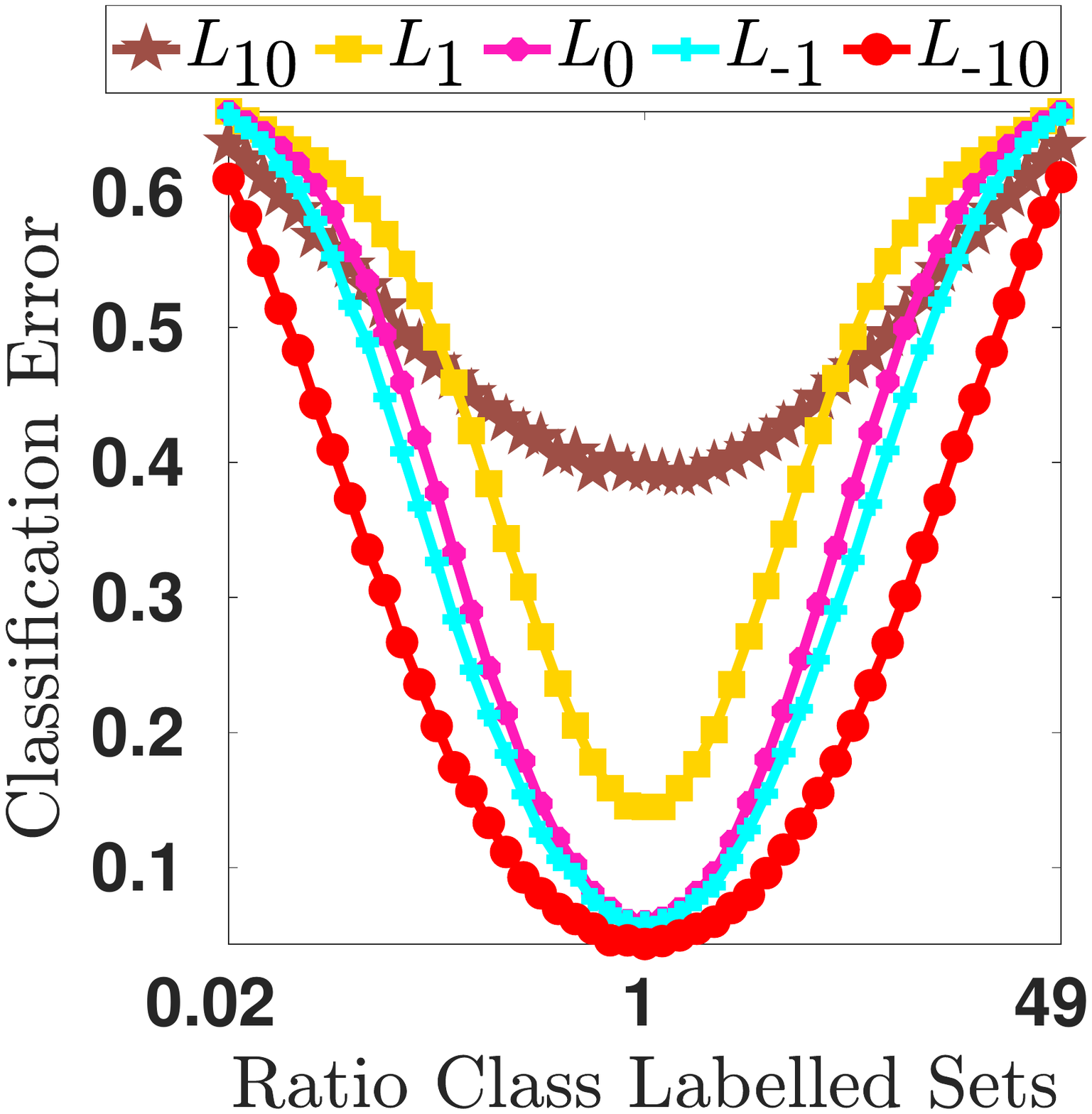}
 \caption{Uniform loss}
 \label{fig:lossFunctions:homogeneous}
 \end{subfigure}
\hfill
 \begin{subfigure}[]{0.23\linewidth}
 \includegraphics[width=1\linewidth, clip,trim=120 30 170 35]{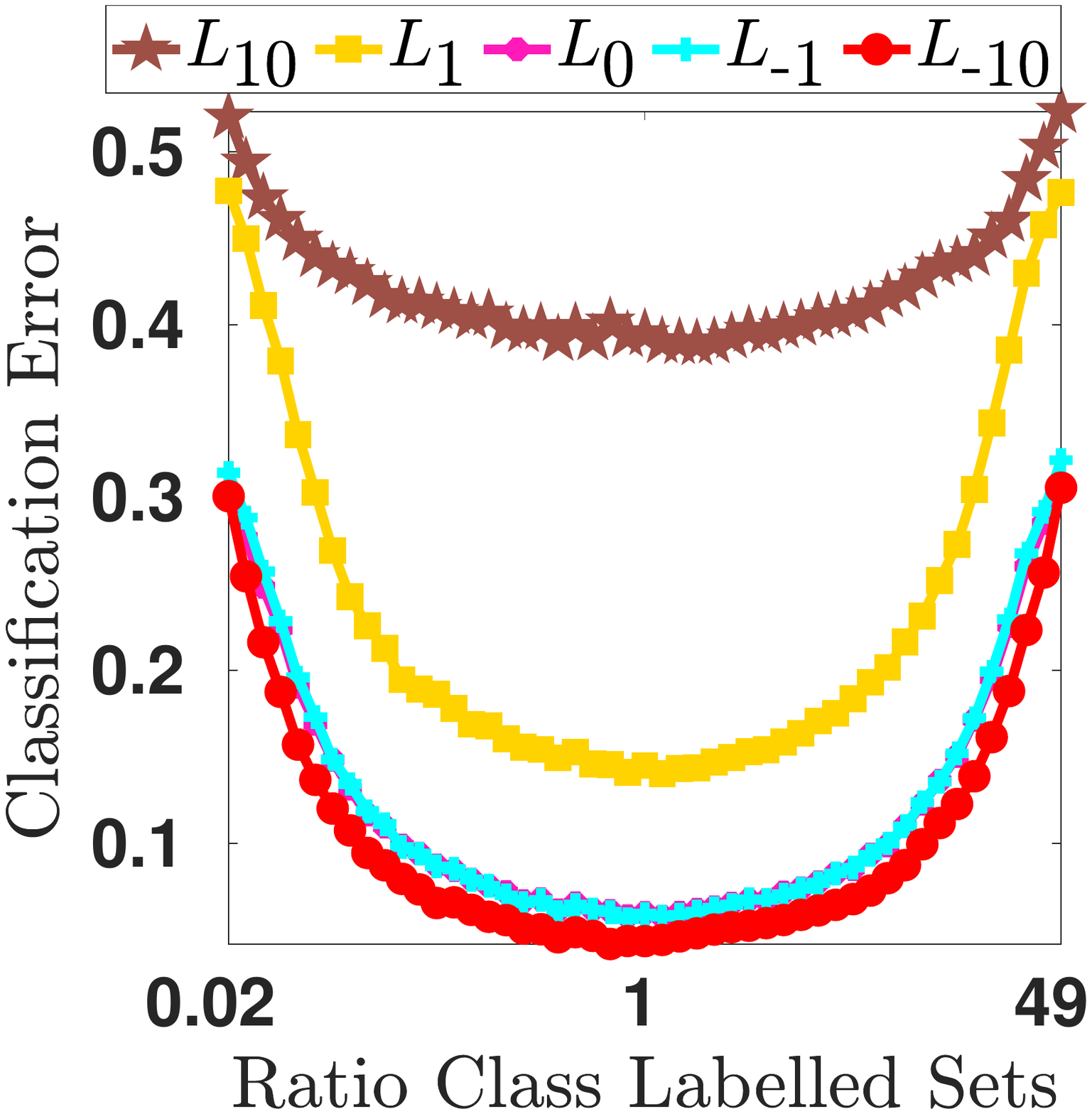}
 \caption{Weighted loss}
 \label{fig:lossFunctions:weighted}
 \end{subfigure}
 \hfill 
 \begin{subfigure}[]{0.23\linewidth}
 \includegraphics[width=1\linewidth, clip,trim=120 30 170 35]{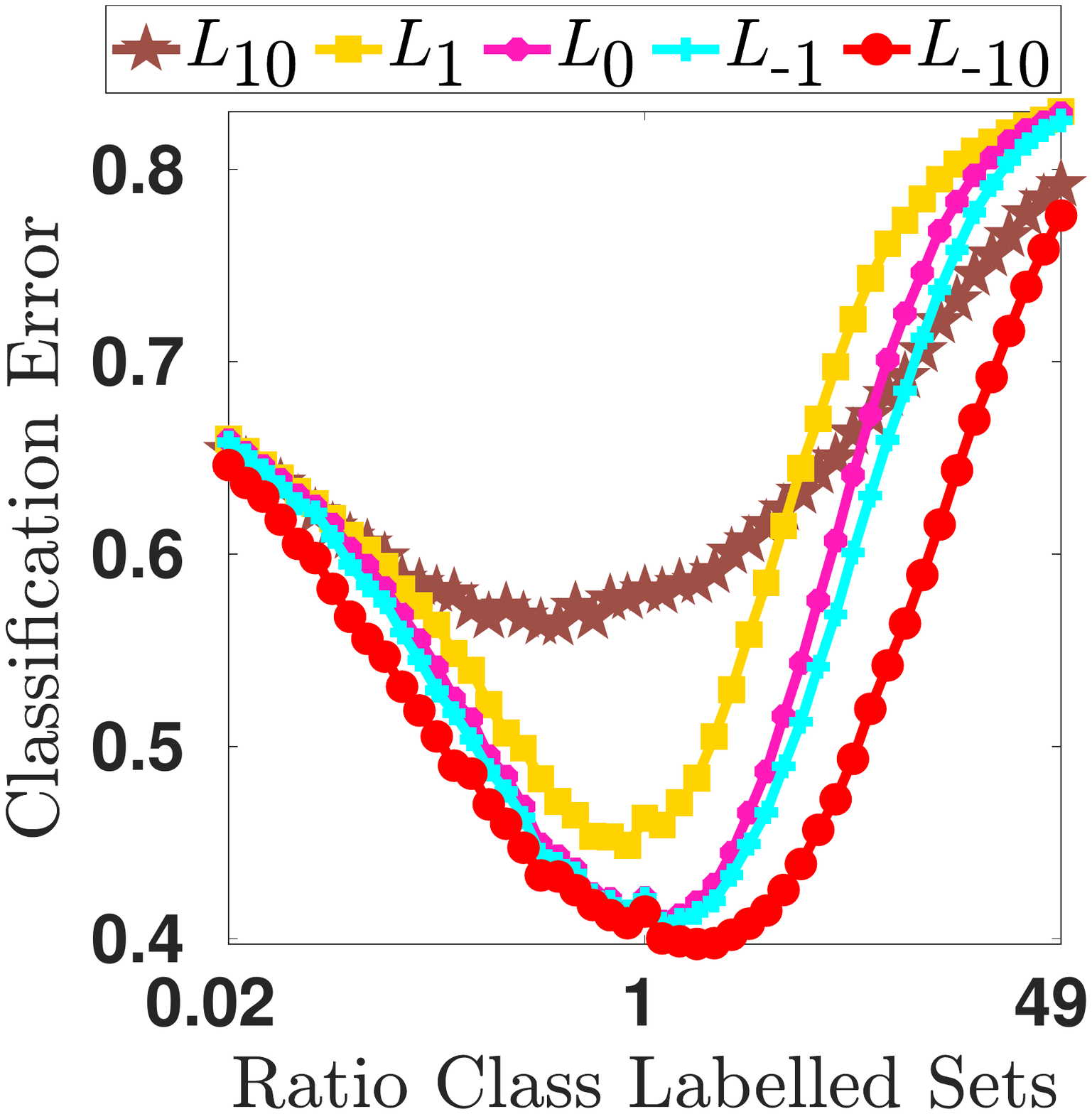}
 \caption{CMN}
 \label{fig:lossFunctions:CMN}
 \end{subfigure}
  \hfill 
  \hfill
  \hfill
 \begin{subfigure}[]{0.23\linewidth}
\fillcaption{Different class weighted loss strategies. 
Left to right: uniform loss, weighted loss, and Class Mass Normalization.
} 
\vspace{10pt}
 \end{subfigure}
 \hfill
\label{fig:lossFunctions}
\vspace{-15pt}
\end{figure*}

\textbf{Unbalanced Class Proportion on Labeled Datasets.}
In the previous analysis we assumed that we had the same amount of labeled nodes per class. 
We consider now the case where the number of labeled nodes per class is different. This setting was considered in~\cite{Zhu:2003:SLU:3041838.3041953}, where the goal was to overcome unbalanced class proportions in labeled nodes. To this end, they propose a Class Mass Normalization (CMN) strategy, whose performance was also tested in~\cite{zhu2002learning}.
In the following result we show that, provided the ground truth classes have the same size, different amounts of labeled nodes per class affect the conditions in expectation for zero classification error of \eqref{local-global-optimization-problem-rephrased-power-mean-laplacian}. For simplicity, we consider here only the case of two classes.
\begin{theorem}\label{theorem:unequalLabels:epsilon-goes-to-zero}
 Let $E(\mathbb{G})$ be the expected multilayer graph with $T$ layers following the multilayer SBM with
 two classes $\mathcal{C}_1,\mathcal{C}_2$ of equal size and parameters $\left(\p^{(t)},\q^{(t)}\right)_{t=1}^{T}$. 
 Assume $n_1,n_2$ nodes from $\mathcal{C}_1,\mathcal{C}_2$ are labeled, respectively. 
 Let $\lambda = 1$.
 Then \eqref{local-global-optimization-problem-rephrased-power-mean-laplacian} yields zero test error if
\begin{align}
 m_p(\boldsymbol{\rho_\epsilon})< \min\left\{ \frac{n_1}{n_2}, \frac{n_2}{n_1} \right\}
\end{align}
 where $(\boldsymbol{\rho_\epsilon})_t = 1-(\p^{(t)} - \q^{(t)})/(\p^{(t)} + (k-1)\q^{(t)})
 +\epsilon
 $, and $t=1,\ldots,T$.
\end{theorem}
Observe that Theorem~\ref{theorem:unequalLabels:epsilon-goes-to-zero} provides only a sufficient condition. A  necessary and sufficient condition for zero test error in terms of $p$, $n_1$ and $n_2$ is given in the supplementary material.   

A different objective function can be employed for the case of classes with different number of labels per class. Let $C$ be the diagonal matrix  defined by $C_{ii} = n/n_r$, if node $v_i$ has been labeled to belong to class $\mathcal{C}_r$.  Consider the following modification of \eqref{local-global-optimization-problem-rephrased-power-mean-laplacian}
\begin{align}\label{eq:objFun:weightedLoss}
 \argmin_{f\in\R^n} \norm{f-CY}^2 
 +  
 \lambda
 f^T L_pf
\end{align}
The next Theorem shows that using \eqref{eq:objFun:weightedLoss} in place of \eqref{local-global-optimization-problem-rephrased-power-mean-laplacian} allows us to retrieve the same condition of Theorem \ref{theorem:generalization-equallySized-equallyLabelled} for zero test error in expectation in the setting where the number of labeled nodes per class are not equal.
\begin{theorem}\label{theorem:generalization-equallySized-NOTequallyLabelled}
 Let $E(\mathbb{G})$ be the expected multilayer graph with $T$ layers following the multilayer SBM 
 $k$ classes $\mathcal{C}_1, \ldots, \mathcal{C}_k$ of equal size and parameters $\left(\p^{(t)},\q^{(t)}\right)_{t=1}^{T}$. 
 Let $n_1,\ldots,n_k$ be the number of labeled nodes per class.
 Let $C\in\R^{n\times n}$ be a diagonal matrix with $C_{ii} = n/n_r$ for $v_i\in\mathcal{C}_r$.
 The solution to \eqref{eq:objFun:weightedLoss} yields a zero test classification error if and only if 
\begin{align}
  m_p(\boldsymbol{\rho_\epsilon}) < 1+\epsilon\, ,
\end{align}
 where $(\boldsymbol{\rho_\epsilon})_t = 1-(\p^{(t)} - \q^{(t)})/(\p^{(t)} + (k-1)\q^{(t)})+\epsilon$, and $t=1,\ldots,T$.
\end{theorem}
In Figs.~\ref{fig:lossFunctions:homogeneous},~\ref{fig:lossFunctions:weighted}, and~\ref{fig:lossFunctions:CMN}.
 we present a numerical experiment with random graphs of our analysis in expectation. 
We consider the following setting: we generate multilayer graphs with two layers ($T=2$) and two classes ($k=2$) each composed by 100 nodes ($\abs{\mathcal{C}}=100$). 
We fix $\p^{(1)}-\q^{(1)}=0.08$ and $\p^{(2)}-\q^{(2)}=0$, with $\p^{(t)}+\q^{(t)}=0.1$ for both layers. We fix the total amount of labeled nodes to be $n_1+n_2 =50$ and let $n_1,n_2=1,\ldots 49$.
For each setting we generate 10 multilayer graphs and 10 sets of labeled nodes, yielding a total of 100 runs per setting, and report the average test classification error.
In Fig.~\ref{fig:lossFunctions:homogeneous} we can see the performance of the power mean Laplacian regularizer without modifications. We can observe how different proportions of labeled nodes per class affect the performance. In Fig.~\ref{fig:lossFunctions:weighted}, we present the performance of the modified approach~\eqref{eq:objFun:weightedLoss} and observe that it yields a better performance against different class label proportions. Finally in Fig.~\ref{fig:lossFunctions:CMN} we present the performance based on Class Mass Normalization
\footnote{We follow the authors' implementation: http://pages.cs.wisc.edu/\textasciitilde jerryzhu/pub/harmonic\_function.m}, where we can see that its effect is slightly skewed to one class and its overall performance is larger than the proposed approach.

\begin{figure}[t]
 \centering
 \textbf{Classification Error\\}
\includegraphics[width=0.3\linewidth, clip,trim=200 50 150 480]{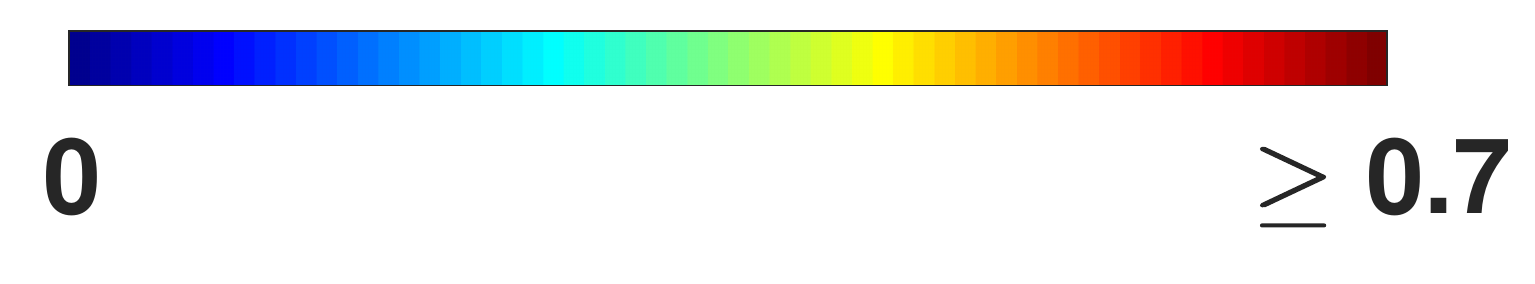}
\vspace{-5pt}
\\
\hfill
 \begin{subfigure}[]{0.18\linewidth}
 \includegraphics[width=1\linewidth, clip,trim=120 30 170 35]{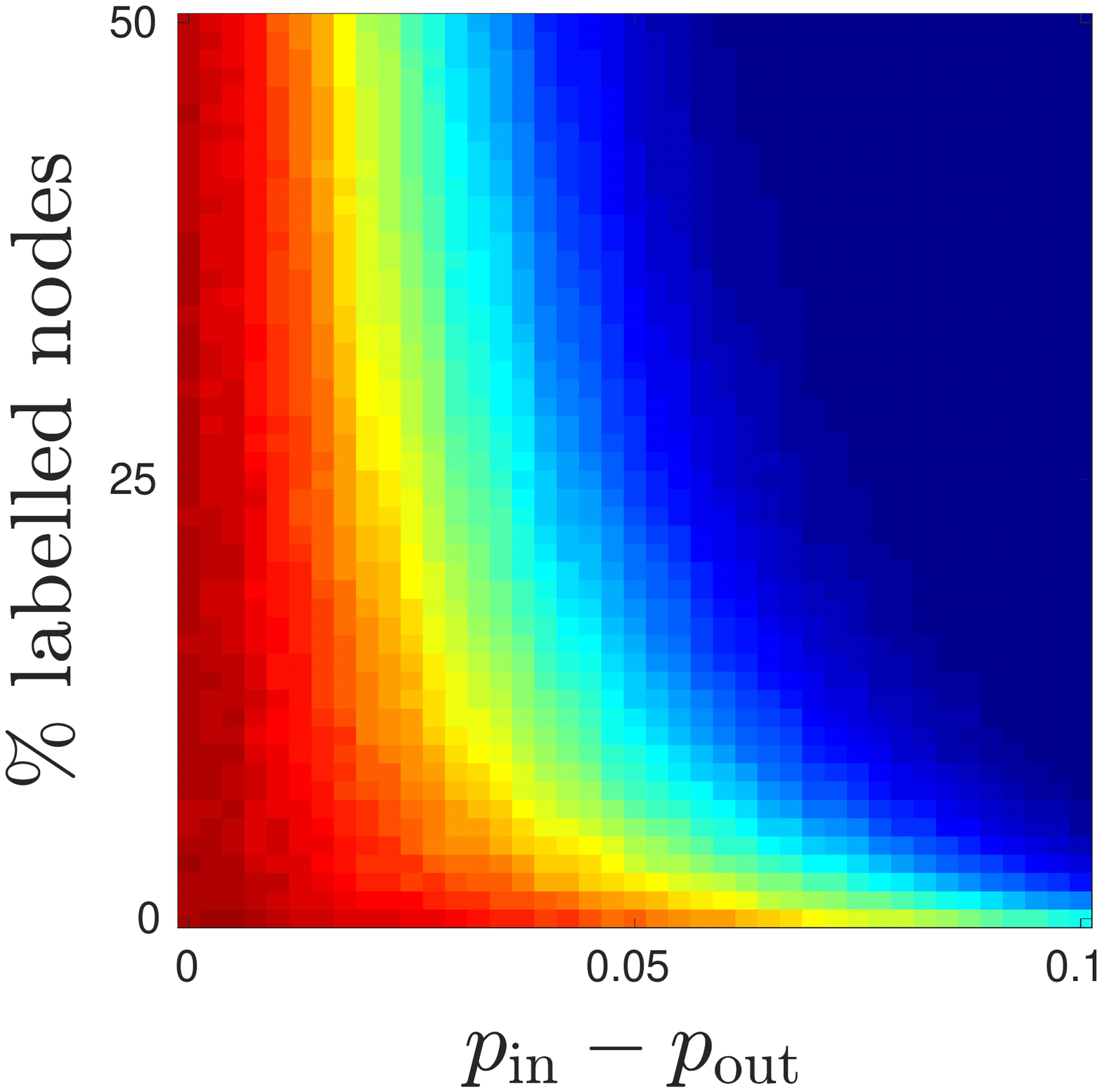}
 \caption{$L_{-10}$}
 \end{subfigure}
 \hfill 
 \begin{subfigure}[]{0.18\linewidth}
 \includegraphics[width=1\linewidth, clip,trim=120 30 170 35]{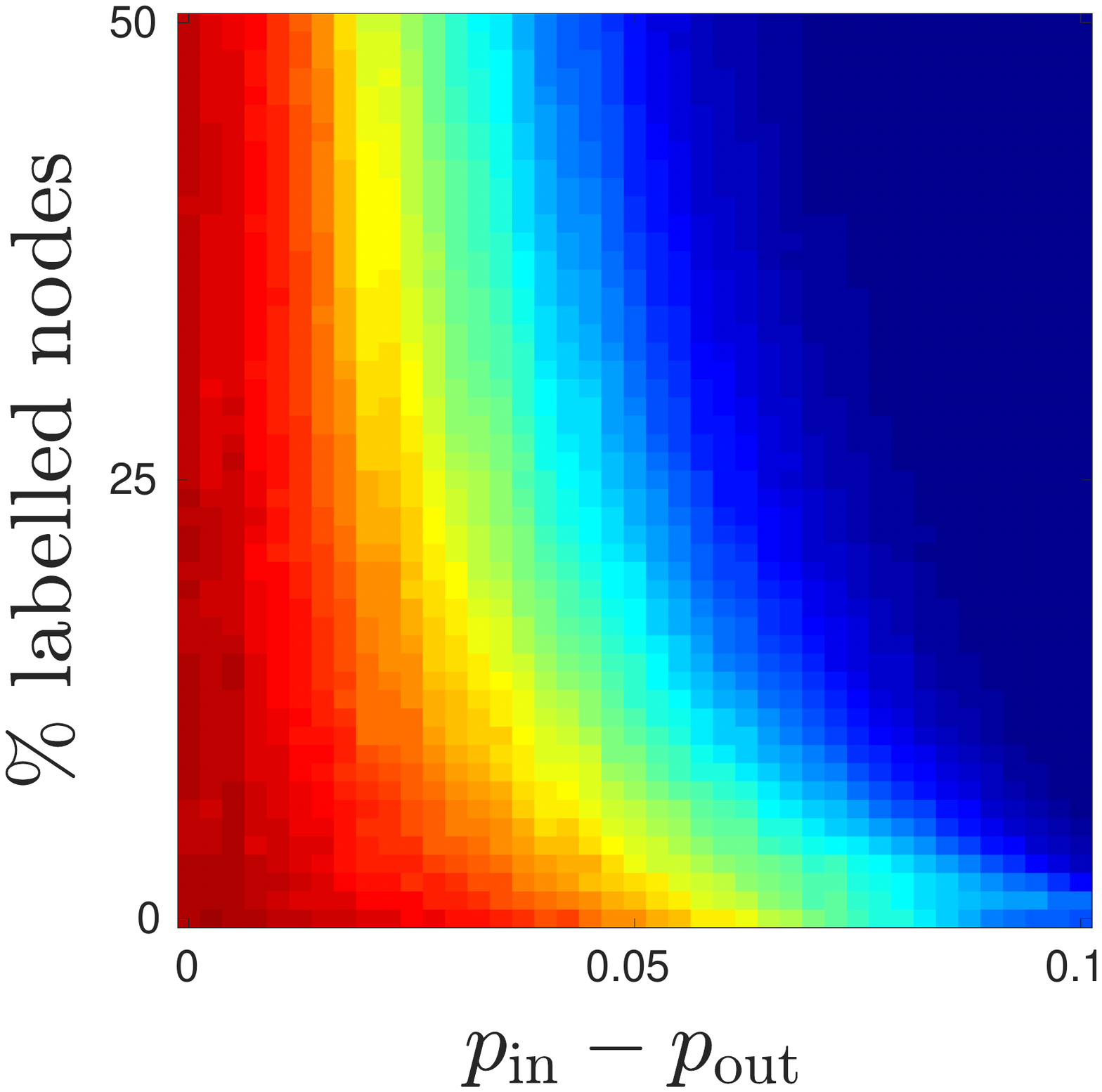}
 \caption{$L_{-1}$}
 \end{subfigure}
 \hfill 
 \begin{subfigure}[]{0.18\linewidth}
 \includegraphics[width=1\linewidth, clip,trim=120 30 170 35]{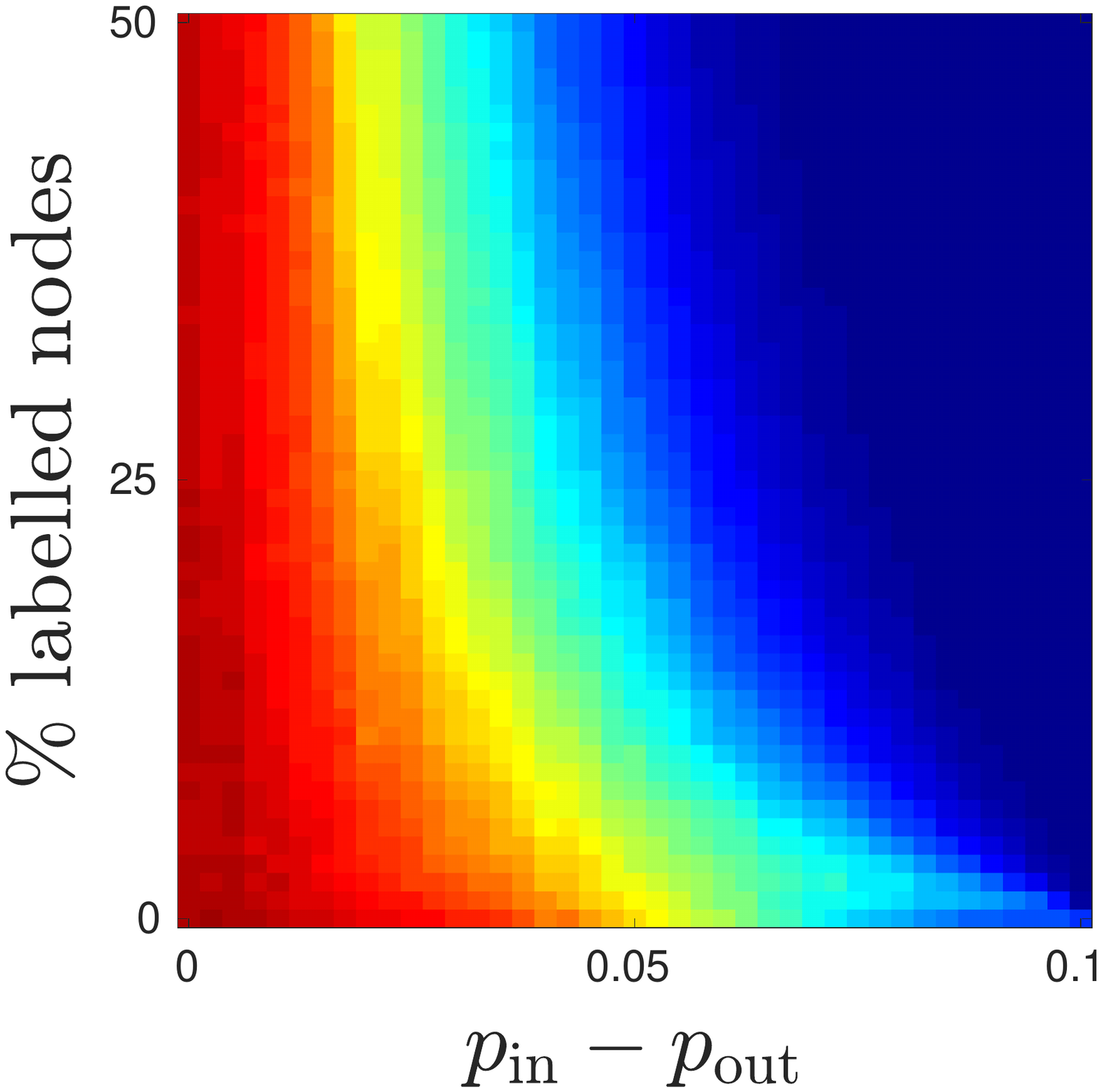}
 \caption{$L_{0}$}
 \end{subfigure}
 \hfill 
 \begin{subfigure}[]{0.18\linewidth}
 \includegraphics[width=1\linewidth, clip,trim=120 30 170 35]{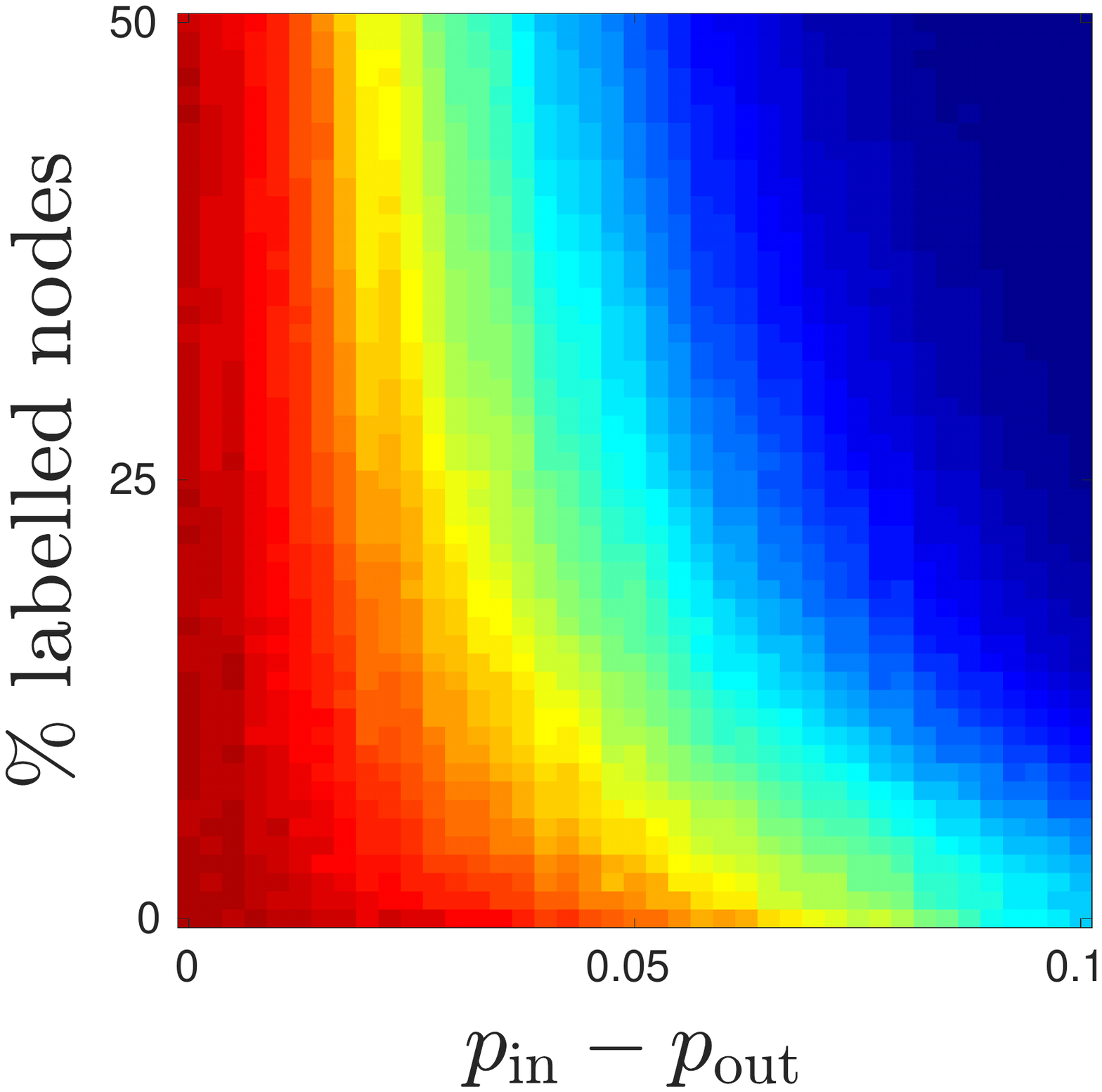}
 \caption{$L_{1}$}
 \end{subfigure}
 \hfill 
 \begin{subfigure}[]{0.18\linewidth}
 \includegraphics[width=1\linewidth, clip,trim=120 30 170 35]{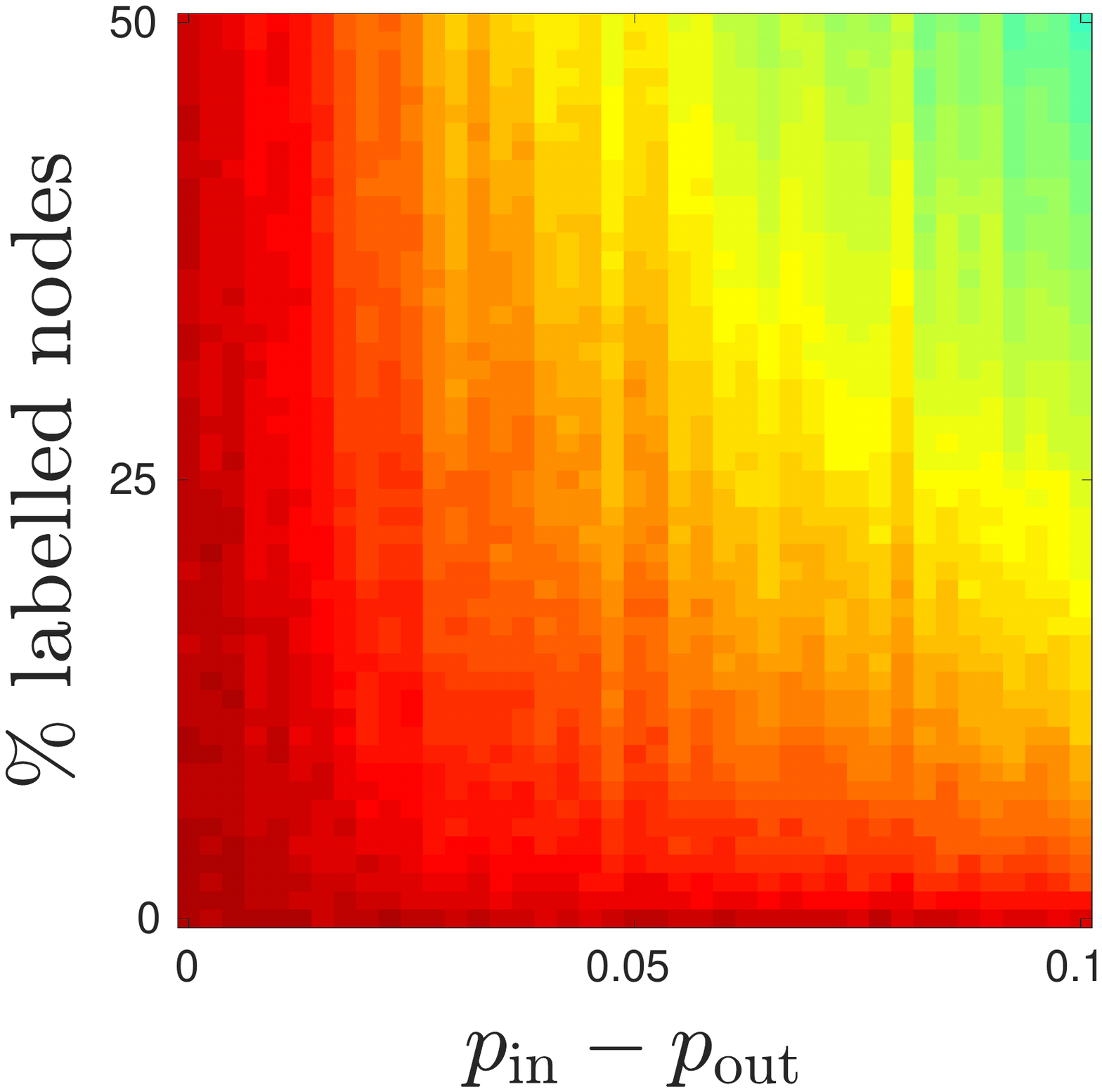}
 \caption{$L_{10}$}
 \end{subfigure}
 \hfill

 \hfill
 \begin{subfigure}[]{0.18\linewidth}
 \includegraphics[width=1\linewidth, clip,trim=120 30 170 35]{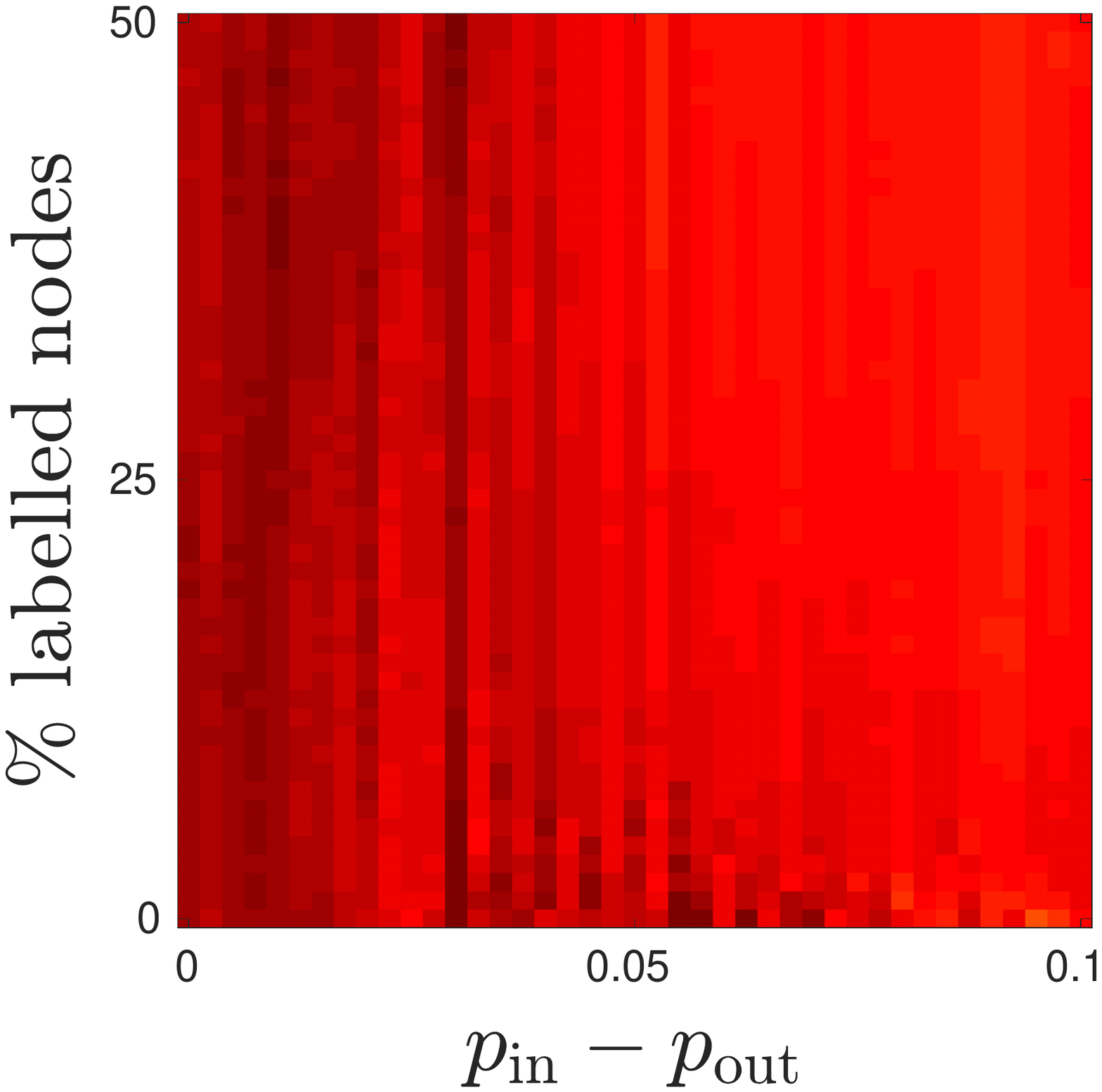}
 \caption{SMACD}
 \end{subfigure}
 \hfill 
 \begin{subfigure}[]{0.18\linewidth}
 \includegraphics[width=1\linewidth, clip,trim=120 30 170 35]{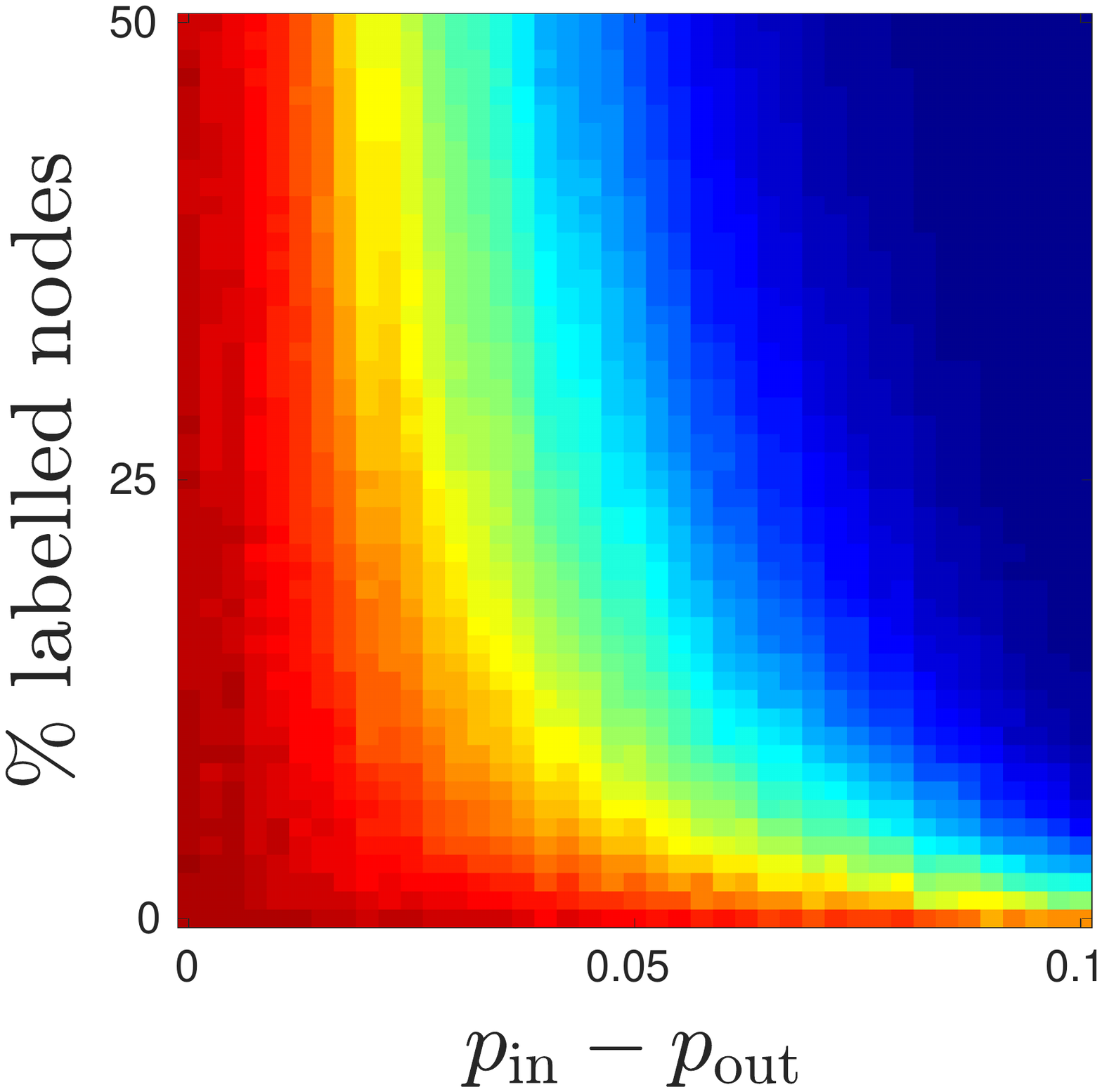}
 \caption{AGML}
 \end{subfigure}
 \hfill 
 \begin{subfigure}[]{0.18\linewidth}
 \includegraphics[width=1\linewidth, clip,trim=120 30 170 35]{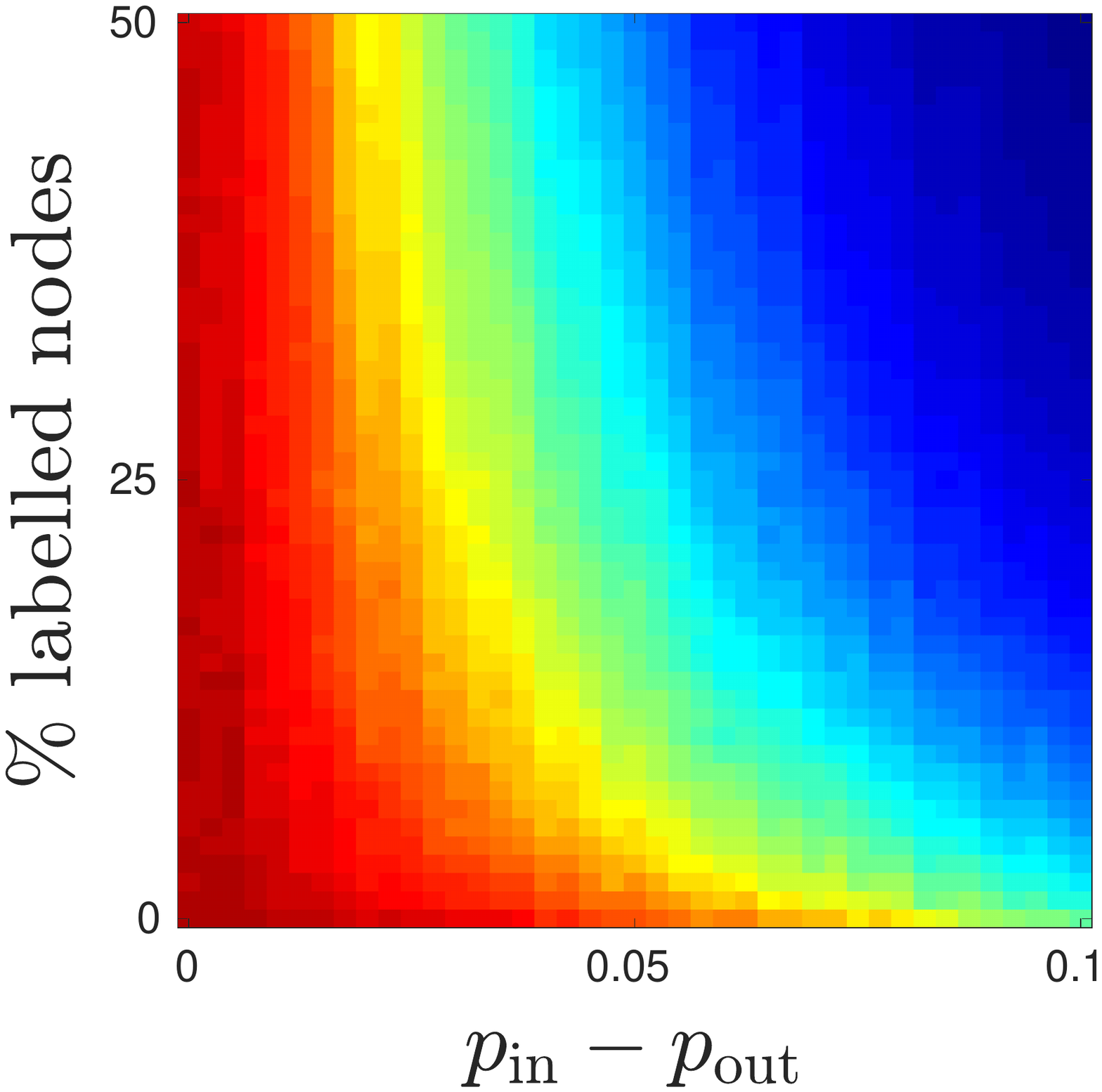}
 \caption{TLMV}
 \end{subfigure}
 \hfill 
 \begin{subfigure}[]{0.18\linewidth}
 \includegraphics[width=1\linewidth, clip,trim=120 30 170 35]{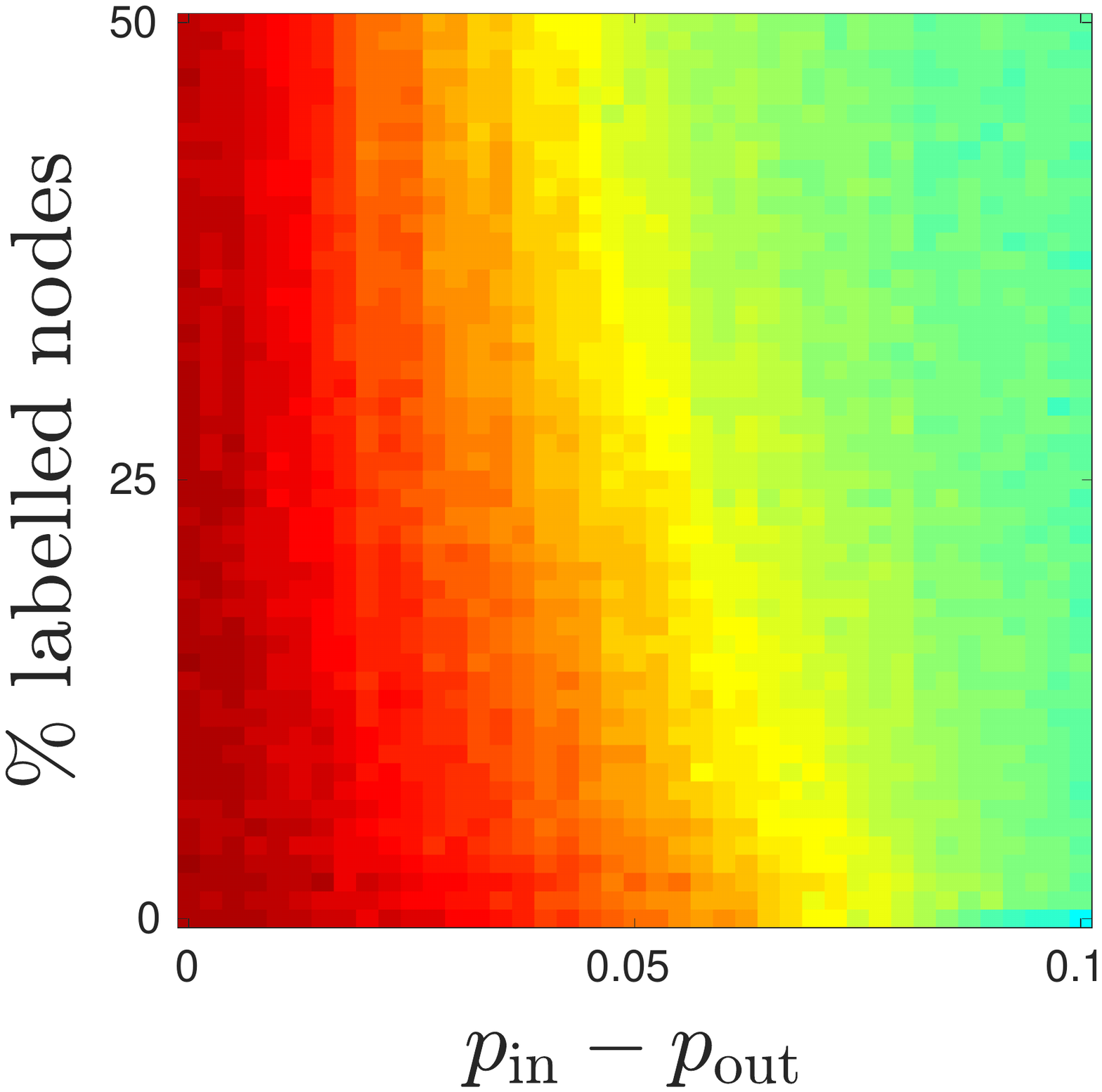}
 \caption{SGMI}
 \end{subfigure}
 \hfill 
 \begin{subfigure}[]{0.18\linewidth}
 \includegraphics[width=1\linewidth, clip,trim=120 30 170 35]{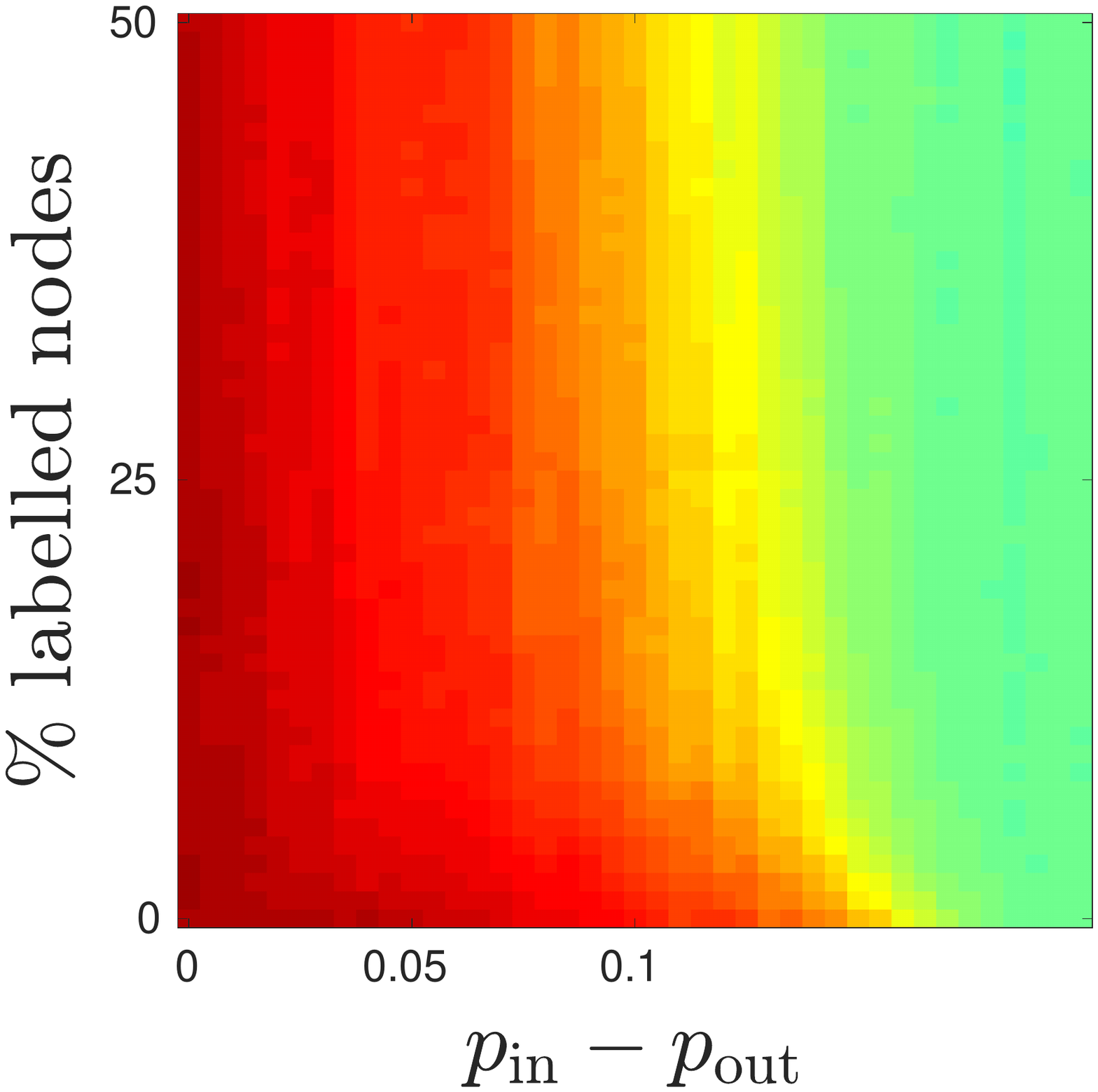}
 \caption{TSS}
 \end{subfigure}
 \hfill
 \caption{Average test error under the SBM.
Multilayer graph with 3 layers and 3 classes.
 \textbf{Top Row:} Particular cases with the power mean Laplacian.
 \textbf{Bottom Row:} State of the art models.
%
 }
 \label{fig:3layers}
 \vspace{-5pt}
\end{figure}
\subsection{Information-Independent Layers}\label{sec:independentInformation}

In the previous section we considered the case where at least one layer had enough information to correctly estimate node class labels. In this section we now consider the case where single layers taken alone obtain a large classification error, whereas when all the layers are taken together it is possible to obtain a good classification performance.
For this setting we consider multilayer graphs with 3 layers ($T=3$) and three classes ($k=3$)  $\mathcal{C}_1,\mathcal{C}_2,\mathcal{C}_3$, each composed by 100 nodes ($\abs{\mathcal{C}}=100$) with the following expected adjacency matrix per layer:
\begin{align}\label{eq:3Clusters}
 \mathcal W^{(t)}_{i,j} = 
\begin{cases} 
      \p, & \quad v_i,v_j\in \mathcal C_t \text{ or }  v_i,v_j\in \overline{\mathcal C_t}\\
      \q, & \quad \textrm{else} \\
   \end{cases} 
\end{align}
for ${t=1,2,3}$, i.e. layer $G^{(t)}$ is informative of class $\mathcal{C}_t$ but not of the remaining classes, and hence any classification method using one single layer will provide a poor classification performance.
In Fig.~\ref{fig:3layers} we present numerical experiments:
for each parameter setting $(\p,\q)$ we generate 5 multilayer graphs together with 5 samples of labeled nodes yielding a total of 25 runs per setting, and report the average test classification error.
Also in this case we observe that
the power mean Laplacian regularizer does identify the global class structure and that it leverages the information provided by labeled nodes, particularly for smaller values of $p$. On the other hand, this is not the case for all other state of the art methods. In fact, we can see that SGMI and TSS performs similarly to $L_{10}$  which has the largest classification error. 
Moreover, we can see that AGML and TLMV perform similarly to the arithmetic mean of Laplacians $L_1$, which in turn is outperformed by the power mean Laplacian regularizer $L_{-10}$.
Please see the supplementary material for a more detailed comparison.
\def\phi{\varphi}
\section{A Scalable Matrix-free Numerical Method for the System $(I+\lambda L_p)f = Y$}\label{Section:numerics} 
In this section we introduce a matrix-free method for the solution of the system $(I+\lambda L_p)f = Y$ based on contour integrals and Krylov subspace methods.
The method exploits the sparsity of the Laplacians of each layer and is matrix-free, in the sense that it requires only to compute the matrix-vector product $L_\sym^{(i)}\times vector$, without requiring to store the matrices. 
Thus, when the layers are sparse, the method scales to large datasets. Observe that this is a critical requirement as $L_p$ is in general a dense matrix, even for very sparse layers, and thus computing and storing $L_p$ is very prohibitive for large multilayer graphs. 
We present a method for negative integer values $p<0$, leaving aside the limit case $p\to 0$ as it requires a particular treatment.
The following is a brief overview of the proposed approach. Further details are available in the supplementary material.

Let $A_1,\ldots,A_T$ be symmetric positive definite matrices,
$\phi:\mathbb{C}\to\mathbb{C}$ defined by $\phi(z)= z^{1/p}$ and $L_p = T^{-1/p}\phi(S_p)$, where $S_p = A_1^p+\cdots+A_T^p$. The proposed method consists of three main steps:
\begin{enumerate}[topsep=-3pt,leftmargin=*]\setlength\itemsep{-3pt}
 \item 
 We  solve the system $(I+\lambda L_p)^{-1}Y$ 
 via a Krylov method (e.g. PCG or GMRES) 
 with convergence rate
 $ O( (\frac{\kappa^2-1}{\kappa^2})^{h/2} )$~\cite{saad1986gmres}, where 
$\kappa=\lambda_\text{max}(L_p)/\lambda_\text{min}(L_p)$. At iteration $h$, this method projects the problem onto the Krylov subspace spanned by 
$ \{Y,\lambda L_p Y , (\lambda L_p)^2Y,\ldots,(\lambda L_p)^h Y\}, $
and efficiently solve the projected problem.

\item The previous step requires the matrix-vector product $L_p Y=T^{-1/p}\phi(S_p)Y$ which we compute by approximating the  Cauchy integral form of the function $\phi$ with  the trapezoidal rule in the complex plane~\cite{hale2008computing}. Taking $N$ suitable contour points and coefficients $\beta_0,\ldots,\beta_N$, we have
\begin{align}
\textstyle{\phi_N(S_p)Y = \beta_0 S_p\,\, \mathrm{Im}\left\{\sum_{i=1}^N \beta_i (z_i^2I-S_p)^{-1}Y \right\},}
\end{align}
which has geometric convergence~\cite{hale2008computing}:
$ \|\phi(S_p)Y - \phi_N(S_p)Y\| = O(e^{-2\pi^2N/(\ln(M/m)+6)})$, 
where $m,M$ are such that $M\geq\lambda_\text{max}(S_p)$ and $m\leq \lambda_\text{min}(S_p)$.

\item The previous step requires to solve linear systems of the form $(zI-S_p)^{-1}Y$. We solve each of these systems via a Krylov subspace method, projecting, at each iteration $h$, onto the subspace spanned by $  \{Y,S_p Y , S_p^2Y,\ldots,S_p^h Y\}$. 
Since $S_p=\sum_{i=1}^T A_i^{-\abs{p}}$
this problem reduces to computing $\abs{p}$ linear systems with $A_i$ as coefficient matrix, for $i=1\ldots,T$. Provided that $A_1,\ldots,A_T$ are sparse matrices, this is done efficiently using pcg with incomplete Cholesky preconditioners.
\end{enumerate}
Notice that the method allows a high level of parallelism. In fact, the $N$ (resp.\ $p$) linear systems solvers at step $2$ (resp.\ $3$) are independent and can be run in parallel. Moreover, note that the main task of the method is solving linear systems with Laplacian matrices, which can be solved linearly in the number of edges in the corresponding adjacency matrix. Hence, the proposed approach scales to large sparse graphs and is highly parallelizable. 
A time execution analysis is provided in Fig~\ref{fig:timeComparison}, where we can see that the time execution of
our approach is competitive to the state of the art as TSS\cite{Tsuda:2005:FPC:1093772.1181531}, outperforming AGML\cite{Nie:2016:PAM:3060832.3060884}, SGMI\cite{Karasuyama:2013} and SMACD\cite{gujral2018smacd}.
%

%
\begin{figure*}[t]
\floatbox[{\capbeside\thisfloatsetup{capbesideposition={right,top},capbesidewidth=.6\textwidth}}]{figure}[\FBwidth]%
{
\vspace{-5pt}
\caption{
Mean execution time of 10 runs for different methods. 
$L_{-1}(\textrm{ours})$
stands for the power mean Laplacian regularizer together with our proposed matrix-free contour integral based method. 
We generate multilayer graphs with two layers, each with two classes of same size with parameters $p_{\mathrm{in}} = 0.05$ and $p_{\mathrm{in}} = 0.025$
and graphs of of sizes $[0.5,1,2,4,8]\times 10^4$.
Observe that our matrix free approach for $L_{-1}$ (solid blue curve) is competitive to state of the art approaches 
as TSS\cite{Tsuda:2005:FPC:1093772.1181531}, outperforming AGML\cite{Nie:2016:PAM:3060832.3060884}, SGMI\cite{Karasuyama:2013} and SMACD\cite{gujral2018smacd}.
For TLMV\cite{zhou2007spectral} and SGMI we use our own implementation.
}
    \label{fig:timeComparison}}{
\begin{minipage}{.35\textwidth}
  \begin{center}
    \includegraphics[width=1.0\linewidth, clip,trim=00 200 35 390]{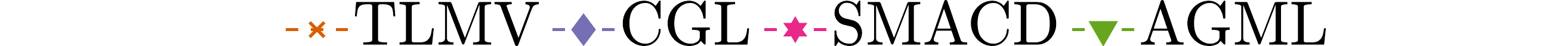}
    \vspace{-12pt}
    \\
    \includegraphics[width=1.0\linewidth, clip,trim=00 282 35 300]{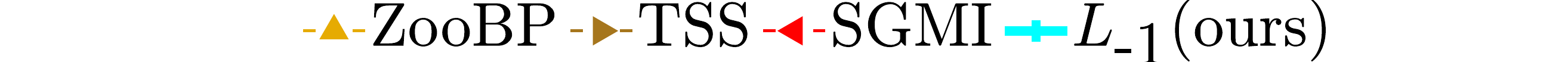}
    \\
    \includegraphics[width=1.0\textwidth, clip,trim=40 35 80 65]{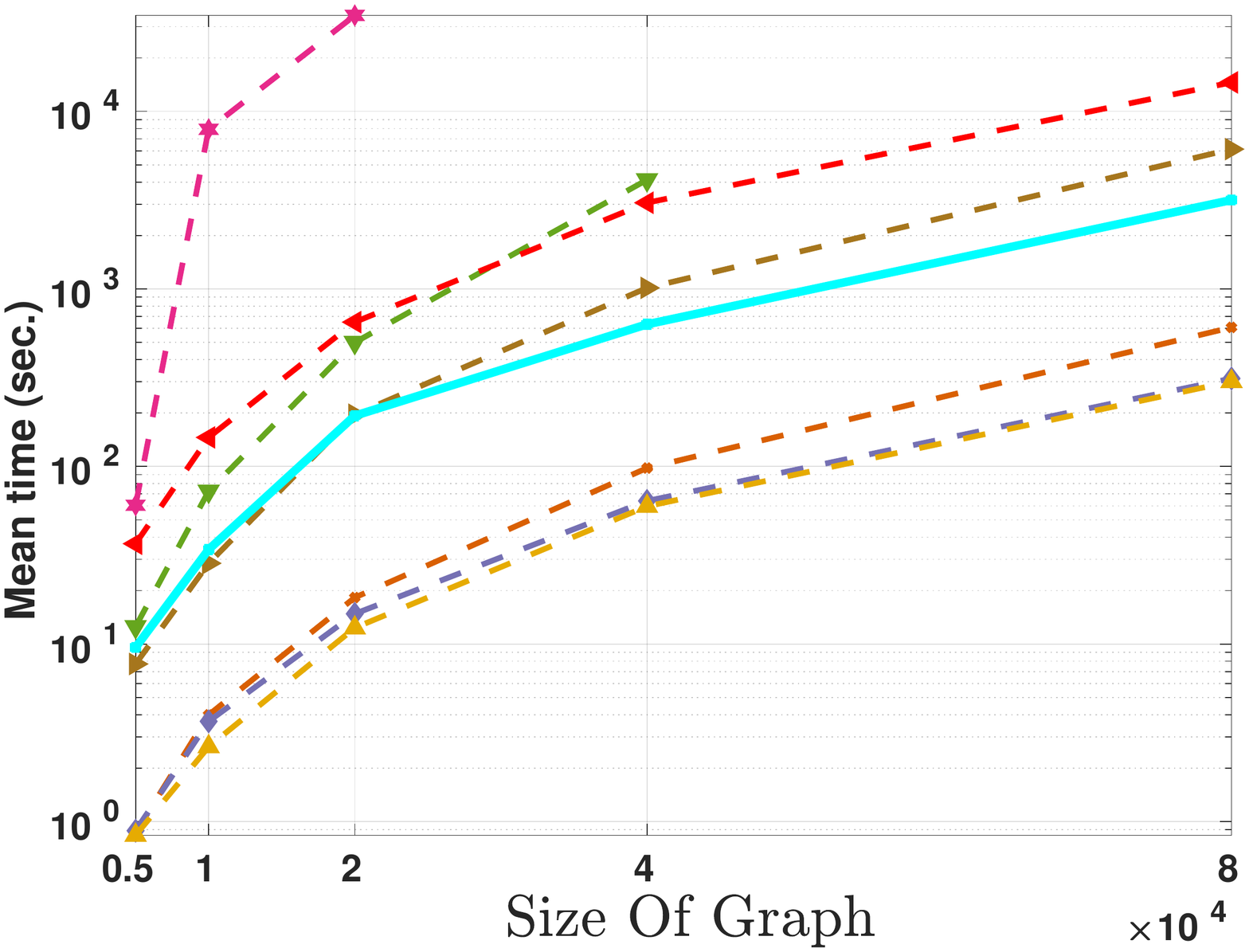}
  \end{center}
\end{minipage}
}
    \vspace{-15pt}
\end{figure*}
%

%
\section{Experiments on Real Datasets}\label{Section:experiments}
In this section we compare the performance of the proposed approach with state of the art methods on real world datasets. 
We consider the following datasets:
\textit{3-sources}~\cite{liu2013multi}, which consists of news articles that were covered by news sources BBC, Reuters and Guardian;
\textit{BBC}\cite{Greene2005} and \textit{BBC Sports}\cite{greene2009matrix} news articles,
a dataset of Wikipedia articles with ten different classes~\cite{rasiwasia2010new}, 
the hand written \textit{UCI} digits dataset with six different set of
features, 
and citations datasets
\textit{CiteSeer}\cite{lu:icml03},
\textit{Cora}\cite{mccallum2000automating} and
\textit{WebKB}(Texas)\cite{Craven:1998:LES:295240.295725}. 
For each dataset we build the corresponding layer adjacency matrices by taking the symmetric $k$-nearest neighbour graph using as similarity measure the Pearson linear correlation, (i.e. we take the $k$ neighbours with highest correlation), and take the unweighted version of it.
Datasets CiteSeer, Cora and WebKB have only two layers, where the first one is a fixed precomputed citation layer, and the second one is the corresponding  $k$-nearest neighbour graph built from document features.

As \textbf{baseline methods} we consider:
TSS~\cite{Tsuda:2005:FPC:1093772.1181531} which identifies an optimal linear combination of graph Laplacians,
SGMI~\cite{Karasuyama:2013} which performs label propagation by sparse integration,
TLMV~\cite{zhou2007spectral} which is a weighted arithmetic mean of adjacency matrices,
CGL~\cite{Argyriou:2005} which is a convex combination of the pseudo inverse Laplacian kernel,
AGML~\cite{Nie:2016:PAM:3060832.3060884} which is a parameter-free method for optimal graph layer weights,
ZooBP~\cite{Eswaran:2017:ZBP:3055540.3055554} which is a fast approximation of Belief Propagation, and
SMACD~\cite{gujral2018smacd} which is a tensor factorization method designed for semi-supervised learning.
Finally we set parameters for TSS to ($c=10,c_0=0.4$), SMACD ($\lambda=0.01$)\footnote{this is the default value in the code released by the authors: https://github.com/egujr001/SMACD},
TLMV ($\lambda = 1$), SGMI ($\lambda_1=1,\lambda_2=10^{-3}$)
and $\lambda=0.1$ for $L_1$ and $\lambda=10$ for $L_{-1}$ and $L_{-10}$.
We do not perform cross validation in our experimental setting due to the large execution time in some of the methods here considered. Hence we fix the parameters for each method in all experiments.

We fix nearest neighbourhood size to $k=10$ and generate 10 samples of labeled nodes, where the percentage of labeled nodes per class is in the range $\{1\%,5\%,10\%,15\%,20\%,25\%\}$. The average test errors are presented in table~\ref{table:experiments}, where the \colorbox{gray!25}{\textbf{best}} (resp.\colorbox{gray!25}{second best}) performances are marked with bold fonts and gray background (resp. with only gray background).
We can see that the first and second best positions are in general taken by the power mean Laplacian regularizers $L_1,L_{-1},L_{-10}$, being clear for all datasets except with 3-sources. Moreover we can see that in $77\%$ of all cases $L_{-1}$ presents either the best or the second best performance,
further verifying that our proposed approach based on the power mean Laplacian for semi-supervised learning in multilayer graph is a competitive alternative to 
state of the art methods\footnote{Communications with the authors of~\cite{gujral2018smacd} could not clarify the bad performance of SMACD.}. 
\vfill
\begin{table}[t]\label{table:experiments}         
\scriptsize
\begin{tabular}{|ccccccc|}
\hline
\multicolumn{7}{|c|}{\textbf{3sources}}  \\
\hline                                                         
  & $1\%$ & $5\%$ & $10\%$ & $15\%$ & $20\%$ & $25\%$ \\       
\specialrule{1pt}{-1pt}{1pt}                                                                                                                                        
$\textrm{TLMV}$ & 29.8 & \cellcolor{gray!25}21.5 & \cellcolor{gray!25}\textbf{20.8} & 20.3 & 15.5 & 16.5 \\                               
$\textrm{CGL}$ & 50.2 & 45.5 & 36.4 & 30.6 & 23.8 & 19.8 \\                                                                               
$\textrm{SMACD}$ & 91.5 & 91.1 & 91.2 & 90.9 & 90.7 & 91.3 \\                                                                             
$\textrm{AGML}$ & \cellcolor{gray!25}\textbf{23.9} & 26.3 & 33.9 & 33.3 & 26.1 & 22.0 \\                                                  
$\textrm{ZooBP}$ & 31.0 & 21.9 & \cellcolor{gray!25}21.3 & \cellcolor{gray!25}19.8 & \cellcolor{gray!25}15.0 & 15.3 \\                    
$\textrm{TSS}$ & 29.8 & 23.9 & 33.1 & 34.6 & 34.8 & 35.0 \\                                                                               
$\textrm{SGMI}$ & 34.4 & 26.6 & 25.4 & 24.4 & 19.1 & 17.9 \\                                                                              
$L_{\textrm{1}}$ & 33.5 & 23.9 & 23.4 & 20.1 & 15.6 & \cellcolor{gray!25}\textbf{14.6} \\                                                 
$L_{\textrm{-1}}$ & \cellcolor{gray!25}28.4 & \cellcolor{gray!25}\textbf{20.0} & 21.8 & 22.0 & 17.2 & 17.9 \\                             
$L_{\textrm{-10}}$ & 40.9 & 29.1 & 21.9 & \cellcolor{gray!25}\textbf{19.3} & \cellcolor{gray!25}\textbf{14.8} & \cellcolor{gray!25}14.7 \\
\hline                                               
\end{tabular} 
\hfill
\hfill
\begin{tabular}{|ccccccc|}   
\hline
\multicolumn{7}{|c|}{\textbf{BBC}}  \\
\hline
& $1\%$ & $5\%$ & $10\%$ & $15\%$ & $20\%$ & $25\%$ \\       
\specialrule{1pt}{-1pt}{1pt}                                                      
$\textrm{TLMV}$ & \cellcolor{gray!25}\textbf{29.0} & \cellcolor{gray!25}19.3 & \cellcolor{gray!25}13.2 & \cellcolor{gray!25}11.1 & \cellcolor{gray!25}9.3 & \cellcolor{gray!25}8.8 \\                                      
$\textrm{CGL}$ & 72.5 & 52.3 & 36.1 & 27.4 & 22.0 & 17.1 \\                                                                                                                                                                
$\textrm{SMACD}$ & 74.4 & 73.5 & 72.8 & 72.6 & 72.5 & 72.4 \\                                                                                                                                                              
$\textrm{AGML}$ & 60.0 & 34.2 & 18.6 & 13.1 & 11.0 & 9.5 \\                                                                                                                                                                
$\textrm{ZooBP}$ & 31.1 & 20.1 & 15.0 & 12.2 & 10.0 & 9.1 \\                                                                                                                                                               
$\textrm{TSS}$ & 40.4 & 26.1 & 20.9 & 20.1 & 19.8 & 19.7 \\                                                                                                                                                                
$\textrm{SGMI}$ & 37.6 & 28.9 & 24.9 & 22.8 & 20.7 & 19.3 \\                                                                                                                                                               
$L_{\textrm{1}}$ & 31.3 & 22.8 & 17.4 & 13.5 & 10.2 & 8.9 \\                                                                                                                                                               
$L_{\textrm{-1}}$ & \cellcolor{gray!25}31.0 & \cellcolor{gray!25}\textbf{17.0} & \cellcolor{gray!25}\textbf{11.5} & \cellcolor{gray!25}\textbf{10.5} & \cellcolor{gray!25}\textbf{9.2} & \cellcolor{gray!25}\textbf{8.7} \\
$L_{\textrm{-10}}$ & 51.6 & 26.9 & 16.6 & 12.8 & 10.3 & 9.5 \\                                                                                                                                                             
\hline                                                          
\end{tabular}  
\hfill
\\
\vspace{10pt}
\begin{tabular}{|ccccccc|}   
\hline
\multicolumn{7}{|c|}{\textbf{BBCS}}  \\
\hline                                                         
  & $1\%$ & $5\%$ & $10\%$ & $15\%$ & $20\%$ & $25\%$ \\       
\specialrule{1pt}{-1pt}{1pt}     
$\textrm{TLMV}$ & 25.6 & \cellcolor{gray!25}12.6 & \cellcolor{gray!25}10.5 & 7.5 & 6.4 & 5.4 \\                                                                                                                                   
$\textrm{CGL}$ & 79.2 & 51.6 & 34.9 & 23.4 & 16.5 & 12.7 \\                                                                                                                                                                       
$\textrm{SMACD}$ & 77.8 & 80.6 & 82.4 & 96.4 & 98.4 & 98.3 \\                                                                                                                                                                     
$\textrm{AGML}$ & 34.6 & 17.4 & 12.1 & \cellcolor{gray!25}7.0 & \cellcolor{gray!25}6.0 & \cellcolor{gray!25}5.4 \\                                                                                                                
$\textrm{ZooBP}$ & 33.8 & 13.9 & 11.3 & 8.8 & 7.6 & 6.2 \\                                                                                                                                                                        
$\textrm{TSS}$ & \cellcolor{gray!25}23.9 & 13.2 & 14.1 & 12.3 & 13.1 & 12.2 \\                                                                                                                                                    
$\textrm{SGMI}$ & 31.9 & 19.6 & 16.6 & 15.5 & 14.8 & 12.1 \\                                                                                                                                                                      
$L_{\textrm{1}}$ & 29.9 & 15.0 & 13.5 & 10.6 & 8.7 & 7.2 \\                                                                                                                                                                       
$L_{\textrm{-1}}$ & \cellcolor{gray!25}\textbf{23.8} & \cellcolor{gray!25}\textbf{11.6} & \cellcolor{gray!25}\textbf{8.7} & \cellcolor{gray!25}\textbf{6.3} & \cellcolor{gray!25}\textbf{5.8} & \cellcolor{gray!25}\textbf{5.1} \\
$L_{\textrm{-10}}$ & 48.7 & 22.5 & 14.2 & 9.1 & 7.8 & 6.1 \\                                                                                                                                                                      
\hline                                                 
\end{tabular} 
\hfill
\begin{tabular}{|ccccccc|}       
\hline
\multicolumn{7}{|c|}{\textbf{Wikipedia}}  \\
\hline
& $1\%$ & $5\%$ & $10\%$ & $15\%$ & $20\%$ & $25\%$ \\       
\specialrule{1pt}{-1pt}{1pt}                                                      
$\textrm{TLMV}$ & \cellcolor{gray!25}65.7 & \cellcolor{gray!25}56.8 & 46.4 & 43.1 & 40.8 & 39.2 \\                                                                                                                                    
$\textrm{CGL}$ & 87.3 & 83.0 & 82.5 & 82.2 & 83.0 & 83.0 \\                                                                                                                                                                           
$\textrm{SMACD}$ & 85.4 & 85.6 & 85.4 & 85.3 & 86.8 & 90.0 \\                                                                                                                                                                         
$\textrm{AGML}$ & 71.3 & 66.6 & 48.1 & 42.1 & 38.4 & 37.3 \\                                                                                                                                                                          
$\textrm{ZooBP}$ & 67.6 & 58.0 & 47.0 & 43.8 & 41.2 & 39.8 \\                                                                                                                                                                         
$\textrm{TSS}$ & 87.7 & 84.7 & 83.3 & 81.9 & 82.3 & 81.4 \\                                                                                                                                                                           
$\textrm{SGMI}$ & 69.3 & 84.8 & 84.5 & 83.8 & 83.2 & 82.8 \\                                                                                                                                                                          
$L_{\textrm{1}}$ & 68.2 & 61.1 & 53.6 & 48.3 & 44.1 & 42.3 \\                                                                                                                                                                         
$L_{\textrm{-1}}$ & \cellcolor{gray!25}\textbf{59.1} & \cellcolor{gray!25}\textbf{52.3} & \cellcolor{gray!25}\textbf{40.2} & \cellcolor{gray!25}\textbf{36.3} & \cellcolor{gray!25}\textbf{35.1} & \cellcolor{gray!25}\textbf{34.1} \\
$L_{\textrm{-10}}$ & 66.9 & 57.2 & \cellcolor{gray!25}43.2 & \cellcolor{gray!25}38.7 & \cellcolor{gray!25}36.3 & \cellcolor{gray!25}34.9 \\                                                                                           
\hline                                                         
\end{tabular}  
\hfill
\\
\vspace{10pt}
\begin{tabular}{|ccccccc|}    
\hline
\multicolumn{7}{|c|}{\textbf{UCI}}  \\
\hline                                                         
  & $1\%$ & $5\%$ & $10\%$ & $15\%$ & $20\%$ & $25\%$ \\       
\specialrule{1pt}{-1pt}{1pt}     
$\textrm{TLMV}$ & 28.9 & 20.4 & 16.3 & 14.4 & 13.7 & 12.7 \\                                                                                                                                              
$\textrm{CGL}$ & 81.8 & 64.0 & 54.6 & 49.1 & 46.7 & 46.7 \\                                                                                                                                               
$\textrm{SMACD}$ & 73.6 & 81.0 & 90.0 & 90.0 & 86.2 & 81.9 \\                                                                                                                                             
$\textrm{AGML}$ & \cellcolor{gray!25}25.3 & \cellcolor{gray!25}17.2 & \cellcolor{gray!25}15.2 & \cellcolor{gray!25}13.2 & \cellcolor{gray!25}12.5 & \cellcolor{gray!25}12.0 \\                            
$\textrm{ZooBP}$ & 30.8 & 21.7 & 17.6 & 15.1 & 14.1 & 13.0 \\                                                                                                                                             
$\textrm{TSS}$ & \cellcolor{gray!25}\textbf{24.0} & 17.6 & 16.6 & 15.9 & 15.8 & 15.6 \\                                                                                                                   
$\textrm{SGMI}$ & 36.0 & 44.4 & 50.9 & 50.4 & 50.2 & 48.8 \\                                                                                                                                              
$L_{\textrm{1}}$ & 31.3 & 23.8 & 18.7 & 15.6 & 14.4 & 13.2 \\                                                                                                                                             
$L_{\textrm{-1}}$ & 30.5 & \cellcolor{gray!25}\textbf{17.1} & \cellcolor{gray!25}\textbf{13.8} & \cellcolor{gray!25}\textbf{12.6} & \cellcolor{gray!25}\textbf{12.3} & \cellcolor{gray!25}\textbf{11.9} \\
$L_{\textrm{-10}}$ & 57.0 & 33.8 & 23.7 & 17.6 & 15.3 & 13.4 \\                                                                                                                                           
\hline                                                 
\end{tabular} 
\hfill
\hfill
\begin{tabular}{|ccccccc|}     
\hline
\multicolumn{7}{|c|}{\textbf{Citeseer}}  \\
\hline
& $1\%$ & $5\%$ & $10\%$ & $15\%$ & $20\%$ & $25\%$ \\       
\specialrule{1pt}{-1pt}{1pt}                                                      
$\textrm{TLMV}$ & \cellcolor{gray!25}51.5 & 39.4 & 36.5 & 33.7 & 31.6 & 30.3 \\                                                                                                                                                     
$\textrm{CGL}$ & 89.3 & 71.8 & 58.0 & 49.8 & 44.5 & 40.9 \\                                                                                                                                                                         
$\textrm{SMACD}$ & 90.7 & 90.4 & 67.0 & 65.5 & 66.8 & 68.9 \\                                                                                                                                                                       
$\textrm{AGML}$ & \cellcolor{gray!25}\textbf{47.3} & \cellcolor{gray!25}\textbf{32.3} & \cellcolor{gray!25}\textbf{29.6} & \cellcolor{gray!25}\textbf{28.2} & \cellcolor{gray!25}\textbf{27.5} & \cellcolor{gray!25}\textbf{27.0} \\
$\textrm{ZooBP}$ & 63.6 & 41.9 & 38.7 & 35.8 & 33.8 & 32.2 \\                                                                                                                                                                       
$\textrm{TSS}$ & 58.5 & 49.5 & 45.9 & 42.1 & 39.8 & 38.4 \\                                                                                                                                                                         
$\textrm{SGMI}$ & 59.4 & 46.8 & 44.0 & 42.3 & 40.5 & 39.2 \\                                                                                                                                                                        
$L_{\textrm{1}}$ & 56.3 & 44.1 & 41.2 & 38.5 & 36.1 & 34.7 \\                                                                                                                                                                       
$L_{\textrm{-1}}$ & 52.4 & \cellcolor{gray!25}39.0 & \cellcolor{gray!25}35.6 & \cellcolor{gray!25}32.6 & \cellcolor{gray!25}30.9 & \cellcolor{gray!25}29.5 \\                                                                       
$L_{\textrm{-10}}$ & 68.6 & 54.6 & 48.5 & 43.0 & 39.7 & 37.2 \\                                                                                                                                                                     
\hline                                                        
\end{tabular}  
\hfill
\\
\vspace{10pt}
\begin{tabular}{|ccccccc|}    
\hline
\multicolumn{7}{|c|}{\textbf{Cora}}  \\
\hline                                                         
  & $1\%$ & $5\%$ & $10\%$ & $15\%$ & $20\%$ & $25\%$ \\       
\specialrule{1pt}{-1pt}{1pt}     
$\textrm{TLMV}$ & 46.0 & 34.1 & 28.8 & 25.8 & 22.5 & 20.6 \\                                                                                                                    
$\textrm{CGL}$ & 85.5 & 70.1 & 56.5 & 49.1 & 44.2 & 40.0 \\                                                                                                                     
$\textrm{SMACD}$ & 75.6 & 76.7 & 78.7 & 78.7 & 81.0 & 87.1 \\                                                                                                                   
$\textrm{AGML}$ & 54.7 & 36.0 & 25.4 & \cellcolor{gray!25}\textbf{20.7} & \cellcolor{gray!25}\textbf{18.1} & \cellcolor{gray!25}\textbf{16.5} \\                                
$\textrm{ZooBP}$ & 54.7 & 38.0 & 32.9 & 30.2 & 27.6 & 26.2 \\                                                                                                                   
$\textrm{TSS}$ & \cellcolor{gray!25}\textbf{38.8} & \cellcolor{gray!25}\textbf{27.7} & \cellcolor{gray!25}\textbf{24.1} & 21.5 & 20.0 & 19.1 \\                                 
$\textrm{SGMI}$ & 57.3 & 47.7 & 43.0 & 41.8 & 40.1 & 38.5 \\                                                                                                                    
$L_{\textrm{1}}$ & 50.7 & 38.2 & 33.4 & 31.2 & 28.2 & 25.6 \\                                                                                                                   
$L_{\textrm{-1}}$ & \cellcolor{gray!25}43.2 & \cellcolor{gray!25}31.8 & \cellcolor{gray!25}24.5 & \cellcolor{gray!25}21.1 & \cellcolor{gray!25}18.8 & \cellcolor{gray!25}17.2 \\
$L_{\textrm{-10}}$ & 62.0 & 46.3 & 35.4 & 29.4 & 25.2 & 22.3 \\                                                                                                                 
\hline                                                  
\end{tabular} 
\hfill
\hfill
\begin{tabular}{|ccccccc|}   
\hline
\multicolumn{7}{|c|}{\textbf{WebKB}}  \\
\hline
& $1\%$ & $5\%$ & $10\%$ & $15\%$ & $20\%$ & $25\%$ \\       
\specialrule{1pt}{-1pt}{1pt}                                                      
$\textrm{TLMV}$ & 58.6 & 49.4 & 45.6 & 47.2 & 47.6 & 48.2 \\                                                                                                           
$\textrm{CGL}$ & 80.4 & 82.4 & 84.4 & 86.9 & 82.7 & 89.2 \\                                                                                                            
$\textrm{SMACD}$ & 87.3 & 87.2 & 87.2 & 87.4 & 87.8 & 87.8 \\                                                                                                          
$\textrm{AGML}$ & 56.5 & 50.3 & 46.8 & 44.7 & 47.6 & 46.8 \\                                                                                                           
$\textrm{ZooBP}$ & 52.0 & 45.0 & \cellcolor{gray!25}38.7 & 38.5 & \cellcolor{gray!25}\textbf{36.4} & \cellcolor{gray!25}\textbf{33.5} \\                               
$\textrm{TSS}$ & 60.9 & 51.0 & 50.5 & 47.3 & 49.2 & 48.7 \\                                                                                                            
$\textrm{SGMI}$ & \cellcolor{gray!25}\textbf{44.9} & \cellcolor{gray!25}\textbf{39.7} & 41.9 & \cellcolor{gray!25}\textbf{34.9} & 40.3 & 52.5 \\                       
$L_{\textrm{1}}$ & 58.5 & 49.0 & 44.8 & 44.3 & 44.5 & 44.4 \\                                                                                                          
$L_{\textrm{-1}}$ & \cellcolor{gray!25}49.9 & 45.5 & 40.7 & 39.5 & 39.9 & 40.3 \\                                                                                      
$L_{\textrm{-10}}$ & 52.3 & \cellcolor{gray!25}41.9 & \cellcolor{gray!25}\textbf{38.0} & \cellcolor{gray!25}38.1 & \cellcolor{gray!25}36.8 & \cellcolor{gray!25}39.5 \\
\hline                                                          
\end{tabular}  
\hfill
\caption{
 Experiments in real datasets. Notation:\colorbox{gray!25}{\textbf{best}} performances are marked with bold fonts and gray background and \colorbox{gray!25}{second best} performances with only gray background.
}                                                  
\label{table:experiments}       
\end{table}    

\textbf{Acknowledgement}
P.M and M.H are
supported by the DFG Cluster of Excellence “Machine Learning – New
Perspectives for Science”, EXC 2064/1, project number 390727645

\bibliography{referencesUpdated}%
\bibliographystyle{abbrv}
\ifpaper
\newpage\clearpage
\appendix
\appendixpage

This section contains all the proofs of results mentioned in the main paper, together with a detailed description of the proposed numerical scheme. It is organized as follows:
\begin{itemize}[topsep=-3pt,leftmargin=*]\setlength\itemsep{-3pt}
 \item Section~\ref{appendix:auxiliaryResults} contains auxiliary results,
 \item Section~\ref{appendix:Proof:theorem:generalization-equallySized-equallyLabelled} contains the proof of Theorem~\ref{theorem:generalization-equallySized-equallyLabelled},
 \item Section~\ref{appendix:corollary:limit_cases} contains the proof of Corollary~\ref{corollary:limit_cases},
 \item Section~\ref{appendix:corollary:contention} contains the proof of Corollary~\ref{corollary:contention},
 \item Section~\ref{appendix:theorem:arbitrary-labelled-datasets--sup} contains general version of Theorem~\ref{theorem:unequalLabels:epsilon-goes-to-zero}
 \item Section~\ref{appendix:theorem:unequalLabels:epsilon-goes-to-zero} contains the proof of Theorem~\ref{theorem:unequalLabels:epsilon-goes-to-zero},
 \item Section~\ref{appendix:theorem:generalization-equallySized-NOTequallyLabelled} contains the proof of Theorem~\ref{theorem:generalization-equallySized-NOTequallyLabelled}.
 \item Section~\ref{appendix:numerics} contains a detailed exposition of the numerical scheme presented in Section~\ref{Section:numerics}.
 \item Section~\ref{appendix:3Layers} we present a numerical analysis for the case of three layers and three classes as in Sec.~\ref{sec:independentInformation}
 \item Section~\ref{sec:OnLambda} presents a numerical analysis of the effect of the regularization parameter $\lambda$
\end{itemize}

\section{Auxiliary Results}\label{appendix:auxiliaryResults}

We first present some results that will be useful.

The following theorem states the monotonicity of the scalar power mean.
\begin{theorem}[\cite{bullen2013handbook}, Ch.~3, Thm.~1]\label{theorem:mp_monotone}
 Let $p<q$ then $m_{p}(a,b)\leq m_{q}(a,b)$
with equality if and only if $a=b$.
\end{theorem}

The following lemma shows the effect of the matrix power mean when matrices have a common eigenvector.
\begin{lemma}
[\cite{Mercado:2018:powerMean}]
\label{lemma:eigenvalues_and_eigenvectors_of_generalized_mean_V2}
Let $\u$ be an eigenvector of $A_1,\ldots,A_T$ with corresponding eigenvalues $\lambda_1,\ldots,\lambda_T$.
Then 
$\u$ is an eigenvector of $M_p(A_1,\ldots,A_T)$ with eigenvalue $m_p(\lambda_1,\ldots,\lambda_T)$.
\end{lemma}

The following Lemma states the eigenvalues and eigenvectors of expected adjacency matrices according to the Stochastic Block Model here considered.
\begin{lemma}\label{Lemma:eigvecs-eigvals}
Let $\mathcal{C}_1, \ldots, \mathcal{C}_k$ be clusters of equal size $\abs{\mathcal{C}}=n/k$.
 Let $\mathcal{W}\in\R^{n\times n}$ be defined as
\begin{align}
 \mathcal{W} = (\p-\q)\sum_{i=1}^k\one_{\mathcal{C}_i}\one_{\mathcal{C}_i}^T + \q \one\one^T
\end{align}
and let $\bchi_1,\ldots,\bchi_k \in\R^n$ be defined as
\begin{align}
 \bchi_1 = \one,\, \qquad \bchi_r = \sum_{j=1}^r \one_{\mathcal{C}_j} - r\one_{\mathcal{C}_r}
\end{align} 
for $r=2,\ldots,k$.
Then, $\bchi_1,\ldots,\bchi_k$ are orthogonal eigenvectors of $\mathcal{W}$, with eigenvalues
\begin{align}
 \lambda_1 = \abs{\mathcal{C}}(\p+(k-1)\q),\, \qquad \lambda_r = \abs{\mathcal{C}} (\p-\q)
\end{align}

\end{lemma}

\begin{proof}
Please note that from the definition that the matrix $\mathcal{W}$ is equal to $\p$ in the block diagonal and $\q$ elsewhere. 
We first consider the following matrix vector products that can be easily verified:
\begin{align}
 \mathcal{W}\one &= \abs{\mathcal{C}}(\p+(k-1)\q)\one
 \\
 \mathcal{W}\one_{\mathcal{C}_i} &= \abs{\mathcal{C}}(\p\one_{\mathcal{C}_i} +\q\one_{\overline{\mathcal{C}_i}} )
\end{align}
Moreover, we can see that
\begin{align*}
 \mathcal{W} \left( \one_{\mathcal{C}_j} - \one_{\mathcal{C}_i} \right)
 &=
 \abs{\mathcal{C}} \left( \left( \p \one_{\mathcal{C}_j} + \q \one_{\overline{\mathcal{C}}_j} \right) - \left( \p \one_{\mathcal{C}_i} + \q \one_{\overline{\mathcal{C}}_i} \right) \right)
 \\
  &=
  \abs{\mathcal{C}} \left( \p \left( \one_{\mathcal{C}_j}-\one_{\mathcal{C}_i}\right) + \q \left( \one_{\overline{\mathcal{C}}_j} - \one_{\overline{\mathcal{C}}_i} \right) \right)
 \\
  &=
  \abs{\mathcal{C}} \left( \p \left( \one_{\mathcal{C}_j}-\one_{\mathcal{C}_i}\right) - \q \left( \one_{\mathcal{C}_i} - \one_{\mathcal{C}_j} \right) \right)
 \\
  &=
  \abs{\mathcal{C}} (\p-\q)\left( \one_{\mathcal{C}_j}-\one_{\mathcal{C}_i}\right)
\end{align*}

Now we show that $\bchi_2,\ldots,\bchi_k$ are eigenvectors of $\mathcal{W}$.
\begin{align*}
 \mathcal{W}\bchi_r 
 &= 
 \mathcal{W} \left( \sum_{j=1}^r \one_{\mathcal{C}_j} - r\one_{\mathcal{C}_r}  \right)
 \\
 &= 
 \mathcal{W}  \sum_{j=1}^r ( \one_{\mathcal{C}_j} - \one_{\mathcal{C}_r} )
 \\
 &= 
 \sum_{j=1}^r \mathcal{W} \left( \one_{\mathcal{C}_j} - \one_{\mathcal{C}_r} \right)
 \\
 &= 
 \sum_{j=1}^r \abs{\mathcal{C}} (\p-\q)\left( \one_{\mathcal{C}_j}-\one_{\mathcal{C}_r}\right)
 \\
 &= 
 \abs{\mathcal{C}} (\p-\q)
 \sum_{j=1}^r \left( \one_{\mathcal{C}_j}-\one_{\mathcal{C}_r}\right)
 \\
 &= 
 \abs{\mathcal{C}} (\p-\q)
 \left( \sum_{j=1}^r \one_{\mathcal{C}_j} - r\one_{\mathcal{C}_r}  \right)
 \\
 &= 
 \abs{\mathcal{C}} (\p-\q)
 \bchi_r
 \\
 &= \lambda_r\bchi_r
\end{align*}

Furthermore, we can see that eigenvectors $\bchi_2,\ldots,\bchi_k$ are orthogonal. Let $2\leq r < s \leq k$, then
\begin{align*}
 \bchi_r^T \bchi_s 
 &= 
 \left( \sum_{j_1=1}^r \one_{\mathcal{C}_{j_{1}}} - r\one_{\mathcal{C}_r}  \right)^T  
 \left( \sum_{j_2=1}^s \one_{\mathcal{C}_{j_{2}}} - s\one_{\mathcal{C}_s}  \right)
 \\
 &=
 \sum_{j_1=1}^r \sum_{j_2=1}^s  
 \one_{\mathcal{C}_{j_{1}}}^T \one_{\mathcal{C}_{j_{2}}} - 
 s\sum_{j_1=1}^r \one_{\mathcal{C}_{j_{1}}}^T\one_{\mathcal{C}_s} -
 r\sum_{j_2=1}^s \one_{\mathcal{C}_{j_{2}}}^T \one_{\mathcal{C}_{r}} + r s\one_{\mathcal{C}_{r}}^T\one_{\mathcal{C}_s} 
 \\
 &=
 \sum_{j_1=1}^r \sum_{j_2=1}^s 
 \one_{\mathcal{C}_{j_{1}}}^T \one_{\mathcal{C}_{j_{2}}} -
 r\sum_{j_2=1}^s \one_{\mathcal{C}_{j_{2}}}^T \one_{\mathcal{C}_{r}}
 \\
 &=
 \sum_{j_1=1}^r \sum_{j_2=1}^s 
 \one_{\mathcal{C}_{j_{1}}}^T \one_{\mathcal{C}_{j_{2}}} -
 r \one_{\mathcal{C}_{r}}^T \one_{\mathcal{C}_{r}}
 \\
 &=
 \sum_{j_1=1}^r \sum_{j_2=1}^s 
 \left( 
 \one_{\mathcal{C}_{j_{1}}}^T \one_{\mathcal{C}_{j_{2}}}\right) -
 r\abs{\mathcal{C}}
 \\
 &=
 \sum_{j_1=1}^r  
 \left( 
 \one_{\mathcal{C}_{j_{1}}}^T \one_{\mathcal{C}_{j_{1}}}\right) -
 r\abs{\mathcal{C}}
 \\
 &=
 \sum_{j_1=1}^r 
 \abs{\mathcal{C}}
-
 r\abs{\mathcal{C}}
\\
 &=
 r\abs{\mathcal{C}}
-
 r\abs{\mathcal{C}}
\\
&=
0
 \end{align*}
 where in the third step we used that fact that $\one_{\mathcal{C}_{r}}^T\one_{\mathcal{C}_s}=0$ as $r<s$, and
$\one_{\mathcal{C}_{j_{1}}}^T\one_{\mathcal{C}_s}=0$ as $j_1<s$.
Finally, we can see that for $2\leq r \leq k$
\begin{align*}
  \bchi_1^T \bchi_r 
 &= 
 \one^T \left( \sum_{j=1}^r \one_{\mathcal{C}_{j}} - r\one_{\mathcal{C}_r} \right)
 \\
 &=
 \sum_{j=1}^r \left( \one^T\one_{\mathcal{C}_{j}} \right) - r \one^T \one_{\mathcal{C}_r} 
 \\
 &=
 r \abs{\mathcal{C}}  - r \abs{\mathcal{C}}
 \\
 &=
 0
\end{align*}
and hence $\bchi_1,\ldots,\bchi_k$ are orthogonal eigenvectors of the matrix $\mathcal{W}$.
 
\end{proof} 

The following Lemma shows the eigenvectors and eigenvalues of the power mean Laplacian in expectation under the considered Stochastic Block Model.

\begin{lemma}\label{lemma:eigVecs-eigVals-PowerMean}
 Let $E(\mathbb{G})$ be the expected multilayer graph with $T$ layers following the multilayer SBM with
 $k$ 
 classes $\mathcal{C}_1, \ldots, \mathcal{C}_k$ 
 of equal size
 and parameters $\left(\p^{(t)},\q^{(t)}\right)_{t=1}^{T}$. 
 Then the eigenvalues of the power mean Laplacian $\mathcal{L}_p$ are
  \begin{align}\label{eq:eigenvaluesOfPowerMeanInExpectation2}
  \lambda_1(\mathcal L_p) =\varepsilon, \qquad
  \lambda_i(\mathcal L_p) = m_p(\boldsymbol{\rho_\epsilon}),  \qquad
  \lambda_j(\mathcal L_p) =1+\varepsilon
  \end{align}
with eigenvectors
\begin{align*}
 \bchi_1 = \one,\, \qquad \bchi_i = \sum_{j=1}^i \one_{\mathcal{C}_j} - i\one_{\mathcal{C}_i}
\end{align*} 
where $(\boldsymbol{\rho_\epsilon})_t = 1-(\p^{(t)} - \q^{(t)})/(\p^{(t)} + (k-1)\q^{(t)})+\epsilon$, $t=1,\ldots,T$, 
$i=2,\ldots,k$, and $j=k+1,\ldots,\abs{V}$,
\end{lemma}

\begin{proof}
From Lemma~\ref{Lemma:eigvecs-eigvals} we know that $\boldsymbol \chi_1,\ldots,\boldsymbol \chi_k$ are eigenvectors of $\mathcal W^{(1)},\ldots,\mathcal W^{(T)}$. 
In particular, we have seen that
\begin{align*}
 \lambda^{(t)}_1 = \abs{\mathcal{C}}(\p^{(t)}+(k-1)\q^{(t)}), \quad \lambda^{(t)}_i = \abs{\mathcal{C}}(\p^{(t)}-\q^{(t)})\\
\end{align*}
for $i=2,\ldots,k$.
Further, as matrices $\mathcal W^{(1)},\ldots,\mathcal W^{(T)}$ share all their eigenvectors, they are simultaneously diagonalizable, i.e. there exists a 
non-singular matrix $\Sigma$ such that $\Sigma^{-1}\mathcal{W}^{(t)}\Sigma =\Lambda^{(t)}$, 
where $\Lambda^{(t)}$ are diagonal matrices  $\Lambda^{(t)} = \diag(\lambda_1^{(t)}, \dots, \lambda_k^{(t)}, 0, \dots, 0)$.

As we assume that all clusters are of the same size $\abs{\mathcal{C}}$, the expected layer graphs are regular graphs with degrees $d^{(1)},\ldots,d^{(T)}$. Hence, the normalized Laplacians  of the expected layer graphs can be expressed as 
\begin{align*}
 \mathcal L^{(t)}_\sym &= \Sigma(I-\frac{1}{d^{(t)}}\Lambda^{(t)})\Sigma^{-1}\\
\end{align*}
Thus, we can observe that
\begin{align*}
 \lambda^{(t)}_1(\mathcal L^{(t)}_\sym) = 0,\qquad
 \lambda^{(t)}_i(\mathcal L^{(t)}_\sym) =1-\rho^{(t)}, \qquad 
 \lambda^{(t)}_j(\mathcal L^{(t)}_\sym) =1, 
\end{align*}
for $i=2,\ldots,k$, and $j=k+1,\ldots,\abs{V}$,
where
\begin{align*}
 \rho^{(t)} &=(\p^{(t)}-\q^{(t)})/(\p^{(t)}+(k-1)\q^{(t)}) 
\end{align*}
for $t=1,\ldots,T$.
By obtaining the power mean Laplacian on diagonally shifted matrices, 
$$\mathcal L_p = M_p(\mathcal L^{(1)}_\sym+\varepsilon I, \ldots, \mathcal L^{(1)}_\sym+\varepsilon I)$$
we have by Lemma~\ref{lemma:eigenvalues_and_eigenvectors_of_generalized_mean_V2}
\begin{equation}\label{eq:eigenvaluesOfPowerMeanInExpectation}
  \begin{aligned}
  \lambda_1(\mathcal L_p) &= m_p(\lambda^{(1)}_1+\varepsilon, \ldots,\lambda^{(T)}_1+\varepsilon)=\varepsilon \\
  \lambda_i(\mathcal L_p) &= m_p(1-\rho^{(1)} +\varepsilon, \ldots, 1-\rho^{(T)} +\varepsilon) = m_p(\boldsymbol{\rho_\epsilon})  \\
  \lambda_j(\mathcal L_p) &= m_p(\lambda^{(1)}_j+\varepsilon, \ldots,\lambda^{(T)}_j+\varepsilon)=1+\varepsilon
  \end{aligned}
\end{equation}
where $(\boldsymbol{\rho_\epsilon})_t = 1-(\p^{(t)} - \q^{(t)})/(\p^{(t)} + (k-1)\q^{(t)})+\epsilon$, and $t=1,\ldots,T$, 
$i=2,\ldots,k$, and $j=k+1,\ldots,\abs{V}$,
\end{proof}
The following Lemma describes the general form of the solution matrix
$$
F = (f^{(1)}, \ldots, f^{(k)})
$$
where the columns of $F$ are obtained from the following optimization problem
$$
 f^{(r)} = \argmin_{f\in\R^n} \|f-CY^{(r)}\|^2 
 +  
 \lambda
 f^T L_pf
$$
Observe that this setting contains as a particular case the problem described in Eq.~\eqref{local-global-optimization-problem-rephrased-power-mean-laplacian}.

\begin{lemma}\label{lemma:solutionMatrix-F-supp}
 Let $E(\mathbb{G})$ be the expected multilayer graph with $T$ layers following the multilayer SBM with
 $k$ classes $\mathcal{C}_1, \ldots, \mathcal{C}_k$ of equal size and parameters $\left(\p^{(t)},\q^{(t)}\right)_{t=1}^{T}$. 
 Let $\boldsymbol{\rho_\epsilon}$ be defined as in Lemma~\ref{lemma:eigVecs-eigVals-PowerMean}.
 Let $n_1,\ldots,n_k$ be the number of labeled nodes per class.
 Let $C\in\R^{n\times n}$ be a diagonal matrix with $C_{ii} = c_r$ for $v_i\in\mathcal{C}_r$.
 Let $l(v_i)$ be the label of node $v_i$, i.e. $l(v_i)=r$ if and only if $v_i\in\mathcal{C}_r$.
%
Let the solution matrix $F = (f^{(1)}, \ldots, f^{(k)})$ where
$$
 f^{(r)} = \argmin_{f\in\R^n} \|f-CY^{(r)}\|^2 
 +  
 \mu
 f^T \mathcal{L}_pf
$$
Then the solution matrix F
is such that:
 \begin{itemize}[topsep=-3pt,leftmargin=*]\setlength\itemsep{-3pt}
 \item If $r < l(v_i)$, then
 \begin{align*}
 f_i^{(r)} = 
 c_r
 \frac{n_r}{n}\alpha
 +
 c_r
 n_r
   \beta
\left(
 (1-l(v_i))\frac{1}{\norm{\bchi_{l(v_i)}}^2}
+
\sum_{j=l(v_i)+1}^k \frac{1}{\norm{\bchi_j}^2}
\right)
\end{align*}
 \item If $r > l(v_i)$, then
\begin{align*}
 f_i^{(r)} = 
 c_r
 \frac{n_r}{n}\alpha
 +
 c_r
 n_r
   \beta
\left(
 (1-r)\frac{1}{\norm{\bchi_{r}}^2}
+
\sum_{j=r+1}^k \frac{1}{\norm{\bchi_j}^2}
\right)
\end{align*}
 \item If $r = l(v_i)$, then
\begin{align*}
 f_i^{(r)} = 
 c_r
 \frac{n_r}{n}\alpha
 +
 c_r
 n_r
   \beta
\left(
 (1-r)^2\frac{1}{\norm{\bchi_{r}}^2}
+
\sum_{j=r+1}^k \frac{1}{\norm{\bchi_j}^2}
\right)
\end{align*}
\end{itemize}
where 
$\alpha=\frac{1}{1+\mu\epsilon} - \frac{1}{1+\mu(1+\epsilon)}$, and 
$\beta=\frac{1}{1+\mu m_p(\boldsymbol{\rho_\epsilon})} - \frac{1}{1+\mu(1+\epsilon)}$.

\end{lemma}
\begin{proof}
Let $U\in\R^{n\times n}$ be an orthonormal matrix such that $U = (u_1, u_2, \ldots, u_n)$, with $u_i = \bchi_i/\norm{\bchi_i}$ for $i=1,\ldots,k$,
where $\bchi_1,\ldots,\bchi_k$ are eigenvectors of the power mean Laplacian as described in Lemma~\ref{lemma:eigVecs-eigVals-PowerMean}.

The power mean Laplacian $\mathcal{L}_p$ is a symmetric positive semidefinite matrix (see Lemma~\ref{lemma:eigVecs-eigVals-PowerMean}) and hence we can express $\mathcal{L}_p$ as $U\Lambda U^T$ where $\Lambda$ is a diagonal matrix with entries $\Lambda_{ii} = \lambda_i(\mathcal{L}_p)$, with $i=1,\ldots,n$.
Hence, we can see that
\begin{align*}
 (I + \mu\mathcal{L}_p)^{-1} = (I + U\Lambda U^T)^{-1} = (U (I+\Lambda) U^T)^{-1} = U (I+\Lambda)^{-1} U^T = U \Omega U^T
\end{align*}
where $\Omega$ is a diagonal matrix with entries $\Omega_{ii} = \frac{1}{1+\mu\lambda_i}$, with $i=1,\ldots,n$.

From  Lemma~\ref{lemma:eigVecs-eigVals-PowerMean} we know that 
$\lambda_{k+1} = \cdots = \lambda_n = 1+\epsilon =: \widehat{\omega}$, and hence it follows that $\Omega_{ii} = \frac{1}{1+\mu\widehat{\omega}}$ for $i=k+1,\ldots,n$. Moreover, we can express $\Omega$ as the sum of two diagonal matrices, i.e.
\begin{align*}
 \Omega = \omega I + \Theta
\end{align*}
where $\omega=\frac{1}{1+\mu\widehat{\omega}}
$ 
and 
$\Theta=\diag\left( \Omega_{11} - \omega, \ldots,\Omega_{kk} - \omega,0,\ldots,0\right)$.
Observe that 
$\Theta_{11} = \Omega_{11} - \omega = \frac{1}{1+\mu\epsilon} - \frac{1}{1+\mu(1+\epsilon)}=:\alpha$
and
$\Theta_{jj} = \Omega_{jj} - \omega = \frac{1}{1+\mu m_p(\boldsymbol{\rho_\epsilon})} - \frac{1}{1+\mu(1+\epsilon)}=:\beta$, for $j=2,\ldots,k$.
\\
Recall that we are interested in the equation 
$$
F = (I + \mu \mathcal{L}_p)^{-1}C Y = U \Omega U^T C Y\in\R^{n\times k},
$$ 
where each column of $Y=[y^{(1)}, \ldots, y^{(k)} ]$ is a class indicator of labeled nodes, i.e. 
\begin{equation}
y^{(j)}_i=
\begin{cases} 
      1 & \text{ if } \,\, l(v_i) = j \\
      0 & \text{ else }  \\
   \end{cases}
\end{equation}
Hence, each column of $Y$ can be expressed as
\begin{align}
y^{(j)} = 
\sum_{v_i\in V \mid l(v_i)=j } e_i
\end{align}
where $e_i\in\R^n$ and $(e_i)_i = 1$ and zero else. 
With this in mind, we now study the matrix-vector product $U \Omega U^T e_i$.
Recall that $U \Theta U^T$ is a $k$-rank matrix.
Hence we have
\begin{align*}
 U\Omega U^T e_i 
 &= U (\omega I + \Theta) U^T e_i 
 \\
 &= \omega e_i + U \Theta U^T e_i 
 \\
 &= 
 \omega e_i + \left( \sum_{j=1}^k \Theta_{jj} u_j u_j^T \right)e_i
 \\
 &= 
 \omega e_i + \left( \sum_{j=1}^k\frac{1}{\norm{\bchi_j}^2} \Theta_{jj}  \bchi_j \bchi_j^T \right)e_i
 \\
 &= 
 \omega e_i + \frac{1}{n} \Theta_{11}  \bchi_1 + \left( \sum_{j=2}^k\frac{1}{\norm{\bchi_j}^2} \Theta_{jj}  \bchi_j \bchi_j^T \right)e_i
 \\
 &= 
 \omega e_i + \frac{1}{n} \alpha  \bchi_1 + \beta \left( \sum_{j=2}^k\frac{1}{\norm{\bchi_j}^2}   \bchi_j \bchi_j^T \right)e_i
\end{align*}
where in the last steps we used the fact that $\bchi_1^T e_i = \one^T e_i = 1$, and define $\alpha=\Theta_{11}$ and $\beta=\Theta_{jj}$ due to the fact that $\Theta_{jj}$ are all equal for $j=2,\ldots,k$.

The remaining terms $\bchi_j \bchi_j^T e_i$ depend on the cluster to which the corresponding node $v_i$ belongs to. 

We first study the vector product $\bchi_r^T e_i$. Observe that
\begin{align*}
 \bchi_r^T e_i 
 = 
 \left( \sum_{j=1}^r \one_{\mathcal{C}_j} - r\one_{\mathcal{C}_r} \right)^T e_i
 =
  \sum_{j=1}^r \left(\one_{\mathcal{C}_j}^T e_i\right) - r\one_{\mathcal{C}_r}^T e_i
\end{align*}
Recall that $l(v_i)$ is the label of node $v_i$, i.e. 
$l(v_i)=r$ if and only if $v_i\in\mathcal{C}_r$.
Then, we have
\begin{equation}\label{eq:casesByLabel}
\sum_{j=1}^r \left(\one_{\mathcal{C}_j}^T e_i\right) - r\one_{\mathcal{C}_r}^T e_i=
\begin{cases} 
      0 & \text{ for } \quad r<l(v_i) \\
      1-l(v_i) & \text{ for } \quad r=l(v_i) \\
      1 & \text{ for } \quad r>l(v_i) \\
   \end{cases}
\end{equation}
%

Therefore,
\begin{align*}
 \left( \sum_{j=2}^k\frac{1}{\norm{\bchi_j}^2} 
 \bchi_j \bchi_j^T \right)e_i
 &=
 \left( 1-l(v_i) \right) \frac{\bchi_{l(v_i)}}{\norm{\bchi_{l(v_i)}}^2}
 +
 \sum_{j=l(v_i)+1}^k \frac{\bchi_j}{\norm{\bchi_j}^2}
\end{align*}

All in all we have
\begin{align*}
U\Omega U^T e_i =
  \omega e_i + \frac{1}{n} \alpha  \bchi_1 
  + 
    \beta
    \left(
  \left( 1-l(v_i) \right) \frac{\bchi_{l(v_i)}}{\norm{\bchi_{l(v_i)}}^2}
 +
 \sum_{j=l(v_i)+1}^k \frac{\bchi_j}{\norm{\bchi_j}^2}
 \right)
\end{align*}
Moreover, the solution matrix $F$ can now be described column-wise as follows

\begin{align*}
 f^{(r)} 
 &= 
 (I + \mu \mathcal{L}_p)^{-1} Cy^{(r)}
 \\
 &= c_r \left(\sum_{v_i\in V \mid l(v_i)=r } U\Omega U^T  e_i \right)
 \\
 &= c_r \left(
 \sum_{v_i\in V \mid l(v_i)=r }
 \omega e_i 
 \right)
 + \frac{1}{n} c_r n_r \alpha  \bchi_1 + 
  c_r n_r \beta
\left(
 (1-l(v_i)) \frac{\bchi_{l(v_i)}}{\norm{\bchi_{l(v_i)}}^2}
 +
 \sum_{j=l(v_i)+1}^k \frac{\bchi_j}{\norm{\bchi_j}^2}
 \right)
 \\
 &= 
 \omega c_r y^{(r)} + 
 c_r n_r
 \left(
 \frac{1}{n} \alpha  \bchi_1 + 
   \beta
\left(
 (1-r) \frac{\bchi_{r}}{\norm{\bchi_{r}}^2}
 +
 \sum_{j=r+1}^k \frac{\bchi_j}{\norm{\bchi_j}^2}
 \right)
 \right)
\end{align*}
We now study the columns of matrix $F$. 
For this, observe that the $i^{th}$ entry of the column corresponding to the class $r$, is obtained by $ f^{(r)}_i= \langle e_i,f^{(r)} \rangle$, and hence have
\begin{align*}
 \langle e_i,f^{(r)} \rangle=
 &
 \langle e_i,
 \omega c_r y^{(r)} + 
 c_r
 n_r
 \left(
 \frac{1}{n} \alpha  \bchi_1 + 
   \beta
\left(
 (1-r) \frac{\bchi_{r}}{\norm{\bchi_{r}}^2}
 +
 \sum_{j=r+1}^k \frac{\bchi_j}{\norm{\bchi_j}^2}
 \right)
 \right)
  \rangle
\\
 =
 &
c_r
  \frac{n_r}{n}\alpha
  +
c_r
n_r
\beta
 \langle e_i,
\left(
 (1-r) \frac{\bchi_{r}}{\norm{\bchi_{r}}^2}
 c_r
 +
 \sum_{j=r+1}^k \frac{\bchi_j}{\norm{\bchi_j}^2}
 \right)
  \rangle
\end{align*}
where  
$\langle e_i,  \omega c_r y^{(r)} \rangle = 0$ for unlabeled nodes.
Having this, we now proceed to study three different cases of the remaining inner product.
We do this by considering the following cases and making use of Eq.~\eqref{eq:casesByLabel}:

\textbf{First case: $f^{(r)}_i$ with $r < l(v_i)$}.
We first analyze the following term
\begin{align*}
\langle 
e_i,
\left(
 (1-r) \frac{\bchi_{r}}{\norm{\bchi_{r}}^2}
 +
 \sum_{j=r+1}^k \frac{\bchi_j}{\norm{\bchi_j}^2}
 \right)
 \rangle
 &=
 \langle 
e_i,
 (1-r) \frac{\bchi_{r}}{\norm{\bchi_{r}}^2}
 \rangle
 +
 \langle
 e_i,
 \sum_{j=r+1}^k \frac{\bchi_j}{\norm{\bchi_j}^2}
 \rangle
\\
 \text{\tiny{(by first case of Eq.\ref{eq:casesByLabel})}}\qquad\qquad
 &=
 \langle
 e_i,
 \sum_{j=r+1}^k \frac{\bchi_j}{\norm{\bchi_j}^2}
 \rangle
 \\
 \text{\tiny{(by cases of Eq.\ref{eq:casesByLabel})}}\qquad\qquad
 &=
(1-l(v_i))\frac{1}{\norm{\bchi_{l(v_i)}}^2}
+
\sum_{j=l(v_i)+1}^k \frac{1}{\norm{\bchi_j}^2}
 \\
\end{align*}
Thus, we have
\begin{equation*}
\begin{aligned}\label{eq:Theorem:1stcase:general}
 f_i^{(r)} = 
 c_r
 \frac{n_r}{n}\alpha
 +
 c_r
 n_r
   \beta
\left(
 (1-l(v_i))\frac{1}{\norm{\bchi_{l(v_i)}}^2}
+
\sum_{j=l(v_i)+1}^k \frac{1}{\norm{\bchi_j}^2}
\right)
\end{aligned}
\end{equation*}

\textbf{Second case: $f^{(r)}_i$ with $r > l(v_i)$}.
We first analyze the following term
\begin{align*}
\langle 
e_i,
\left(
 (1-r) \frac{\bchi_{r}}{\norm{\bchi_{r}}^2}
 +
 \sum_{j=r+1}^k \frac{\bchi_j}{\norm{\bchi_j}^2}
 \right)
 \rangle
 &=
 \langle 
e_i,
 (1-r) \frac{\bchi_{r}}{\norm{\bchi_{r}}^2}
 \rangle
 +
 \langle
 e_i,
 \sum_{j=r+1}^k \frac{\bchi_j}{\norm{\bchi_j}^2}
 \rangle
 \\
  \text{\tiny{(by third case of Eq.\ref{eq:casesByLabel})}}\qquad\qquad
 &=
(1-r)\frac{1}{\norm{\bchi_{r}}^2}
+
\sum_{j=r+1}^k \frac{1}{\norm{\bchi_j}^2}
 \\
\end{align*}
Thus, we have
\begin{equation*}
\begin{aligned}\label{eq:Theorem:2ndcase:general}
 f_i^{(r)} = 
 c_r
 \frac{n_r}{n}\alpha
 +
 c_r
 n_r
   \beta
\left(
 (1-r)\frac{1}{\norm{\bchi_{r}}^2}
+
\sum_{j=r+1}^k \frac{1}{\norm{\bchi_j}^2}
\right)
\end{aligned}
\end{equation*}

\textbf{Third case: $f^{(r)}_i$ with $r = l(v_i)$}.
We first analyze the following term
\begin{align*}
\langle 
e_i,
\left(
 (1-r) \frac{\bchi_{r}}{\norm{\bchi_{r}}^2}
 +
 \sum_{j=r+1}^k \frac{\bchi_j}{\norm{\bchi_j}^2}
 \right)
 \rangle
 &=
 \langle 
e_i,
 (1-r) \frac{\bchi_{r}}{\norm{\bchi_{r}}^2}
 \rangle
 +
 \langle
 e_i,
 \sum_{j=r+1}^k \frac{\bchi_j}{\norm{\bchi_j}^2}
 \rangle
 \\
 \text{\tiny{(by second case of Eq.\ref{eq:casesByLabel})}}\qquad\qquad
 &=
(1-r)^2\frac{1}{\norm{\bchi_{r}}^2}
+
\sum_{j=r+1}^k \frac{1}{\norm{\bchi_j}^2}
 \\
\end{align*}
Thus, we have
\begin{equation*}
\begin{aligned}\label{eq:Theorem:3rdcase:general}
 f_i^{(r)} = 
 c_r
 \frac{n_r}{n}\alpha
 +
 c_r
 n_r
   \beta
\left(
 (1-r)^2\frac{1}{\norm{\bchi_{r}}^2}
+
\sum_{j=r+1}^k \frac{1}{\norm{\bchi_j}^2}
\right)
\end{aligned}
\end{equation*}

These three cases are the desired conditions.

\end{proof}
\section{Proof Of Theorem~\ref{theorem:generalization-equallySized-equallyLabelled}}\label{appendix:Proof:theorem:generalization-equallySized-equallyLabelled}

\begin{theorem}\label{theorem:generalization-equallySized-equallyLabelled-supp}
 Let $E(\mathbb{G})$ be the expected multilayer graph with $T$ layers following the multilayer SBM with
 $k$ 
 classes $\mathcal{C}_1, \ldots, \mathcal{C}_k$ 
 of equal size
 and parameters $\left(\p^{(t)},\q^{(t)}\right)_{t=1}^{T}$. 
 Let the same number of nodes per class be labeled.
 Then, a zero test classification error is achieved if and only if 
\begin{align*}
  m_p(\boldsymbol{\rho_\epsilon}) < 1+\epsilon\, ,
\end{align*}
 where $(\boldsymbol{\rho_\epsilon})_t = 1-(\p^{(t)} - \q^{(t)})/(\p^{(t)} + (k-1)\q^{(t)})+\epsilon$, and $t=1,\ldots,T$.
 \end{theorem}
\begin{proof}

The proof of this theorem builds on top of Lemma~\ref{lemma:solutionMatrix-F-supp}, where the entries of the solution matrix 
$F = (f^{(1)}, \ldots, f^{(k)})$ are described, where
$$
 f^{(r)} = \argmin_{f\in\R^n} \|f-CY^{(r)}\|^2 
 +  
 \mu
 f^T \mathcal{L}_pf
$$
Let $l(v_i)$ be the label of node $v_i$, i.e. $l(v_i)=r$ if and only if $v_i\in\mathcal{C}_r$.
According to Lemma~\ref{lemma:solutionMatrix-F-supp}
the entries of matrix $F$ for unlabeled nodes are such that
 \begin{itemize}[topsep=-3pt,leftmargin=*]\setlength\itemsep{-3pt}
 \item If $r < l(v_i)$, then
 \begin{align*}
 f_i^{(r)} = 
 c_r
 \frac{n_r}{n}\alpha
 +
 c_r
 n_r
   \beta
\left(
 (1-l(v_i))\frac{1}{\norm{\bchi_{l(v_i)}}^2}
+
\sum_{j=l(v_i)+1}^k \frac{1}{\norm{\bchi_j}^2}
\right)
\end{align*}
 \item If $r > l(v_i)$, then
\begin{align*}
 f_i^{(r)} = 
 c_r
 \frac{n_r}{n}\alpha
 +
 c_r
 n_r
   \beta
\left(
 (1-r)\frac{1}{\norm{\bchi_{r}}^2}
+
\sum_{j=r+1}^k \frac{1}{\norm{\bchi_j}^2}
\right)
\end{align*}
 \item If $r = l(v_i)$, then
\begin{align*}
 f_i^{(r)} = 
 c_r
 \frac{n_r}{n}\alpha
 +
 c_r
 n_r
   \beta
\left(
 (1-r)^2\frac{1}{\norm{\bchi_{r}}^2}
+
\sum_{j=r+1}^k \frac{1}{\norm{\bchi_j}^2}
\right)
\end{align*}
\end{itemize}
where 
$\alpha=\frac{1}{1+\mu\epsilon} - \frac{1}{1+\mu(1+\epsilon)}$, and 
$\beta=\frac{1}{1+\mu m_p(\boldsymbol{\rho_\epsilon})} - \frac{1}{1+\mu(1+\epsilon)}$.

%

Observe that the case here considered corresponds to the case where the amount of labeled data per class is the same, i.e. $n_1=\cdots=n_k$,
and where the matrix $C$ is the identity, i.e. $c_1=\cdots c_r=1$.

Moreover, the estimated label assignment for unlabeled nodes goes by the following rule
\begin{align*}
 \hat{l}(v_i) = \argmax \{ f^{(1)}_i, \ldots, f^{(k)}_i \}
\end{align*}
Hence, we need to find conditions so that the following inequality holds
\begin{align*}
  f^{(j)}_i < f^{(l(v_i))}_i \qquad \forall \,\, j\neq l(v_i)
\end{align*}

Hence, we consider the following two cases:

\textbf{Case 1: $f^{(r)}_i < f^{(l(v_i))}_i$ for $r > l(v_i)$.}

Let $r^*=l(v_i)$, and $r=r^*+\Delta$. Then, we have
\begin{align*}
 &
 f^{(r)}_i 
 < 
 f^{(l(v_i))}_i
 \Leftrightarrow
 \\
 &
 f^{(r)}_i 
 < 
 f^{(r^*)}_i
 \Leftrightarrow
 \\
&
\beta\left(
 (1-r)\frac{1}{\norm{\bchi_{r}}^2}
+
\sum_{j=r+1}^k \frac{1}{\norm{\bchi_j}^2}
\right)
 <
 \beta\left(
 (1-r^*)^2\frac{1}{\norm{\bchi_{r^*}}^2}
+
\sum_{j=r^*+1}^k \frac{1}{\norm{\bchi_j}^2}
\right)
 \Leftrightarrow
 \\
&
 0
 < 
 \beta\left(
 (1-r^*)^2\frac{1}{\norm{\bchi_{r^*}}^2}
-
(1-r)\frac{1}{\norm{\bchi_{r}}^2}
+
\sum_{j=r^*+1}^k \frac{1}{\norm{\bchi_j}^2}
-
\sum_{j=r+1}^k \frac{1}{\norm{\bchi_j}^2}
\right)
 \Leftrightarrow
\\
&
 0
 < 
 \beta\left(
 (1-r^*)^2\frac{1}{\norm{\bchi_{r^*}}^2}
+
(r-1)\frac{1}{\norm{\bchi_{r}}^2}
+
\sum_{j=r^*+1}^k \frac{1}{\norm{\bchi_j}^2}
-
\sum_{j=r^*+\Delta+1}^k \frac{1}{\norm{\bchi_j}^2}
\right)
 \Leftrightarrow
\\
&
 0
 < 
  \beta\left(
 (1-r^*)^2\frac{1}{\norm{\bchi_{r^*}}^2}
+
(r-1)\frac{1}{\norm{\bchi_{r}}^2}
+
\sum_{j=r^*+1}^{r^*+\Delta} \frac{1}{\norm{\bchi_j}^2}
\right)
 \Leftrightarrow
\\
&
 0
 < 
  \beta
\end{align*}


\textbf{Case 2: $f^{(r)}_i < f^{(l(v_i))}_i$ for $r < l(v_i)$.}

Let $r^*=l(v_i)$, and $r^*=r+\Delta$. Then, we have
\begin{align*}
 &
 f^{(r)}_i 
 < 
 f^{(l(v_i))}_i
 \Leftrightarrow
 \\
 &
 f^{(r)}_i 
 < 
 f^{(r^*)}_i
 \Leftrightarrow
 \\
&
  \beta\left(
(1-r^*)\frac{1}{\norm{\bchi_{r^*}}^2}
+
\sum_{j=r^*+1}^k \frac{1}{\norm{\bchi_j}^2}
\right)
 <
   \beta\left(
 (1-r^*)^2\frac{1}{\norm{\bchi_{r^*}}^2}
+
\sum_{j=r^*+1}^k \frac{1}{\norm{\bchi_j}^2}
\right)
 \Leftrightarrow
 \\
&
0
 <
   \beta\left(
 (1-r^*)^2\frac{1}{\norm{\bchi_{r^*}}^2}
 -
 (1-r^*)\frac{1}{\norm{\bchi_{r^*}}^2}
 \right)
 \\
&
0
 <
   \beta\left(
 (1-r^*)^2\frac{1}{\norm{\bchi_{r^*}}^2}
 +
 (r^*-1)\frac{1}{\norm{\bchi_{r^*}}^2}
 \right)
  \Leftrightarrow
\\
&
 0
 < 
  \beta
\end{align*}
All in all, from the two considered cases we can see that
\begin{align*}
  f^{(j)}_i < f^{(l(v_i))}_i \qquad \forall \,\, j\neq l(v_i) \Longleftrightarrow 0<\beta
\end{align*}
In fact,
\begin{align*}
 0
 &
 <\beta 
 \Leftrightarrow 
 \\
 0
 &
 <\frac{1}{1+\mu m_p(\boldsymbol{\rho_\epsilon})} - \frac{1}{1+\mu(1+\epsilon)}
 \Leftrightarrow
 \\
 \frac{1}{1+\mu(1+\epsilon)}
 &
 <\frac{1}{1+\mu m_p(\boldsymbol{\rho_\epsilon})}
 \Leftrightarrow
 \\
 1+\mu m_p(\boldsymbol{\rho_\epsilon})
 &
 <1+\mu(1+\epsilon)
 \Leftrightarrow
 \\
 m_p(\boldsymbol{\rho_\epsilon})
 &
 <1+\epsilon
\end{align*}
which is the desired condition.
\end{proof}

%

\section{Proof of Corollary~\ref{corollary:limit_cases}}\label{appendix:corollary:limit_cases}

 \begin{corollary}\label{corollary:limit_cases--supp}
   Let $E(\mathbb{G})$ be an expected multilayer graph as in Theorem~\ref{theorem:generalization-equallySized-equallyLabelled}.
 Then,
 \begin{itemize}[topsep=-3pt,leftmargin=*]\setlength\itemsep{-3pt}
 \item For $p\to\infty$, the classification error is zero if and only if $\q^{(t)}<\p^{(t)}$ for \textit{all} $t=1,\ldots,T$.
 \item For $p\!\to\! -\infty$,  the classification error is zero if and only there exists a $t\!\in\!\{1,\ldots,T\}$ s.t. $\q^{(t)}<\p^{(t)}$.
\end{itemize}
\end{corollary}
\begin{proof}
 Observe that the limit cases of the scalar power means are
 \begin{align*}
  \lim_{p\to -\infty}m_p(x_1,\ldots,x_T) &= \min\{ x_1,\ldots,x_T \} \\
  \lim_{p\to +\infty}m_p(x_1,\ldots,x_T) &= \max\{ x_1,\ldots,x_T \} \\
 \end{align*}
 Applying this to condition 
 \begin{align*}
  m_p(\boldsymbol{\rho_\epsilon}) < 1+\epsilon\, ,
\end{align*}
 where $(\boldsymbol{\rho_\epsilon})_t = 1-(\p^{(t)} - \q^{(t)})/(\p^{(t)} + (k-1)\q^{(t)})+\epsilon$, and $t=1,\ldots,T$
 yields the desired result.

\end{proof}

\section{Proof of Corollary~\ref{corollary:contention}}\label{appendix:corollary:contention}
\begin{corollary}\label{corollary:contention--supp}
 Let $E(\mathbb{G})$ be an expected multilayer graph as in Theorem~\ref{theorem:generalization-equallySized-equallyLabelled}.
 Let $p\leq q$. If $\mathcal{L}_q$ has a zero-classification error, then $\mathcal{L}_p$ has a zero-classification error.
\end{corollary}
\begin{proof}
 By Theorem~\ref{theorem:mp_monotone} we have that if $p\leq q$ then $m_p(x_1,\ldots,x_T) \leq m_p(x_1,\ldots,x_T)$.
 Therefore, applying this to our case we can see that
 \begin{align*}
  m_p(\boldsymbol{\rho_\epsilon}) \leq m_q(\boldsymbol{\rho_\epsilon}) < 1+\epsilon
 \end{align*}
A zero test classification error with parameter $q$ is achieved if and only if $m_q(\boldsymbol{\rho_\epsilon})<1+\epsilon$,
hence we can see that zero test classification error with parameter $p$ is achieved if it is achieved with parameter $q$ and $p\leq q$.
 
\end{proof}


\section{General version of Theorem~\ref{theorem:unequalLabels:epsilon-goes-to-zero}}\label{appendix:theorem:arbitrary-labelled-datasets--sup}
\begin{theorem}\label{theorem:arbitrary-labelled-datasets--sup}
 Let $E(\mathbb{G})$ be the expected multilayer graph with $T$ layers following the multilayer SBM with
 two 
 classes $\mathcal{C}_1,\mathcal{C}_2$
 of equal size 
 and parameters $\left(\p^{(t)},\q^{(t)}\right)_{t=1}^{T}$. 
 Let $n_1,n_2$ nodes from classes 
 $\mathcal{C}_1,\mathcal{C}_2$
 be labeled, respectively. 
 Let $\mu=1$.
 Then, a zero test classification error is achieved if and only if 
\begin{align*}
 m_p(\boldsymbol{\rho_\epsilon})< \min\left\{ \frac{(n_1+n_2)((1+\epsilon)^2+1)-2n_2 }{2n_2 + (n_1+n_2)\epsilon} , \frac{(n_1+n_2)((1+\epsilon)^2+1)-2n_1 }{2n_1 + (n_1+n_2)\epsilon} \right\}
\end{align*}
 where $(\boldsymbol{\rho_\epsilon})_l = 1-(\p^{(l)} - \q^{(l)})/(\p^{(l)} + (k-1)\q^{(l)})+\epsilon$, and $l=1,\ldots,T$.
\end{theorem}

\begin{proof}

The proof of this theorem builds on top of Lemma~\ref{lemma:solutionMatrix-F-supp}, where the entries of the solution matrix 
$F = (f^{(1)}, \ldots, f^{(k)})$ are described, where
$$
 f^{(r)} = \argmin_{f\in\R^n} \|f-CY^{(r)}\|^2 
 +  
 \mu
 f^T \mathcal{L}_pf
$$
Let $l(v_i)$ be the label of node $v_i$, i.e. $l(v_i)=r$ if and only if $v_i\in\mathcal{C}_r$.
According to Lemma~\ref{lemma:solutionMatrix-F-supp}
the entries of matrix $F$ for unlabeled nodes are such that
 \begin{itemize}[topsep=-3pt,leftmargin=*]\setlength\itemsep{-3pt}
 \item If $r < l(v_i)$, then
 \begin{align*}
 f_i^{(r)} = 
 c_r
 \frac{n_r}{n}\alpha
 +
 c_r
 n_r
   \beta
\left(
 (1-l(v_i))\frac{1}{\norm{\bchi_{l(v_i)}}^2}
+
\sum_{j=l(v_i)+1}^k \frac{1}{\norm{\bchi_j}^2}
\right)
\end{align*}
 \item If $r > l(v_i)$, then
\begin{align*}
 f_i^{(r)} = 
 c_r
 \frac{n_r}{n}\alpha
 +
 c_r
 n_r
   \beta
\left(
 (1-r)\frac{1}{\norm{\bchi_{r}}^2}
+
\sum_{j=r+1}^k \frac{1}{\norm{\bchi_j}^2}
\right)
\end{align*}
 \item If $r = l(v_i)$, then
\begin{align*}
 f_i^{(r)} = 
 c_r
 \frac{n_r}{n}\alpha
 +
 c_r
 n_r
   \beta
\left(
 (1-r)^2\frac{1}{\norm{\bchi_{r}}^2}
+
\sum_{j=r+1}^k \frac{1}{\norm{\bchi_j}^2}
\right)
\end{align*}
\end{itemize}
where 
$\alpha=\frac{1}{1+\mu\epsilon} - \frac{1}{1+\mu(1+\epsilon)}$, and 
$\beta=\frac{1}{1+\mu m_p(\boldsymbol{\rho_\epsilon})} - \frac{1}{1+\mu(1+\epsilon)}$.

%

Observe that the case here considered corresponds to the case with two classes, i.e. $k=2$ with equal size classes $\mathcal{C}_1$ and $\mathcal{C}_2$ where the amount of labeled data per class is $n_1$ and $n_2$, respectively, 
with the matrix $C$ as the identity, i.e. $c_1=c_2=1$, and regularization parameter $\mu = 1$.

Moreover, the estimated label assignment for unlabeled nodes goes by the following rule
\begin{align*}
 \hat{l}(v_i) = \argmax \{ f^{(1)}_i, f^{(2)}_i \}
\end{align*}
Hence, we need to find conditions so that the following inequality holds
\begin{align*}
  f^{(j)}_i < f^{(l(v_i))}_i \qquad \forall \,\, j\neq l(v_i)
\end{align*}

Let 
$l(v_i)=1 \Leftrightarrow v_i\in\mathcal{C}_1$, and
$l(v_i)=2 \Leftrightarrow v_i\in\mathcal{C}_2$.
A quick computation following Lemma~\ref{lemma:solutionMatrix-F-supp} yields
\begin{itemize}
 \item 
 $
f_i^{(1)} = \frac{n_1}{n}\alpha + n_1\beta(\frac{1}{\norm{\bchi_2}^2})
$
for $v_i\in\mathcal{C}_1$, i.e. $l(v_i)=1$
 \item 
 $
f_i^{(1)} = \frac{n_1}{n}\alpha - n_1\beta(\frac{1}{\norm{\bchi_2}^2})
$
for $v_i\in\mathcal{C}_2$, i.e. $l(v_i)=2$
 \item 
 $
f_i^{(2)} = \frac{n_2}{n}\alpha - n_2\beta(\frac{1}{\norm{\bchi_2}^2})
$
for $v_i\in\mathcal{C}_1$, i.e. $l(v_i)=1$
 \item 
 $
f_i^{(2)} = \frac{n_2}{n}\alpha + n_2\beta(\frac{1}{\norm{\bchi_2}^2})
$
for $v_i\in\mathcal{C}_2$, i.e. $l(v_i)=2$
\end{itemize}
Observing that $\norm{\bchi_2}^2=n$
these conditions can be rephrase as follows
\begin{align*}
 f^{(1)} 
 &=
 \frac{n_1}{n} 
 \left( \left( \alpha+\beta \right) \one_{\mathcal{C}}  + \left( \alpha-\beta \right) \one_{\mathcal{\overline{C}}} \right)
 \\
 f^{(2)} 
 &=
 \frac{n_2}{n} 
 \left( \left( \alpha-\beta \right) \one_{\mathcal{C}}  + \left( \alpha+\beta \right) \one_{\mathcal{\overline{C}}} \right)
\end{align*}
Hence, the conditions for correct label assignment of unlabeled nodes are
\begin{align*}
 n_1\left( \alpha+\beta \right) > n_2\left( \alpha-\beta \right) 
 \text{ and }
 n_2\left( \alpha+\beta \right) > n_1\left( \alpha-\beta \right)
\end{align*}
Let 
$
\Omega_{11}=\frac{1}{1+\epsilon},
\Omega_{22}=\frac{1}{1+m_p(\boldsymbol{\rho_\epsilon})}
$, and
$
\omega = \frac{1}{1+(1+\epsilon)}.
$
Then, 
$
\alpha = \Omega_{11}-\omega,
$
and
$
\beta = \Omega_{22}-\omega
$.

By studying the first condition we observe
\begin{align*}
 n_1\left( \alpha+\beta \right) 
 &> n_2\left( \alpha-\beta \right) 
 \Leftrightarrow
 \\
 n_1\left( \Omega_{11}-\omega +\Omega_{22}-\omega \right) 
 &> n_2\left( \Omega_{11}-\omega-(\Omega_{22}-\omega) \right) 
 \Leftrightarrow
 \\
 n_1\left( \Omega_{11}+\Omega_{22}-2\omega\right) 
 &> n_2( \Omega_{11}-\Omega_{22})  
 \Leftrightarrow
 \\
 (n_1-n_2)\Omega_{11} + (n_1+n_2)\Omega_{22} 
 &> 2n_1\omega
 \Leftrightarrow
 \\
 \Omega_{22} 
 &> \frac{1}{n_1+n_2} \left( 2n_1\omega - (n_1-n_2)\Omega_{11} \right)
 \Leftrightarrow
 \\
 \frac{1}{1+m_p(\boldsymbol{\rho_\epsilon})} 
 &> \frac{1}{n_1+n_2} \left( 2n_1\frac{1}{1+(1+\epsilon)} - (n_1-n_2)\Omega_{11} \right)
 \Leftrightarrow
 \\
 \frac{1}{1+m_p(\boldsymbol{\rho_\epsilon})} 
 &> \frac{1}{n_1+n_2} \left( 2n_1\frac{1}{2+\epsilon} - (n_1-n_2)\frac{1}{1+\epsilon} \right)
 \Leftrightarrow
 \\
 \frac{1}{1+m_p(\boldsymbol{\rho_\epsilon})} 
 &> \frac{1}{n_1+n_2} \left( \frac{2n_2 + (n_1+n_2)\epsilon}{(2+\epsilon)(1+\epsilon)} \right)
 \Leftrightarrow
 \\
 1+m_p(\boldsymbol{\rho_\epsilon})
 &< (n_1+n_2) \left( \frac{(2+\epsilon)(1+\epsilon)}{2n_2 + (n_1+n_2)\epsilon} \right)
 \Leftrightarrow
 \\
 m_p(\boldsymbol{\rho_\epsilon})
 &< (n_1+n_2) \left( \frac{(2+\epsilon)(1+\epsilon)}{2n_2 + (n_1+n_2)\epsilon} \right)-1
 \Leftrightarrow
 \\
 m_p(\boldsymbol{\rho_\epsilon})
 &< (n_1+n_2) \left( \frac{(2+\epsilon)(1+\epsilon)-(2n_2 + (n_1+n_2)\epsilon)}{2n_2 + (n_1+n_2)\epsilon} \right)
 \Leftrightarrow
 \\
 m_p(\boldsymbol{\rho_\epsilon})
 &< \frac{(n_1+n_2)((2+\epsilon)(1+\epsilon)-\epsilon)-2n_2 }{2n_2 + (n_1+n_2)\epsilon} 
 \Leftrightarrow
 \\
 m_p(\boldsymbol{\rho_\epsilon})
 &< \frac{(n_1+n_2)((1+\epsilon)^2+1)-2n_2 }{2n_2 + (n_1+n_2)\epsilon} 
 \Leftrightarrow
 \\
\end{align*}
The corresponding condition for $\mathcal{C}_2$ can be obtained in a similar way, yielding
\begin{align*}
 m_p(\boldsymbol{\rho_\epsilon}) < \frac{(n_1+n_2)((1+\epsilon)^2+1)-2n_1 }{2n_1 + (n_1+n_2)\epsilon} 
\end{align*}
Hence, both conditions hold if and only if
\begin{align*}
 m_p(\boldsymbol{\rho_\epsilon}) = m_p(\boldsymbol{\rho_\epsilon})< \min\left\{ \frac{(n_1+n_2)((1+\epsilon)^2+1)-2n_2 }{2n_2 + (n_1+n_2)\epsilon} , \frac{(n_1+n_2)((1+\epsilon)^2+1)-2n_1 }{2n_1 + (n_1+n_2)\epsilon} \right\}
\end{align*}
\end{proof}
%
%

\section{Proof of Theorem~\ref{theorem:unequalLabels:epsilon-goes-to-zero}}\label{appendix:theorem:unequalLabels:epsilon-goes-to-zero}

\begin{theorem}\label{theorem:epsilon-goes-to-zero--sup}
 Let $E(\mathbb{G})$ be the expected multilayer graph with $T$ layers following the multilayer SBM with
 two 
 classes $\mathcal{C}_1,\mathcal{C}_2$ 
 of equal size 
 and parameters $\left(\p^{(t)},\q^{(t)}\right)_{t=1}^{T}$. 
 Let $n_1,n_2$ nodes from clusters $\mathcal{C}_1,\mathcal{C}_2$ be labeled, respectively. 
 Let $\mu=1$.
 Then, a zero test classification error is achieved if 
\begin{align*}
 m_p(\boldsymbol{\rho_\epsilon})< \min\left\{ \frac{n_1}{n_2}, \frac{n_2}{n_1} \right\}
\end{align*}
 where $(\boldsymbol{\rho_\epsilon})_l = 1-(\p^{(l)} - \q^{(l)})/(\p^{(l)} + (k-1)\q^{(l)})+\epsilon$, and $l=1,\ldots,T$.
\end{theorem}
\begin{proof}
We first analyze the first condition of the right hand side of Theorem~\ref{theorem:arbitrary-labelled-datasets--sup}.
Let $g(\epsilon) = \frac{(n_1+n_2)((1+\epsilon)^2+1)-2n_2 }{2n_2 + (n_1+n_2)\epsilon}$.
Then,
 \begin{align*}
  g(0)
  =
  \frac{2(n_1+n_2)-2n_2 }{2n_2}
  =
  \frac{n_1 }{n_2}
 \end{align*}
 Moreover, it is clear that $g$ is monotone,  as it is quadratic on $\epsilon$ on the numerator and linear on the denominator, and hence $g(0)<g(\epsilon)$.
 
 A similar procedure with the second condition of the right hand side of Theorem~\ref{theorem:arbitrary-labelled-datasets--sup} leads to the condition $\frac{n_2}{n_1}$, leading to the desired result.
\end{proof}

\section{Proof of Theorem~\ref{theorem:generalization-equallySized-NOTequallyLabelled}}\label{appendix:theorem:generalization-equallySized-NOTequallyLabelled}
\begin{theorem}\label{theorem:generalization-equallySized-NOTequallyLabelled--supp}
 Let $E(\mathbb{G})$ be the expected multilayer graph with $T$ layers following the multilayer SBM with
 $k$ classes $\mathcal{C}_1, \ldots, \mathcal{C}_k$ of equal size and parameters $\left(\p^{(t)},\q^{(t)}\right)_{t=1}^{T}$. 
 Let $n_1,\ldots,n_k$ be the number of nodes per class be labeled.
 Let $C\in\R^{n}$ be a cost vector where with $C_i = n/n_r$ for $v_i\in\mathcal{C}_r$.
 Then, a zero test classification error is achieved if and only if 
 $$m_p(\boldsymbol{\rho_\epsilon}) < 1+\epsilon\, ,$$
 where $(\boldsymbol{\rho_\epsilon})_l = 1-(\p^{(l)} - \q^{(l)})/(\p^{(l)} + (k-1)\q^{(l)})+\epsilon$, and $l=1,\ldots,T$.
\end{theorem}
\begin{proof}
 The proof is similar to the one of Theorem~\ref{theorem:generalization-equallySized-equallyLabelled} (see Section~\ref{appendix:Proof:theorem:generalization-equallySized-equallyLabelled}). 
 The only change is in the terms
 $c_r \frac{n_r}{n}$. Since we have by definition that $c_r = \frac{n}{n_r}$ we have that
 $c_r \frac{n_r}{n}=1$, leading to the conditions obtained by Theorem~\ref{theorem:generalization-equallySized-equallyLabelled}.
\end{proof}

\section{A scalable matrix-free method for the linear system $(I+\lambda L_p)f=Y$}\label{appendix:numerics}
Computing the generalized matrix mean of $T$ positive definite matrices $A_1,\dots,A_T$ requires to compute $T+1$ matrix functions: $A_1^p, \dots, A_T^p$ and $(\sum_i A_i^p)^{1/p}$. Typically, the matrices $A_i^p$ are full even though each $A_i$ is a sparse matrix and so, computing $L_p$ explicitly is unfeasible if the $A_i$'s have large dimensions. Given a vector $\y$ and a negative integer $p$, here we propose a matrix-free method for solving the linear system $(I+\lambda L_p)^{-1}\y$. The method exploits the sparsity of the Laplacians of each layer and is matrix-free in the sense that it requires only to compute the matrix-vector product $A_i\times vector$, without requiring to store the matrices $A_i$ themselves nor to compute any matrix function $A_i^p$ explicitly. Thus, when the layers are sparse, the method scales to large datasets. 
Below we give further details about the method presented in the short version of the paper. We present the method for a general set of positive definite matrices $A_1,\dots, A_T$, and for a general vector $\y$, for the sake of generality.

Let $S_p = A_1^p + \dots + A_T^p$, $\phi:\mathbb C \to \mathbb C$ be the complex function $\phi(z)=z^{1/p}$ and let $L_p$ be the matrix function $L_p= T^{-1/p}\phi(S_p)$.  
The proposed method essentially transforms the original problem into a series of subproblems which thus allow us to solve the linear system $(I+\lambda L_p)^{-1}\y$ by solving several different linear systems with $A_i$ as coefficient matrices. The method consists of three main nested inner--steps which we present below.

\textbf{1.} First, we solve the linear system $(I+\lambda L_p)^{-1}\y$ by a Krylov method (GMRES in our case \cite{saad1986gmres}). At each iteration, this method projects the problem into the Krylov subspace spanned by $\{\y, \lambda L_p\y, (\lambda L_p)^2\y, \dots,  (\lambda L_p)^{h}\y\}$. If $\kappa =\lambda_{\max}(L_p)/\lambda_{\min}(L_p)$, then the method converges as 
$$
O\left(\Big(\frac{\kappa^2-1}{\kappa^2}\Big)^{h/2}\right)\, .
$$
Thus, if $L_p$ is well conditioned, a relatively small $h$ is required. In order to build the appropriate Krylov subspace, at each iteration we need to efficiently perform one matrix--vector product $L_p\y$. 

\textbf{2.} Second, in order to compute  $L_p\y = T^{-1/p}\phi(S_p)\y$ we use the Cauchy integral form of the function $\phi$, transformed via a conformal map,  to approximate $\phi(S_p)$ via the trapezoidal rule, as proposed in \cite{hale2008computing}.
 Let $m,M>0$ be such that the interval $[m,M]$ contains the whole spectrum of $S_p$ and let $t_1,\dots, t_N$ be $N$ equally spaced contour points to be used in the trapezoidal rule. As $\phi$ has a singularity at $z = 0$ but just a brunch cut on $(-\infty, 0)$,  we can approximate $\phi(S_p)\y$ via \cite{hale2008computing}
\begin{equation*}
\phi_N(S_p)\y = \frac{-8K(mM)^{1/4}}{\pi Nk}S_p\,\, \mathrm{Im}\left\{\sum_{i=1}^N \frac{\phi(z_i^2)c_id_i}{z_i(k^{-1}-s_i)^2} (z_i^2I-S_p)^{-1}\y  \right\} 
\end{equation*}
where $\mathrm{Im}$ denotes the imaginary part, $k = \big((M/m)^{1/4}-1\big)/\big((M/m)^{1/4}+1\big)$, $K$ is the value of the complete elliptic  integral of the first kind, evaluated at $ke^2$,   $s_i=\mathrm{sn}(t_i)$ is the Jacobi elliptic sine function evaluated on the $i$-th contour point $t_i$,  and
$$
z_i = (mM)^{1/4} \left(\frac{k^{-1}+s_i}{k^{-1}-s_i}\right), \quad  c_i = \sqrt{1-s_i^2}, \quad  d_i = \sqrt{1-k^2s_i^2}\, ,
$$
for $i=1,\dots,N$. This approximation converges geometrically as the number of points increases. Precisely, it holds
\begin{equation*}
\|\phi(S_p)\y - \phi_N(S_p)\y\| = O(e^{-2\pi^2N/(\ln(M/m)+6)})\, .
\end{equation*}
Thus, the computation of $\phi(S_p)\y$ is reduced to $N$ linear systems $(z_i^2I-S_p)^{-1}\y$. Note that these systems are independent and thus they can be solved in parallel.

\textbf{3.} Finally, in order to solve the linear system $(zI-S_p)^{-1}\y$ we employ again a Krylov method. In order to build the Krylov space for $(zI-S_p)$ and $\y$ we need to efficiently perform one multiplication $S_p$ times a vector per iteration. As $S_p = \sum_{i=1}^TA_i^{p} = \sum_{i=1}^T (A_i^{-1})^{|p|}$, this problem reduces to solving $q$ linear systems with $A_i$ as coefficient matrix, for $i=1,\dots,T$. As the matrices $A_i$ are assumed sparse and positive definite, we can very efficiently solve each of these systems via the Preconditioned Conjugate Gradient method with an incomplete Cholesky preconditioner.

The pseudocode for the proposed algorithm is presented in Algorithms \ref{alg:main}--\ref{alg:Mp}. 

\noindent\fbox{
\begin{minipage}{.97\textwidth}
\RestyleAlgo{plain}	
$\,$\hspace{1em}\begin{algorithm}[H]
\DontPrintSemicolon
{\KwIn{$A_1,\dots,A_T,p,\y, \lambda$}}
Compute preconditioners $P_1, \dots, P_T$ for $A_1,\dots,A_T$\;
Compute estimates for $m$ and $M$ such that $\mathrm{eigenvalues}(S_p)\subseteq [m,M]$\;
Choose number of contour points $N$\;
Compute contour coefficients $z_i,s_i,K,k$\;
Solve $(I+\lambda L_p)^{-1}\y$ with GMRES, using Alg.\ref{alg:Lp} as subroutine\;
{\KwOut{$\u = (I+\lambda L_p)^{-1}\y$}}
\vspace{.5em}
\caption{Solve $(I+\lambda L_p)^{-1}\y$}
\label{alg:main}
\end{algorithm}
\end{minipage}
}

\noindent\fbox{
\begin{minipage}{.5\textwidth}
\RestyleAlgo{plain}	
$\,$\hspace{1em}\begin{algorithm}[H]
\DontPrintSemicolon
{\KwIn{$A_1,\dots,A_T,p,\y,N,m,M$, contour coefficients $z_i,s_i,c_i,d_i,k,K$}}
$\u \gets S_p\y$, using Alg.\ref{alg:Mp}\;
\For{$i=1,\dots,N$}{
$\u \gets \mathrm{solve}(z_iI-S_p,\y)$ with GMRES, using Alg.\ref{alg:Mp} as subroutine\;
$\u \gets \frac{(z_i^2)^{1/p}c_id_i}{z_i(k^{-1}-s_i)^2}\u$\;
$\u_{k+1} = \|\v_{k+1}\|_q^{1-q}|\v_{k+1}|^{q-2} \v_{k+1}$\;
}
$\u \gets \frac 1 {T^{1/p}} \frac{-8K(mM)^{1/4}}{\pi Nk} \mathrm{Im}(\u)$\;
\vspace{.5em}
\KwOut{$\u = L_p\y$}
\vspace{.5em}
\caption{Multiply $L_p$ times a vector}
\label{alg:Lp}
\end{algorithm}
\end{minipage}
\begin{minipage}{.465\textwidth}
\RestyleAlgo{plain}	
$\,$\hspace{1em}\begin{algorithm}[H]
\DontPrintSemicolon
{\KwIn{$A_1, \dots, A_T,P_1,\dots,P_T,\y$}}
\vspace{1em}
\For{$k=1,\dots,T$}{
$\u \gets \u + \mathrm{solve}(A_i^{|p|},\y)$ using CG preconditioned with $P_i$\;
}
\KwOut{$\u = S_p\y$}
\vspace{.5em}
\caption{Multiply $S_p$ times a vector}
\label{alg:Mp}
\end{algorithm}\vspace{6em}
\end{minipage}
}

\subsection{Implementation details and computational complexity}
Few implementation details are in order: 

The preconditioners $P_i$ can be computed using an incomplete Cholesky factorization. In our test we observe that a \verb!1e-4! threshold is enough to ensure convergence  of Alg.\ref{alg:Mp} to \verb!1e-8! precision in just 2 or 3 iterations. As in our case the $A_i$ are Laplacians, another excellent preconditioner can be obtained using a Combinatorial Multi Grid method (CMG). In our experiments, the CMG preconditioner performed similarly (but slightly worse) than the incomplete Cholesky. 

A precise estimate of $M$ in Alg.\ref{alg:main} step 2 can be obtained using a Krylov eigensolver with Alg.\ref{alg:Mp} as subroutine. As for $m$, since each $A_i^p$ is positive definite and $p$ is a negative integer, a  good estimate can be obtained by exploiting the Weyl's inequality (see e.g.\ \cite{wilkinson1965algebraic})
$$
m = \lambda_{\max}(A_1)^p +\cdots +\lambda_{\max}(A_T)^p \leq \lambda_{\min}(S_p) \, . 
$$
The number of contour points $N$ can be chosen using  the geometric convergence of $\phi_N$. In our experiments, we chose a precision $\tau=$\verb!1e-8! and we set
$$
N = |(\ln(M/m)+6)\ln(\tau)/2\pi^2|\, .
$$
The contour points have been calculated using the code from \cite{driscoll2005algorithm}.

Concerning the computational cost of the method, the following analysis shows that it is proportional to the number of edges in each layer, i.e.\ Alg.\ref{alg:main} scales to large sparse datasets. Let $c(A_i)$ be the cost of multiplying $c(A_i)$ times a vector (which is proportional to the number of nonzeros in $A_i$, i.e.\ the number of edges in the layer $i$ when $A_i$ is the normalized Laplacian of the $i$-th layer).  Let $K_1,K_2,K_3$ be the number of iterations of GMRES,GMRES and PCG in lines 5, 3 and 2 of Algorithms \ref{alg:main}, \ref{alg:Lp} and \ref{alg:Mp}, respectively. Each instance of $\mathrm{solve}(A_i^{|p|},\y)$ in Alg.\ref{alg:Mp} requires $K_3 p \, c(A_i)$ operations per step. So The cost of Alg.\ref{alg:Mp} is roughly $p K_3 \sum_{i=1}^T c(A_i)$. This implies that the cost of Alg.\ref{alg:Lp} is $NK_2K_3 p \,\sum_{i=1}^T c(A_i)$. Therefore, the cost of solving the linear system $(I+\lambda L_p)^{-1}\y$ with Alg.\ref{alg:main} is 
$$
K_1 NK_2K_3 p \big(c(A_1)+\cdots +c(A_T)\big)\, ,
$$
showing that the method scales as the number of nonzeros in each layer, as claimed. It is important to notice that the Algorithm allows for a high level of parallelism. In fact, the computation of the preconditioners $P_i$ at step 1 of Alg.\ref{alg:main}, the \textbf{for} at step 2 of Alg.\ref{alg:Lp} and the \textbf{for} at step 1 of Alg.\ref{alg:Mp} can all be run in parallel.
\section{Analysis on Three Layers with Three Classes}\label{appendix:3Layers}
In this section we give a more detailed exposition of experiments presented in Section~\ref{sec:independentInformation}.
We consider the cases where $\p-\q\in\{0.03, 0.04, \ldots, 0.1\}$ which are depicted in Fig.\ref{fig:3Layers-supplementary}. 
In the $x$-axis we have the amount of labeled nodes and in the $y$- we have the classification error.
We can see that in general there is a trend between the performance of our proposed method (colorful curves) and state of the art methods (black curves). We can see that the larger the gap $\p-\q$ the larger the difference is between our proposed method and state of the art methods.
Moreover, one can see that the smaller the value of $p$ the better the performance of our proposed method. Moreover, there is a set of state of the art methods that do not improve their performance with larger amounts of labeled nodes. Yet, one can observe that there are three methods from the state of the art that perform close to our methods: TLMV, ZooBP and AGML, which performs similarly to our method $L_1$ (i.e. the arithmetic mean of Laplacians).

\begin{figure*}[t]
 \centering
 \centering
 \includegraphics[width=0.50\linewidth, clip,trim=0 190 30 403]{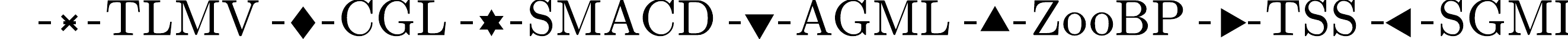}
 \includegraphics[width=0.245\linewidth, clip,trim=179 192 195 391]{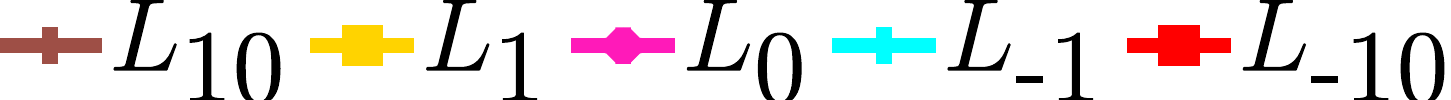}
 \vspace{20pt}

 \hfill 
 \begin{subfigure}[b]{0.23\textwidth}
 \includegraphics[angle=-00,width=1\textwidth,trim=160 100 10 90]{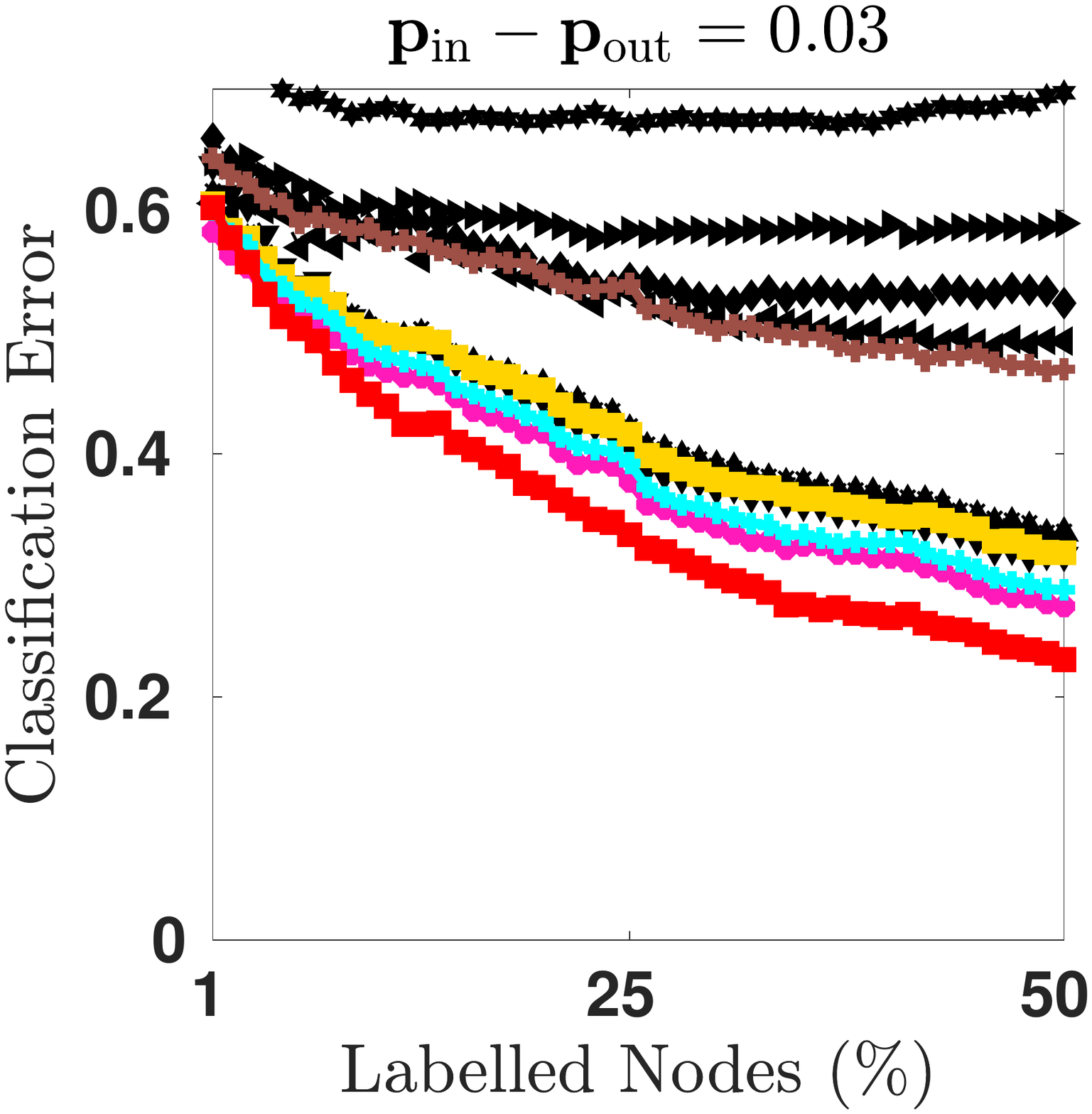}
 \end{subfigure}
 \hfill
 \begin{subfigure}[b]{0.23\textwidth}
 \includegraphics[angle=-00,width=1\textwidth,trim=160 100 10 90]{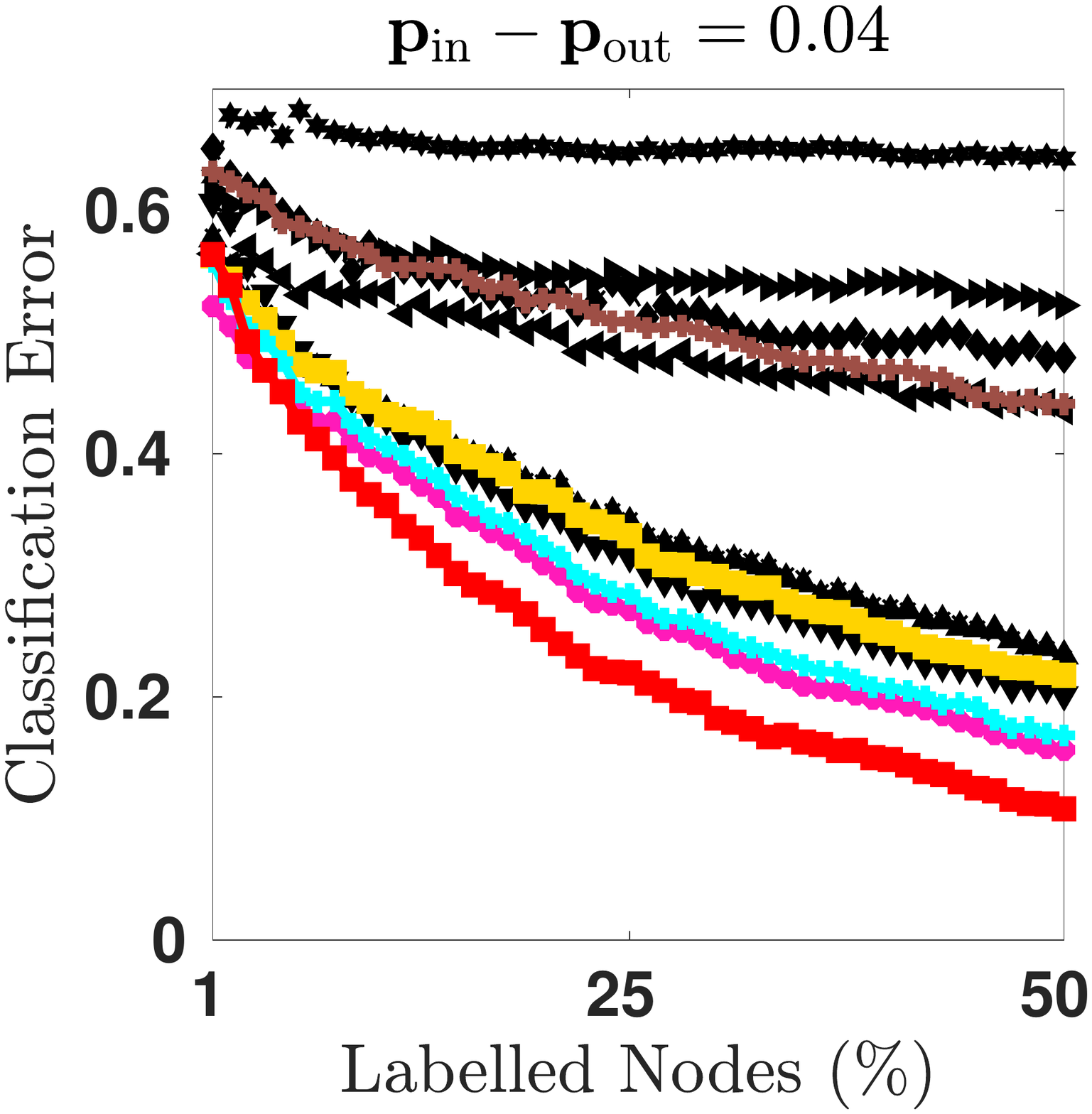}
 \end{subfigure}
 \hfill
 \begin{subfigure}[b]{0.23\textwidth}
 \includegraphics[angle=-00,width=1\textwidth,trim=160 100 10 90]{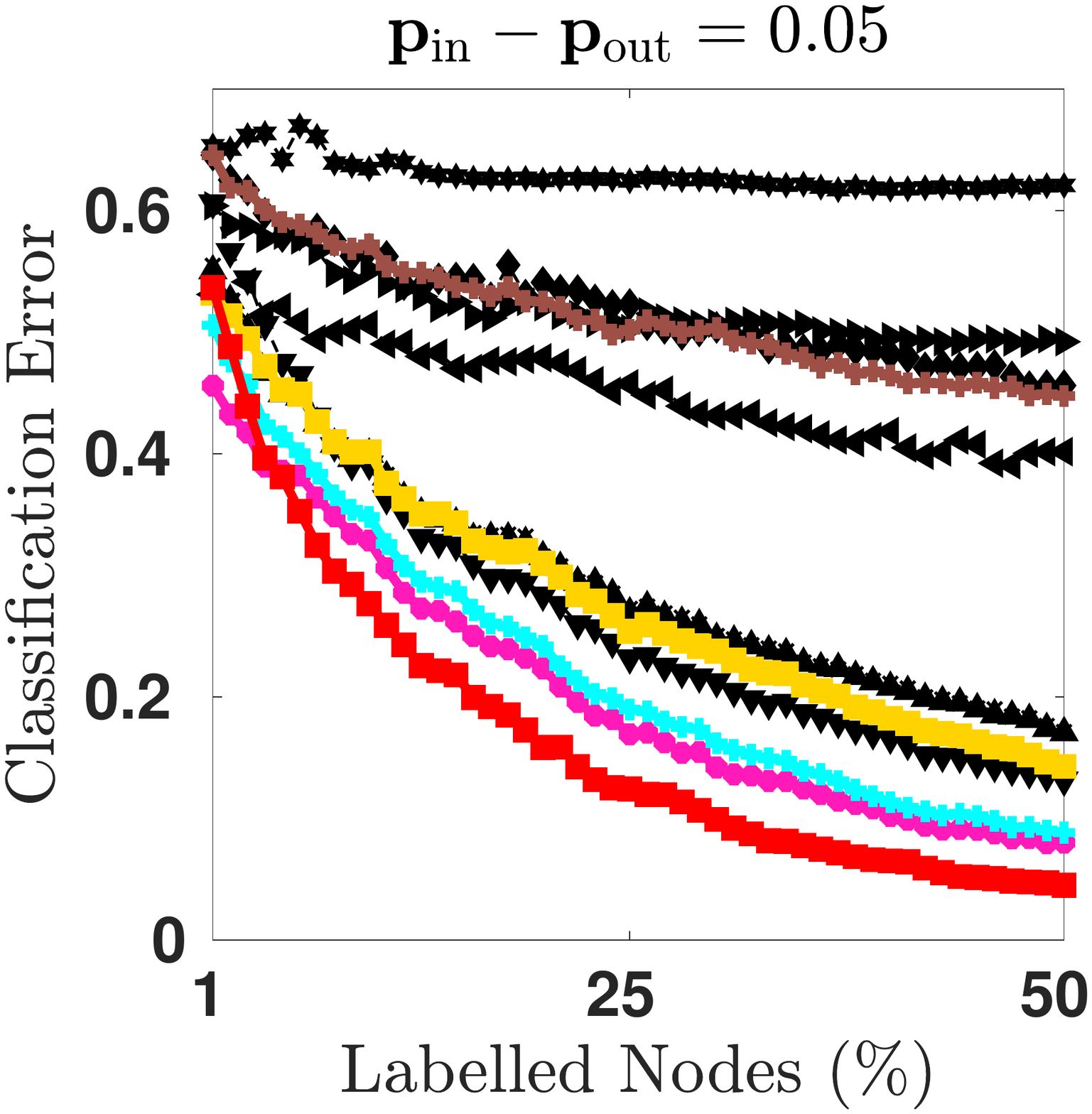}
 \end{subfigure}
 \hfill
 \begin{subfigure}[b]{0.23\textwidth}
 \includegraphics[angle=-00,width=1\textwidth,trim=160 100 10 90]{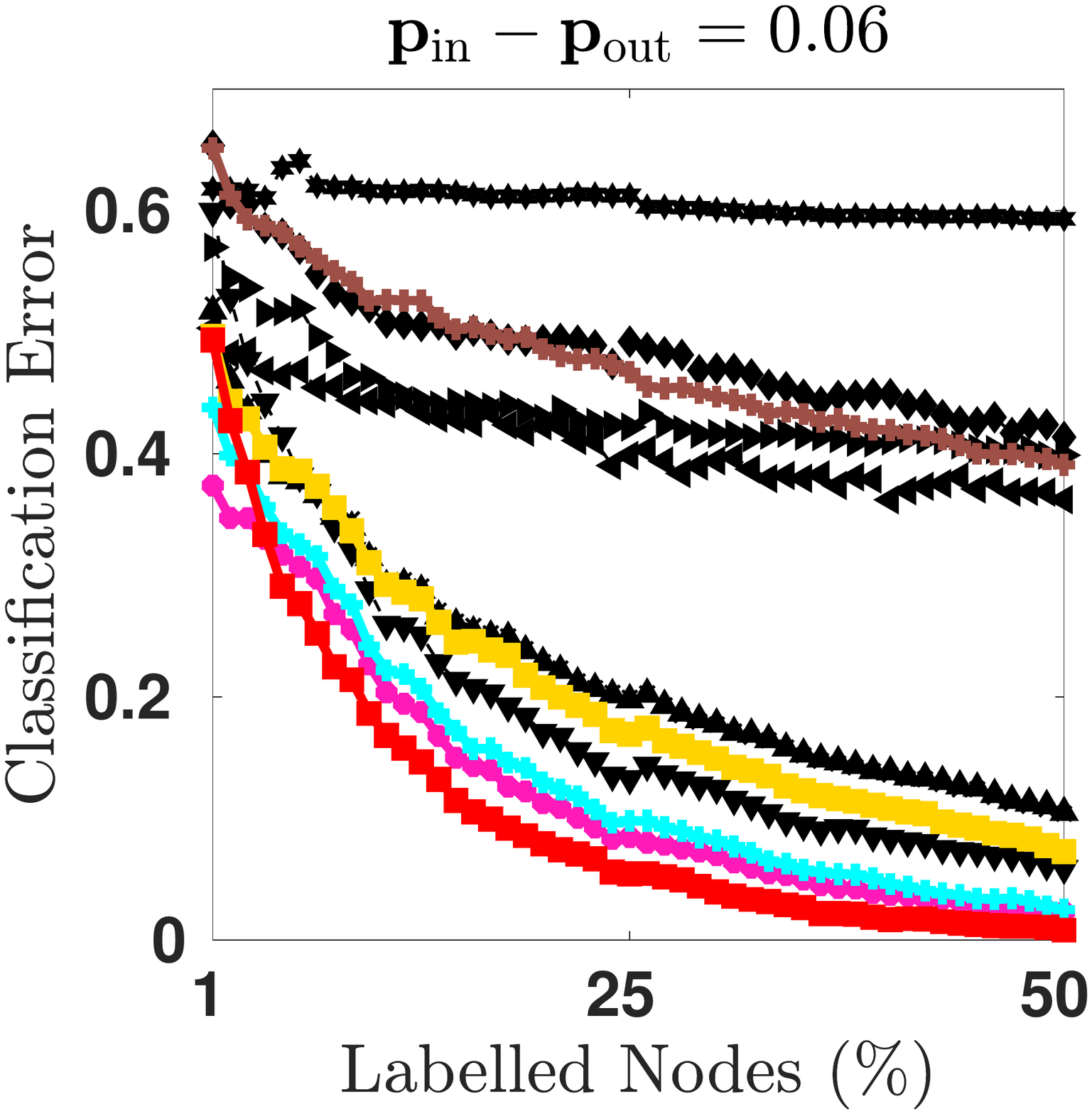}
 \end{subfigure}
 \hfill

 \vspace{25pt}
 
 \hfill
 \begin{subfigure}[b]{0.23\textwidth}
 \includegraphics[angle=-00,width=1\textwidth,trim=160 100 10 90]{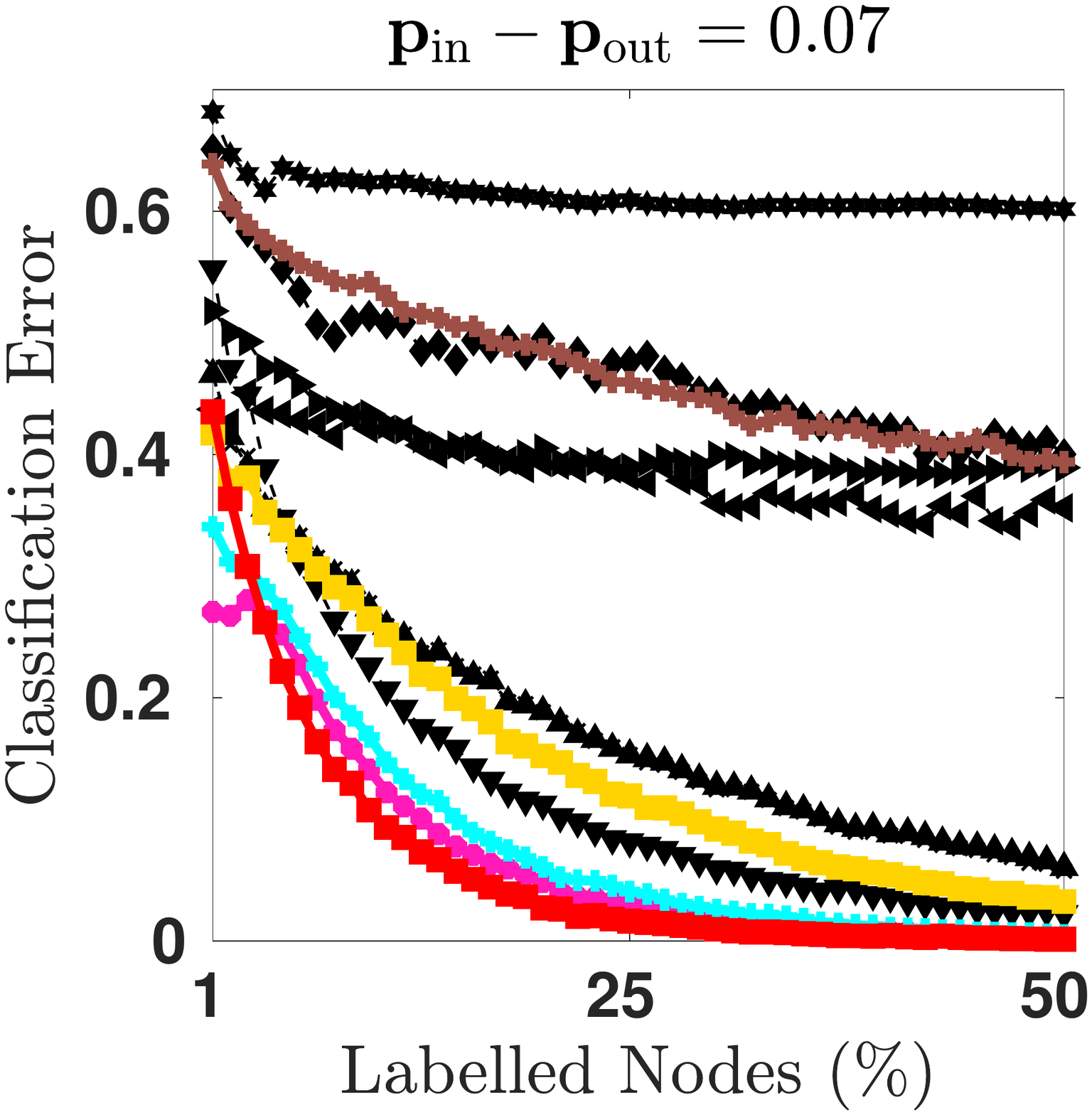}
 \end{subfigure}
 \hfill
 \begin{subfigure}[b]{0.23\textwidth}
 \includegraphics[angle=-00,width=1\textwidth,trim=160 100 10 90]{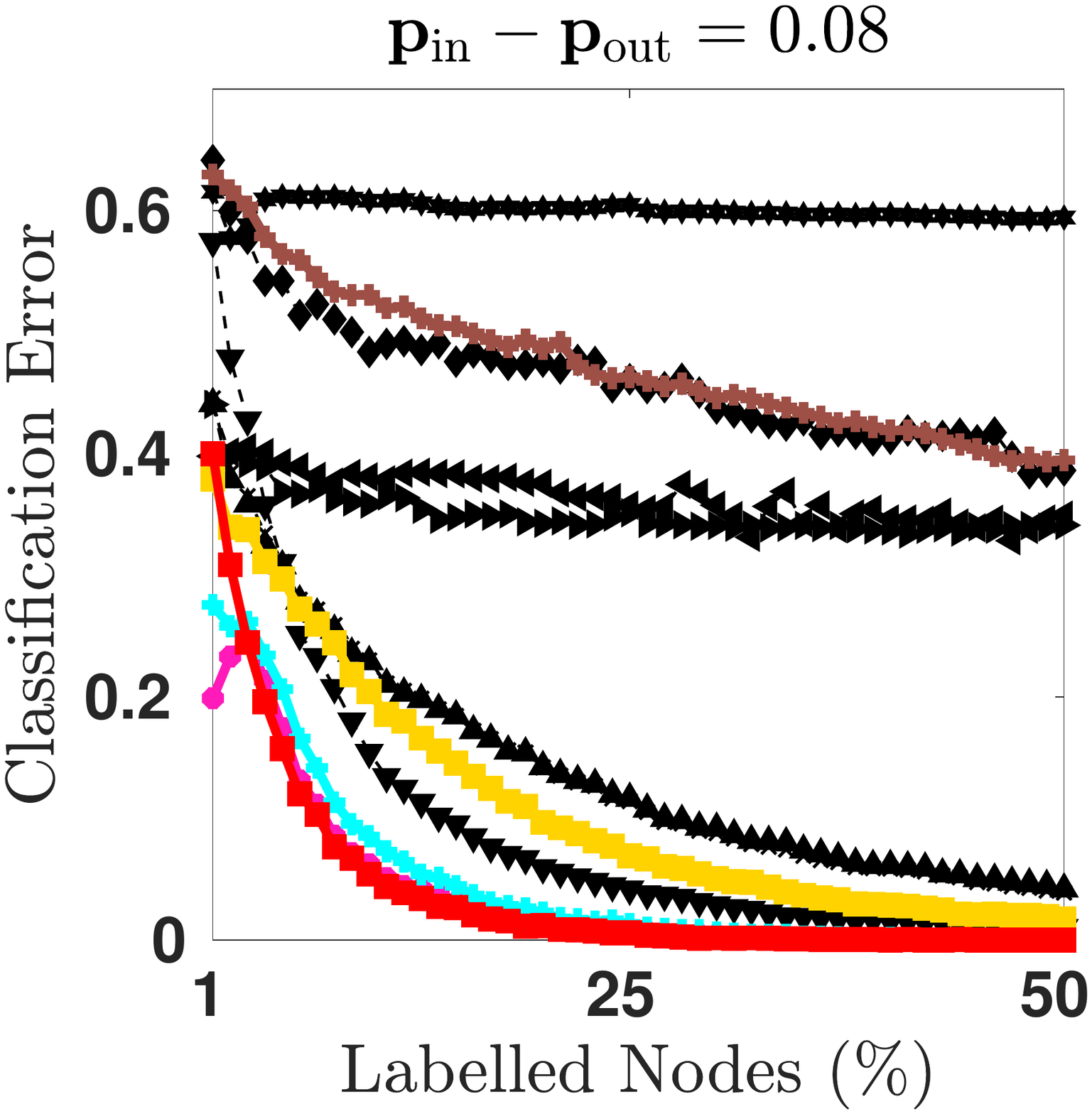}
 \end{subfigure}
 \hfill
 \begin{subfigure}[b]{0.23\textwidth}
 \includegraphics[angle=-00,width=1\textwidth,trim=160 100 10 90]{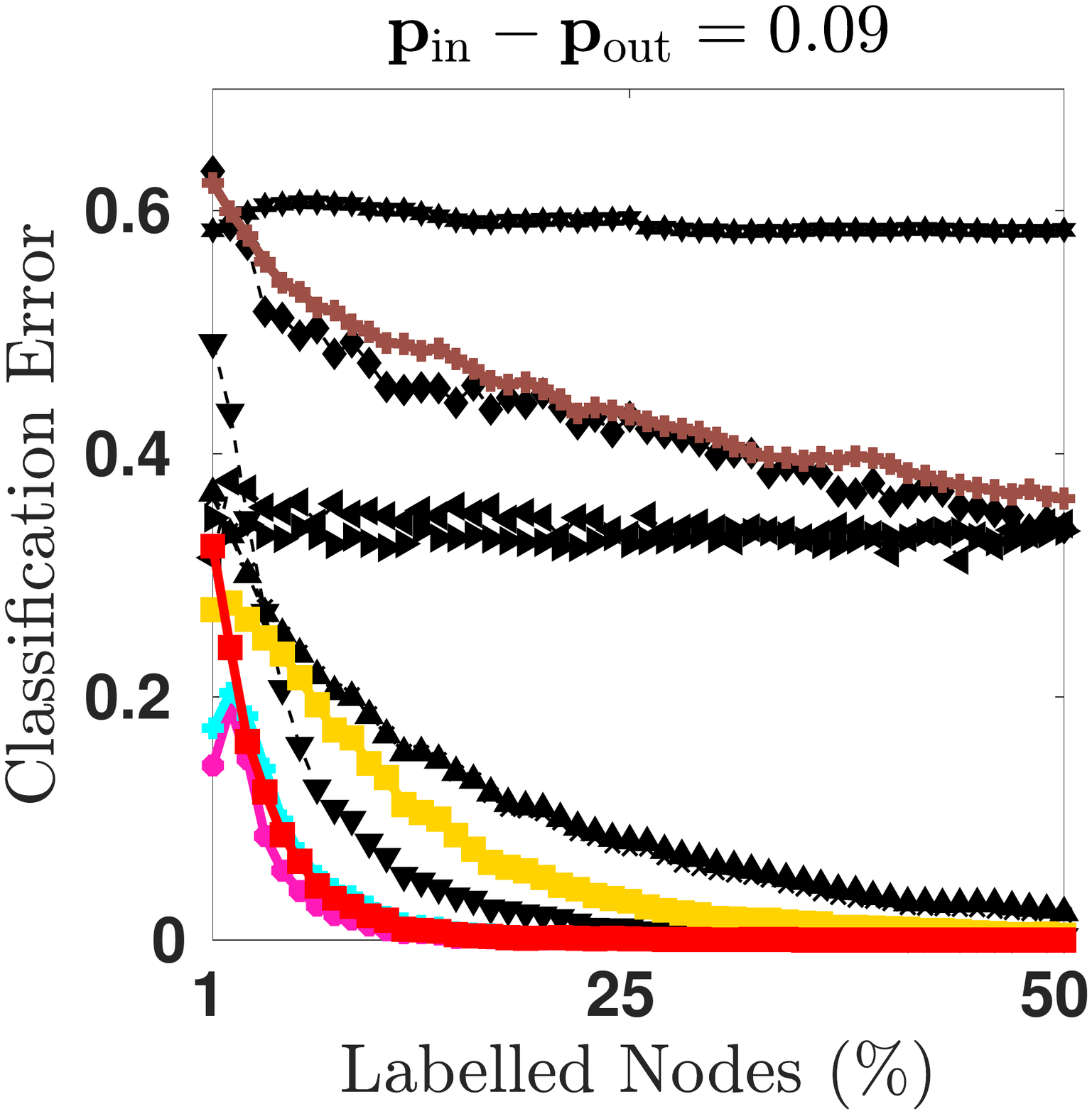}
 \end{subfigure}
 \hfill
 \begin{subfigure}[b]{0.23\textwidth}
 \includegraphics[angle=-00,width=1\textwidth,trim=160 100 10 90]{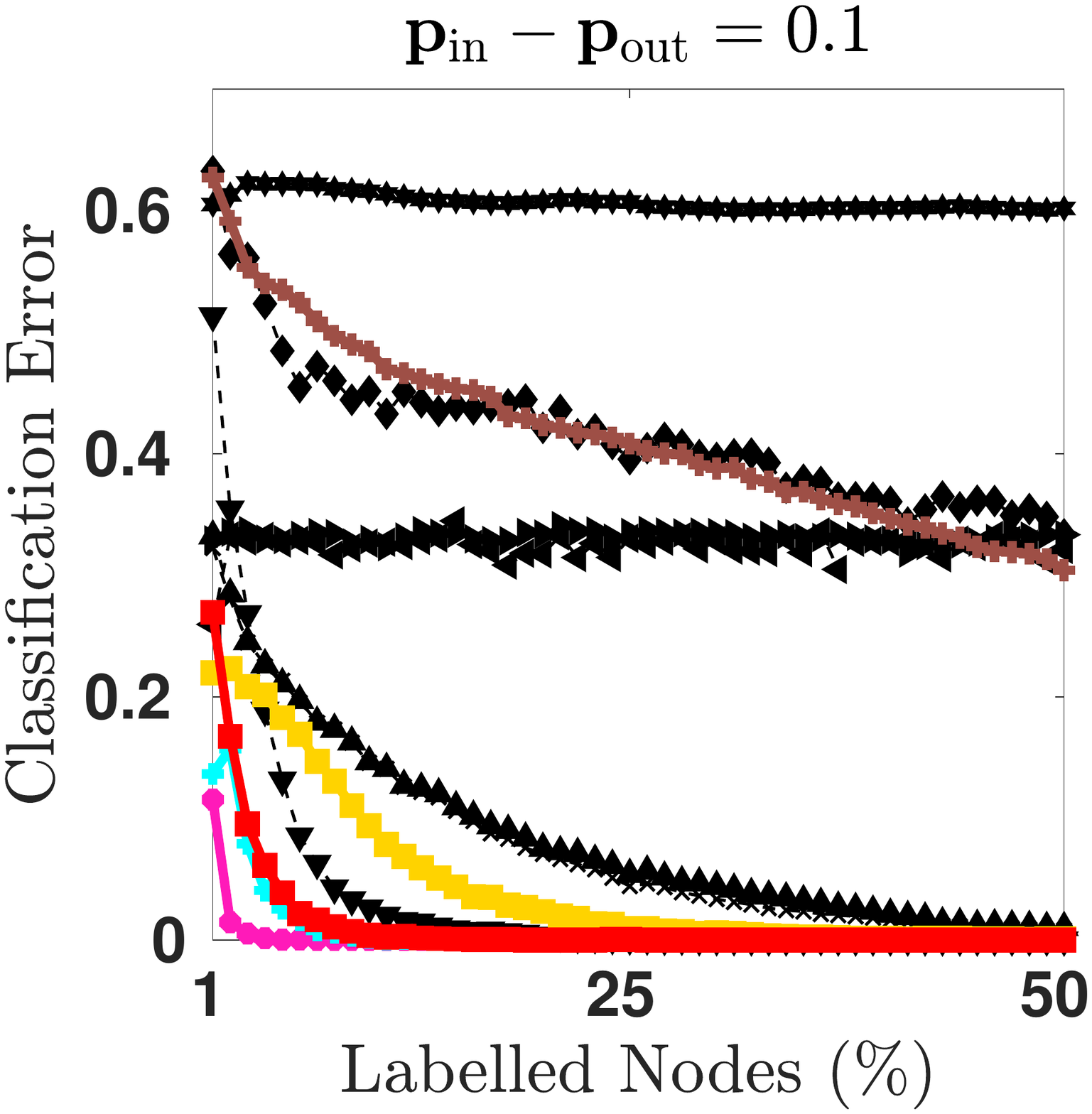}
 \end{subfigure}
 \hfill
 \caption{
Experiments with three Layers and three classes as in Section~\ref{sec:independentInformation} and Figure.\ref{fig:3layers}.
}
 \label{fig:3Layers-supplementary}
\end{figure*}
\section{Analysis on Effect of Regularization Parameter}\label{sec:OnLambda}
In this section we present a numerical evaluation on the effect of the regularization parameter $\lambda$ under the multilayer stochastic block model and on real world datasets. The corresponding results are depicted in Fig.~\ref{fig:SBM:lambda_effect} and Fig.~\label{fig:RealDataSets:lambda_effect}.

\textbf{Experiments under Multilayer Stochastic Block Model}.
We analyze the effect of the regularization parameter $\lambda$ under the Multilayer Stochastic Block Model.
The experimental setting is as follows: We fix the parameters of the first layer $G^{(1)}$ and second layer $G^{(2)}$ to 
$\p^{(1)}=0.09,\q^{(1)}=0.01, \p^{(2)}=0.05,\q^{(2)}=0.05$. We consider values of $\lambda\in\{10^{-3}, 10^{-2}, 10^{-1}, 10^{0},10^{1}, 10^{2}, 10^{3}\}$, 
different amount of labeled nodes $\{1\%,\ldots,50\%\}$. We sample five random multilayer graphs with the corresponding parameters and 5 random samples of labeled nodes with a fixed percentage, and present the average classification error. In Fig.~\ref{fig:SBM:lambda_effect} we can see that in general the larger the value of $\lambda$ the smaller the classification error. In particular we can see that the performance does not present any relevant changes with $\lambda \leq 10^{-1}$.

\textbf{Experiments with real world datasets}.
We analyze the effect of the regularization parameter $\lambda$ with real world datasets considered in Section~\ref{Section:experiments}.
For each dataset we build the corresponding layer adjacency matrices by the taking symmetric $k$-nearest neighbour graph and take as similarity measure the Pearson linear correlation, (i.e. we take the $k$ neighbours with highest correlation), and take the unweighted version of it.

We fix nearest neighbourhood size to $k=10$ and generate 10 samples of labeled nodes, where the percentage of labeled nodes per class is in the range $\{1\%,2\%,\ldots,25\%\}$. The average test errors are presented in Fig.~\ref{fig:RealDataSets:lambda_effect}, for power mean Laplacian regularizers $L_{-1}, L_{-2},L_{-5}$, and $L_{-10}$.
We can see that in general the best performance, i.e. smallest mean test classificaton error corresponds to values of $\lambda=10,10^2,10^3$, verifying the choice of $\lambda=10$ presented in Section~\ref{Section:experiments}.
Moreover, we can see that the mean test error in general decreases with larger amounts of labeled data, which verifies our previous experiments on multilayer graphs following the Multilayer Stochastic Block Model.

\begin{figure*}[t]
\centering
\vskip.0em
\begin{subfigure}[]{0.24\linewidth}
\includegraphics[width=1\linewidth, clip,trim=130 40 170 40]{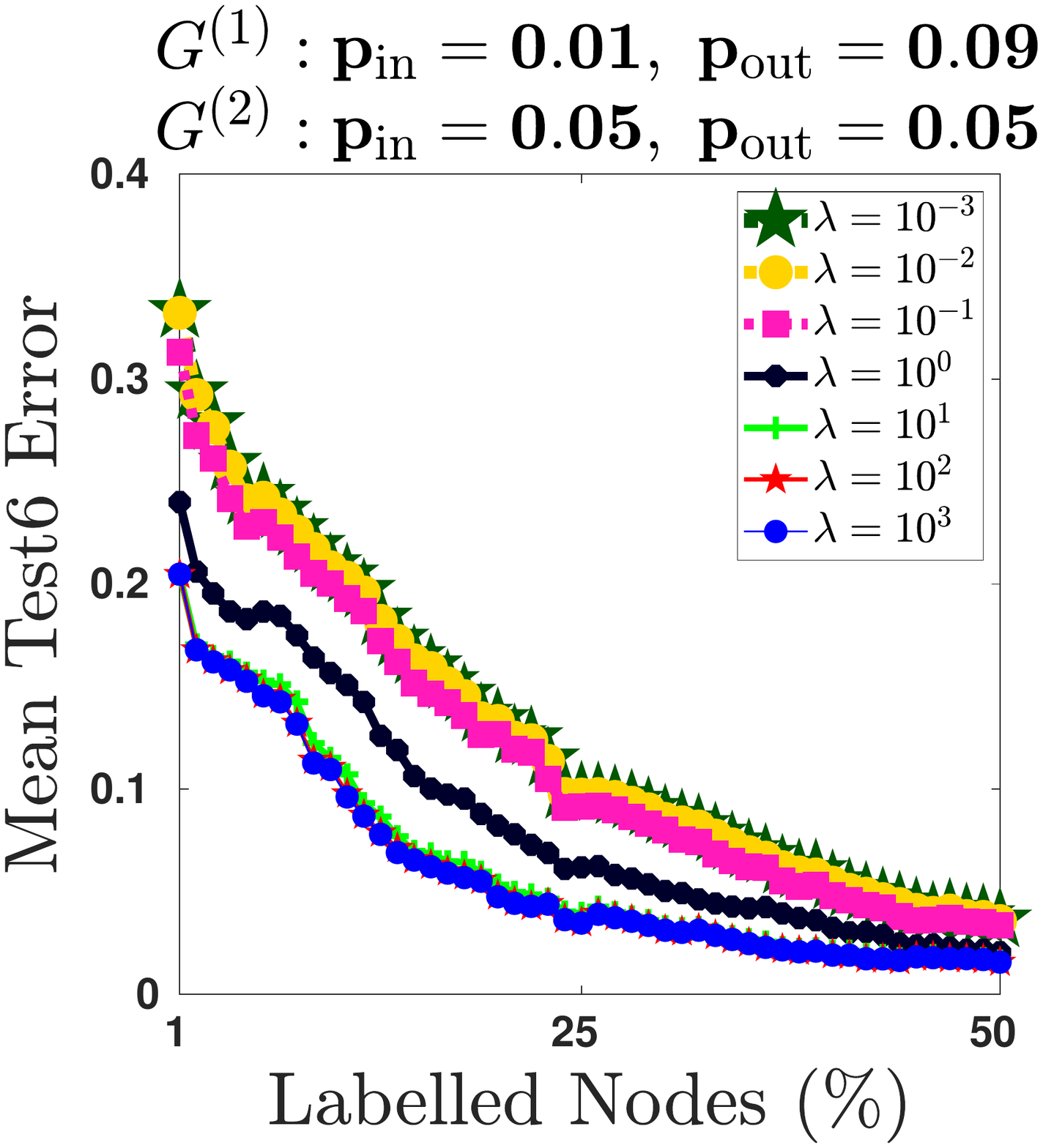}\hspace*{\fill}
\caption{$L_{-1}$}
\label{subfig:fig:SBM:diagonal_shift:fix_Wpos:L_{-1}}
\end{subfigure}%
\hfill
\begin{subfigure}[]{0.24\linewidth}
\includegraphics[width=1\linewidth, clip,trim=130 40 170 40]{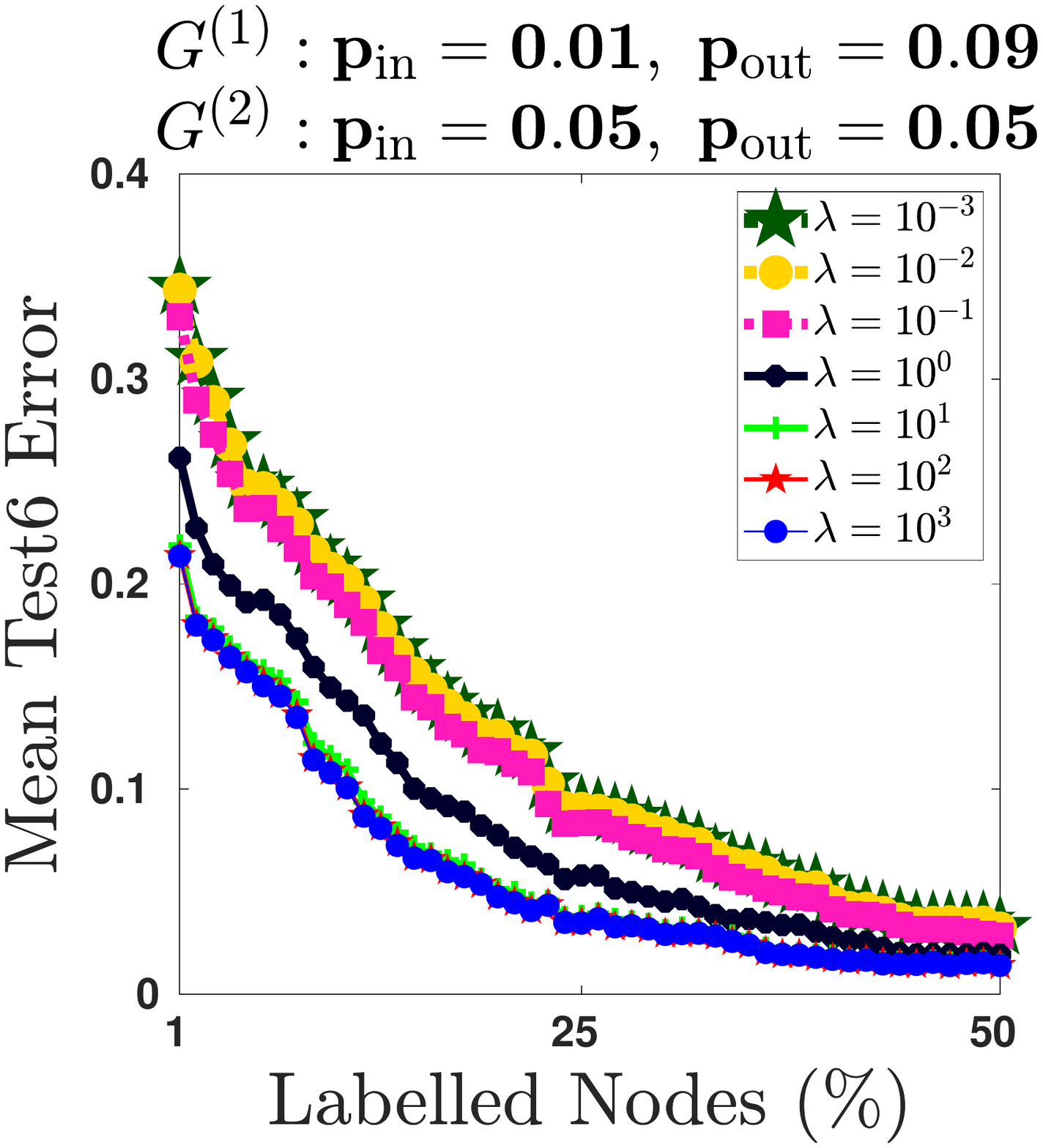}\hspace*{\fill}
\caption{$L_{-2}$}
\label{subfig:fig:SBM:diagonal_shift:fix_Wpos:L_{-2}}
\end{subfigure}%
\hfill
\begin{subfigure}[]{0.24\linewidth}
\includegraphics[width=1\linewidth, clip,trim=130 40 170 40]{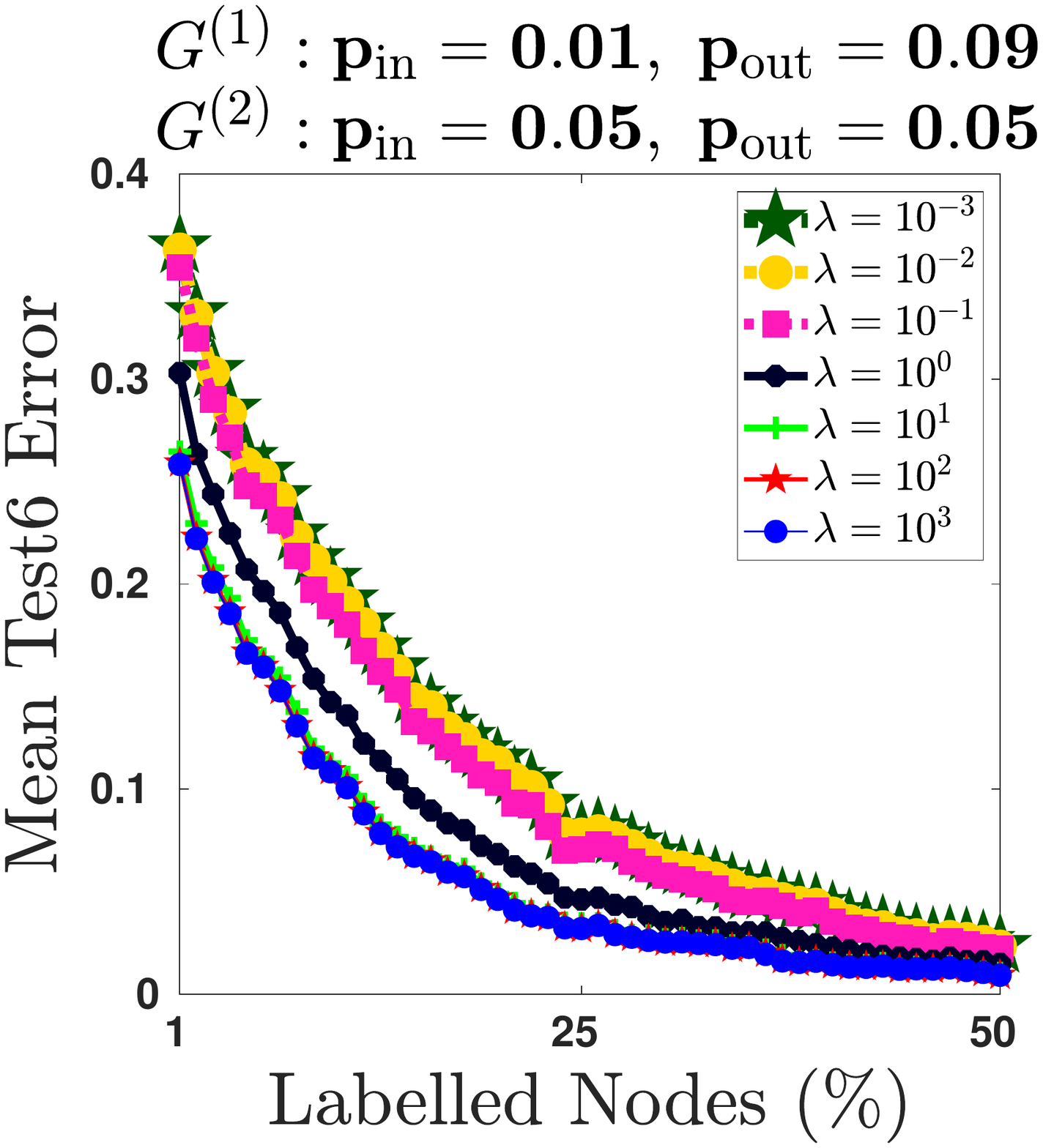}\hspace*{\fill}
\caption{$L_{-5}$}
\label{subfig:fig:SBM:diagonal_shift:fix_Wpos:L_{-5}}
\end{subfigure}%
\hfill
\begin{subfigure}[]{0.24\linewidth}
\includegraphics[width=1\linewidth, clip,trim=130 40 170 40]{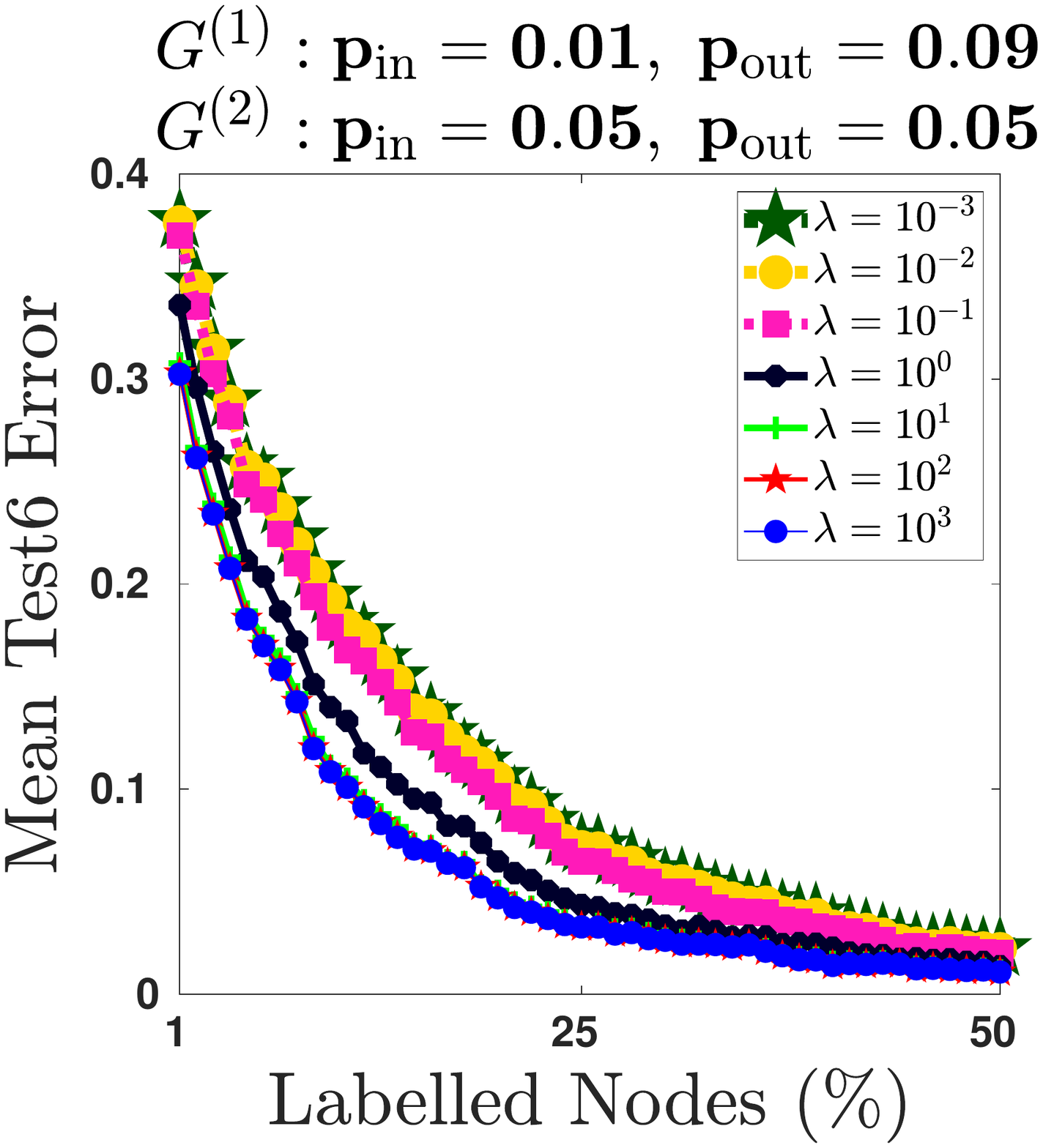}\hspace*{\fill}
\caption{$L_{-10}$}
\label{subfig:fig:SBM:diagonal_shift:fix_Wpos:L_{-10}}
\end{subfigure}
\caption{
Mean test classification error under MSBM for different values of $\lambda$. Details in Sec.~\ref{sec:OnLambda}.
}
\label{fig:SBM:lambda_effect}
\end{figure*}
\begin{figure*}[t]
\centering
\vskip.0em
\begin{subfigure}[t]{1\linewidth}
  \begin{subfigure}[]{0.24\linewidth}
  \renewcommand\thesubfigure{\alph{subfigure}1}
  \includegraphics[width=1\linewidth, clip,trim=110 40 155 40]{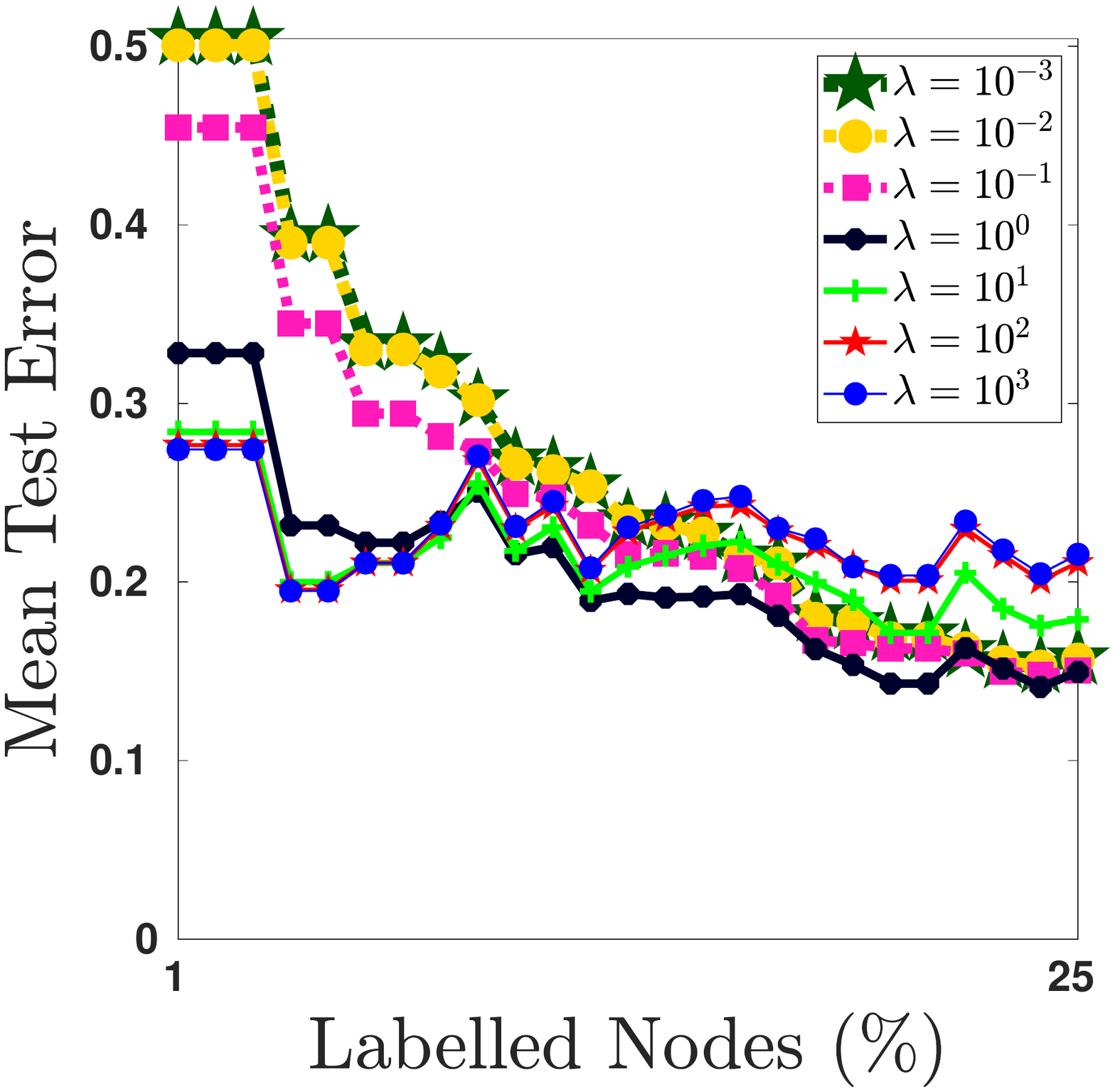}\hspace*{\fill}
  \caption{$L_{-1}$}
  \label{subfig:fig:SBM:diagonal_shift:fix_Wpos:L_{-1}}
  \end{subfigure}%
  \hfill
  \begin{subfigure}[]{0.24\linewidth}
  \addtocounter{subfigure}{-1}
  \renewcommand\thesubfigure{\alph{subfigure}2}
  \includegraphics[width=1\linewidth, clip,trim=110 40 155 40]{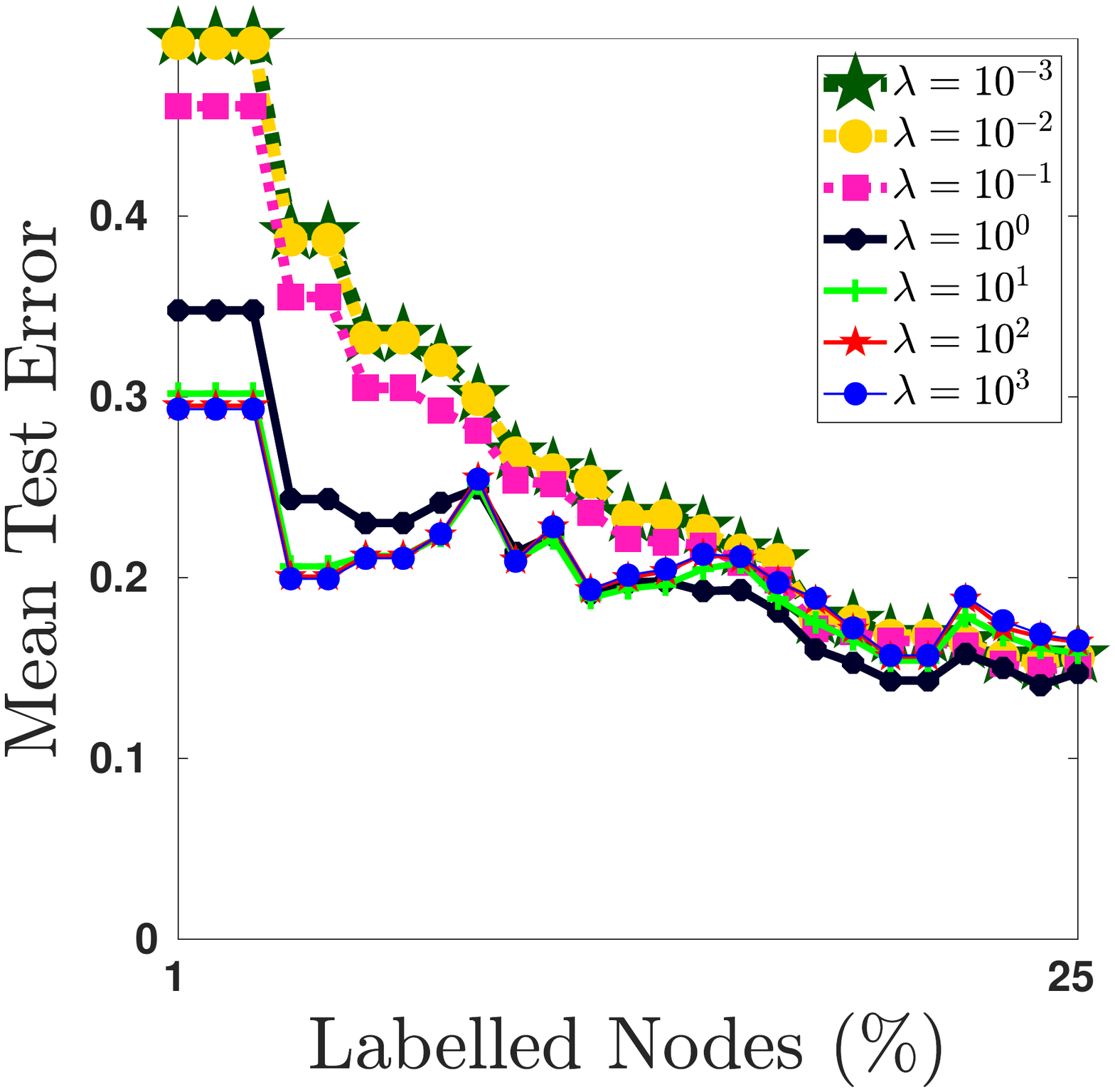}\hspace*{\fill}
  \caption{$L_{-2}$}
  \label{subfig:fig:SBM:diagonal_shift:fix_Wpos:L_{-2}}
  \end{subfigure}%
  \hfill
  \begin{subfigure}[]{0.24\linewidth}
    \addtocounter{subfigure}{-1}
  \renewcommand\thesubfigure{\alph{subfigure}3}
  \includegraphics[width=1\linewidth, clip,trim=110 40 155 40]{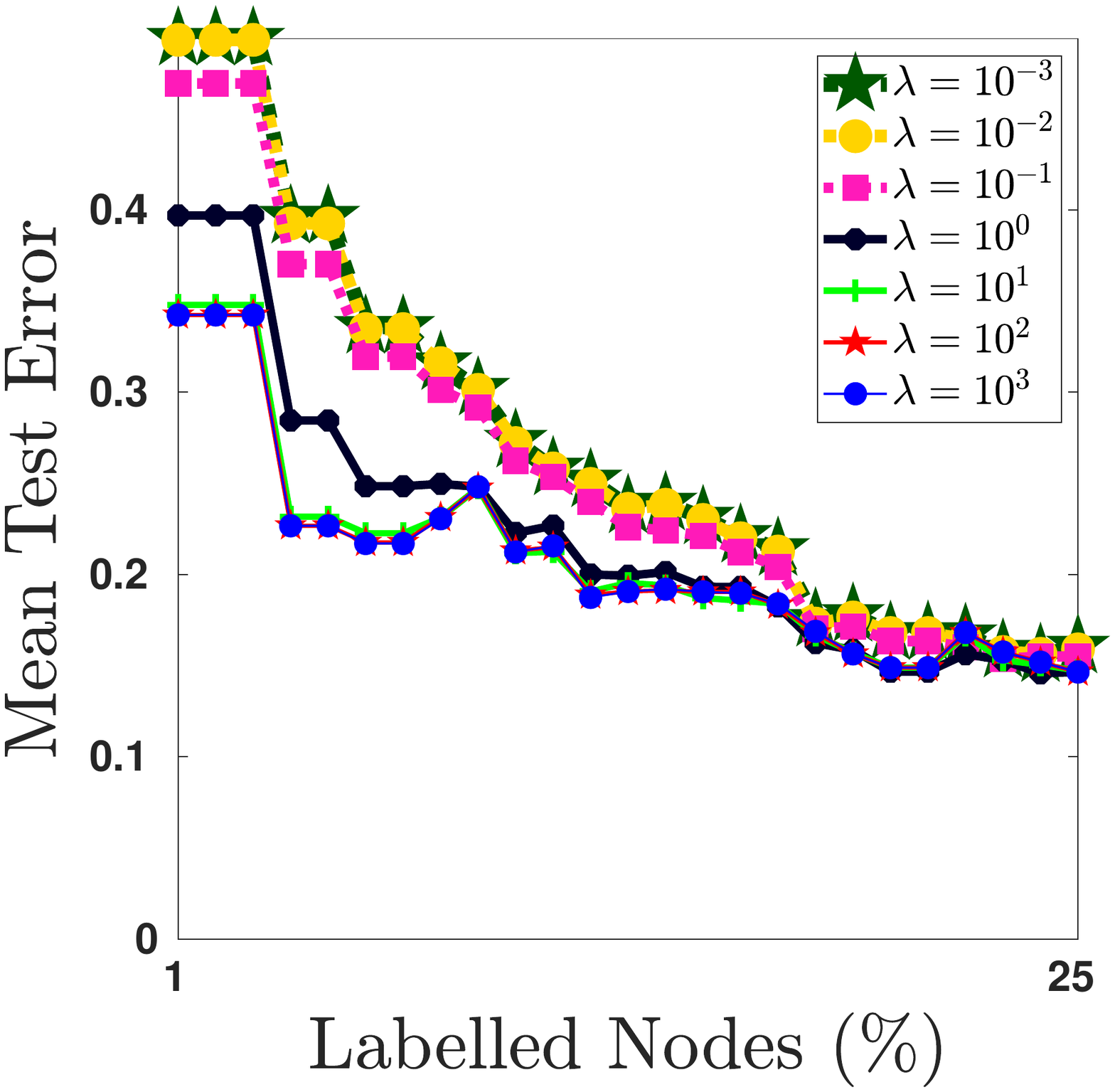}\hspace*{\fill}
  \caption{$L_{-5}$}
  \label{subfig:fig:SBM:diagonal_shift:fix_Wpos:L_{-5}}
  \end{subfigure}%
  \hfill
  \begin{subfigure}[]{0.24\linewidth}
    \addtocounter{subfigure}{-1}
  \renewcommand\thesubfigure{\alph{subfigure}4}
  \includegraphics[width=1\linewidth, clip,trim=110 40 155 40]{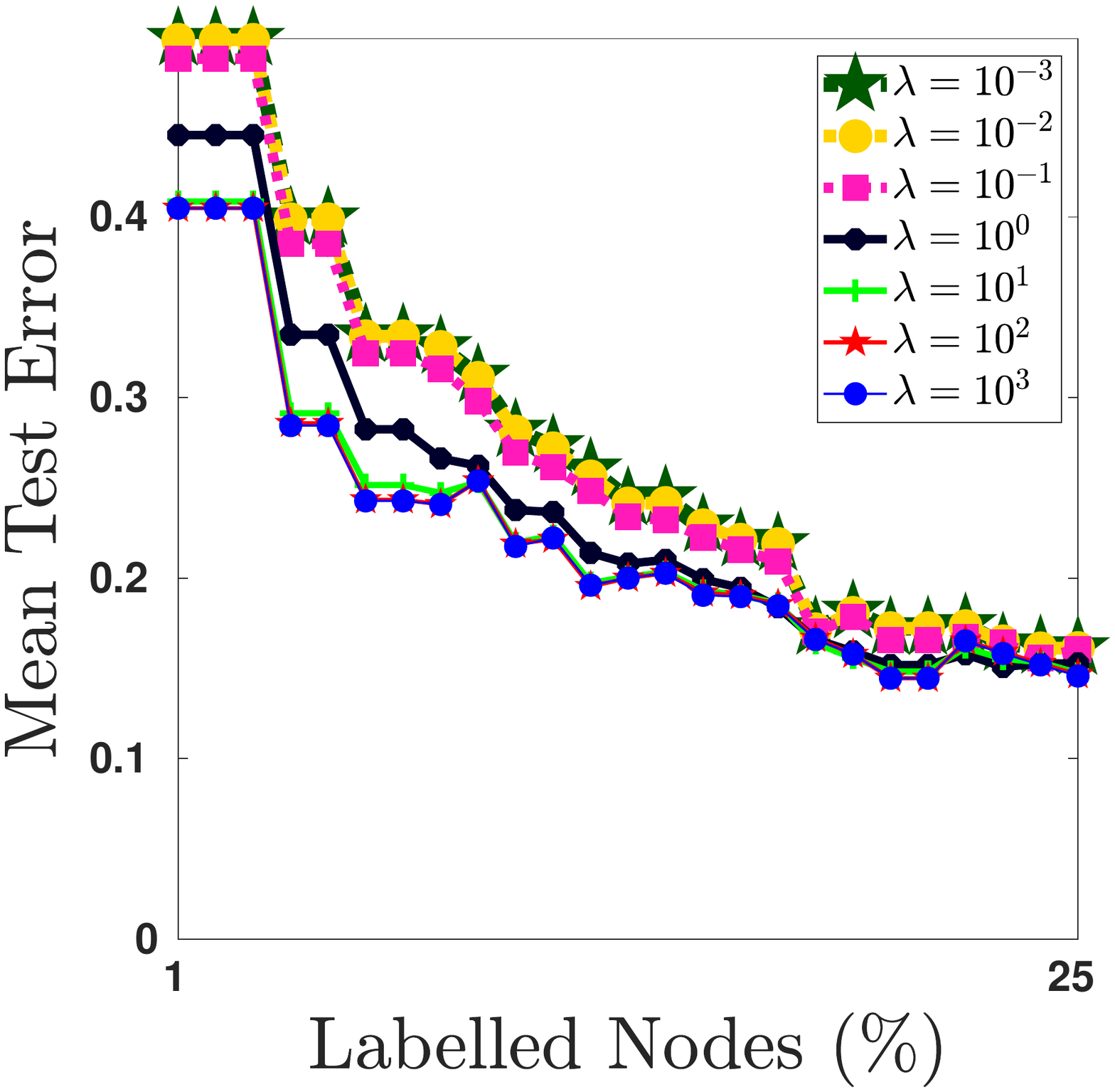}\hspace*{\fill}
  \caption{$L_{-10}$}
  \label{subfig:fig:SBM:diagonal_shift:fix_Wpos:L_{-10}}
  \end{subfigure}
  \setcounter{subfigure}{0}
  \caption{Dataset: 3sources}
\end{subfigure}  
\begin{subfigure}[t]{1\linewidth}
  \begin{subfigure}[]{0.24\linewidth}
  \renewcommand\thesubfigure{\alph{subfigure}1}
  \includegraphics[width=1\linewidth, clip,trim=110 40 155 40]{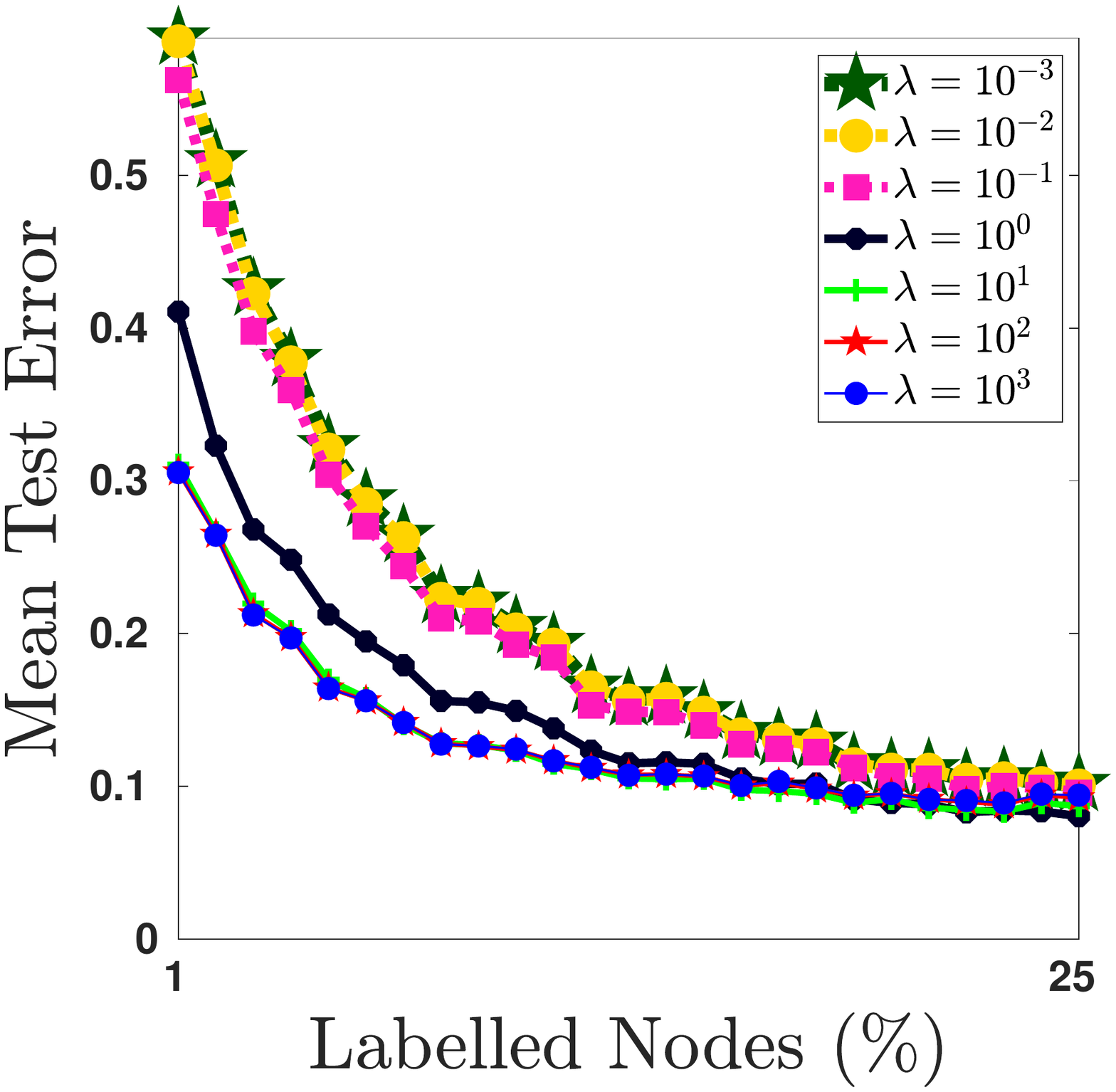}\hspace*{\fill}
  \caption{$L_{-1}$}
  \label{subfig:fig:SBM:diagonal_shift:fix_Wpos:L_{-1}}
  \end{subfigure}%
  \hfill
  \begin{subfigure}[]{0.24\linewidth}
  \addtocounter{subfigure}{-1}
  \renewcommand\thesubfigure{\alph{subfigure}2}
  \includegraphics[width=1\linewidth, clip,trim=110 40 155 40]{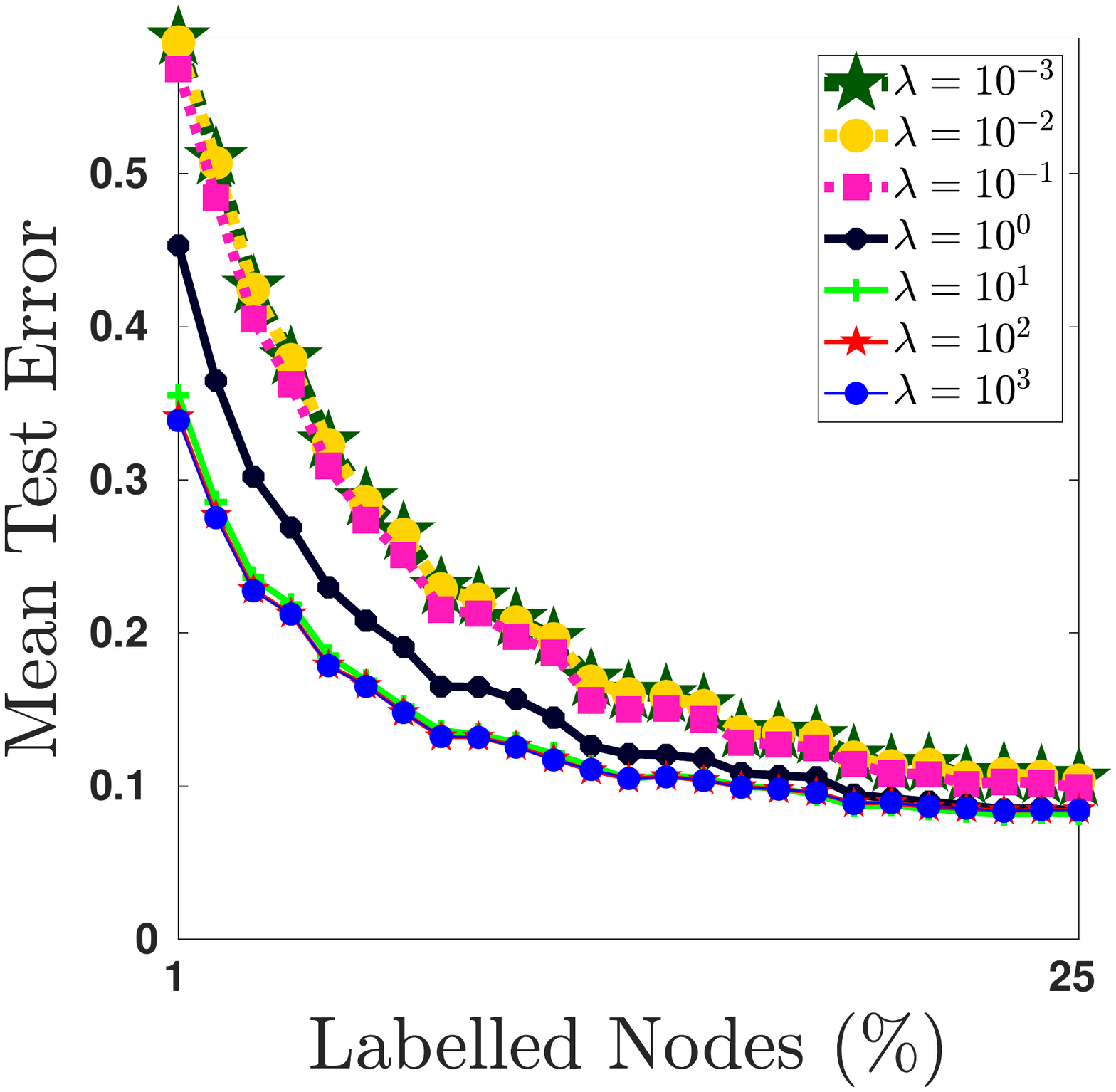}\hspace*{\fill}
  \caption{$L_{-2}$}
  \label{subfig:fig:SBM:diagonal_shift:fix_Wpos:L_{-2}}
  \end{subfigure}%
  \hfill
  \begin{subfigure}[]{0.24\linewidth}
    \addtocounter{subfigure}{-1}
  \renewcommand\thesubfigure{\alph{subfigure}3}
  \includegraphics[width=1\linewidth, clip,trim=110 40 155 40]{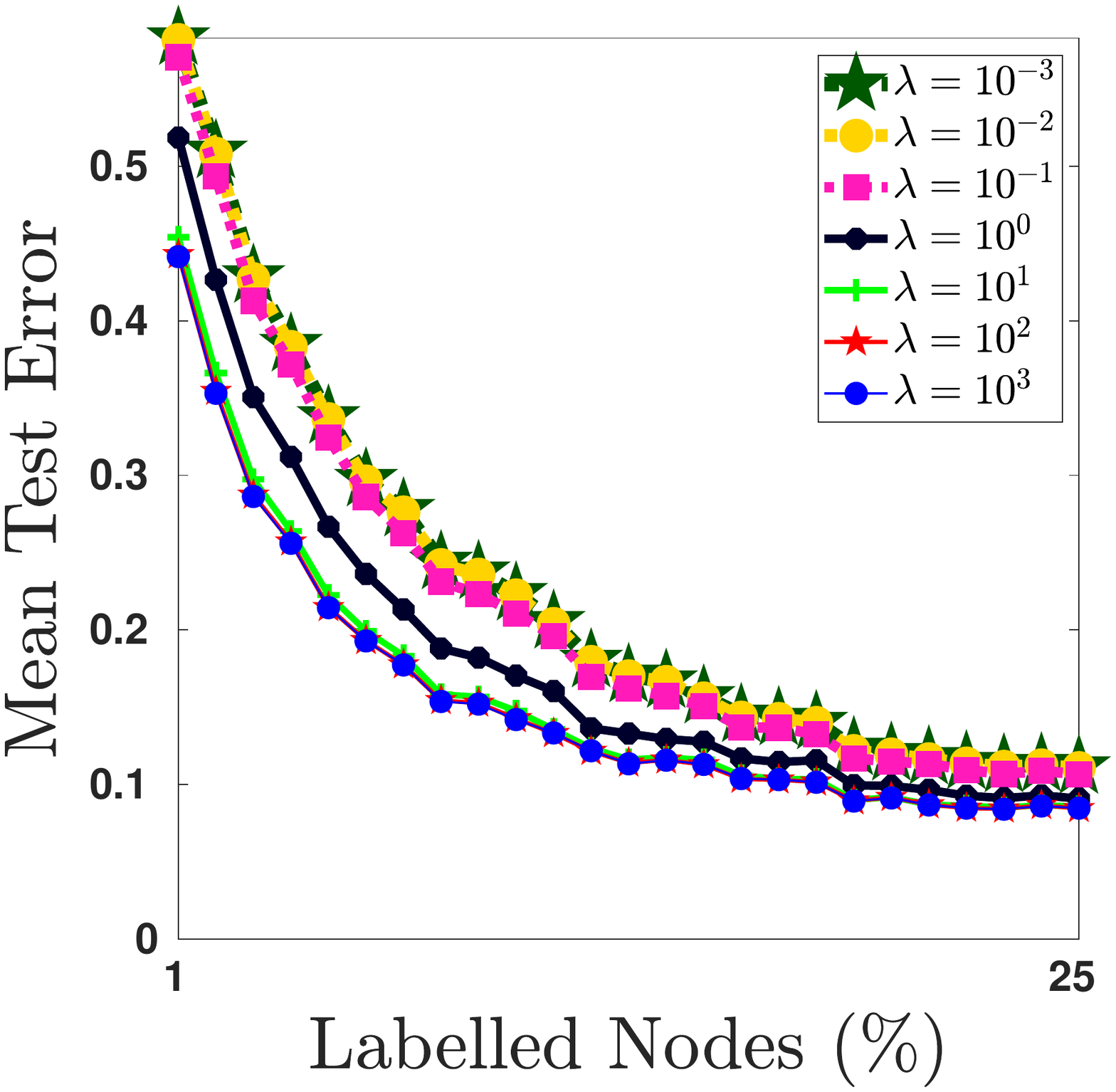}\hspace*{\fill}
  \caption{$L_{-5}$}
  \label{subfig:fig:SBM:diagonal_shift:fix_Wpos:L_{-5}}
  \end{subfigure}%
  \hfill
  \begin{subfigure}[]{0.24\linewidth}
    \addtocounter{subfigure}{-1}
  \renewcommand\thesubfigure{\alph{subfigure}4}
  \includegraphics[width=1\linewidth, clip,trim=110 40 155 40]{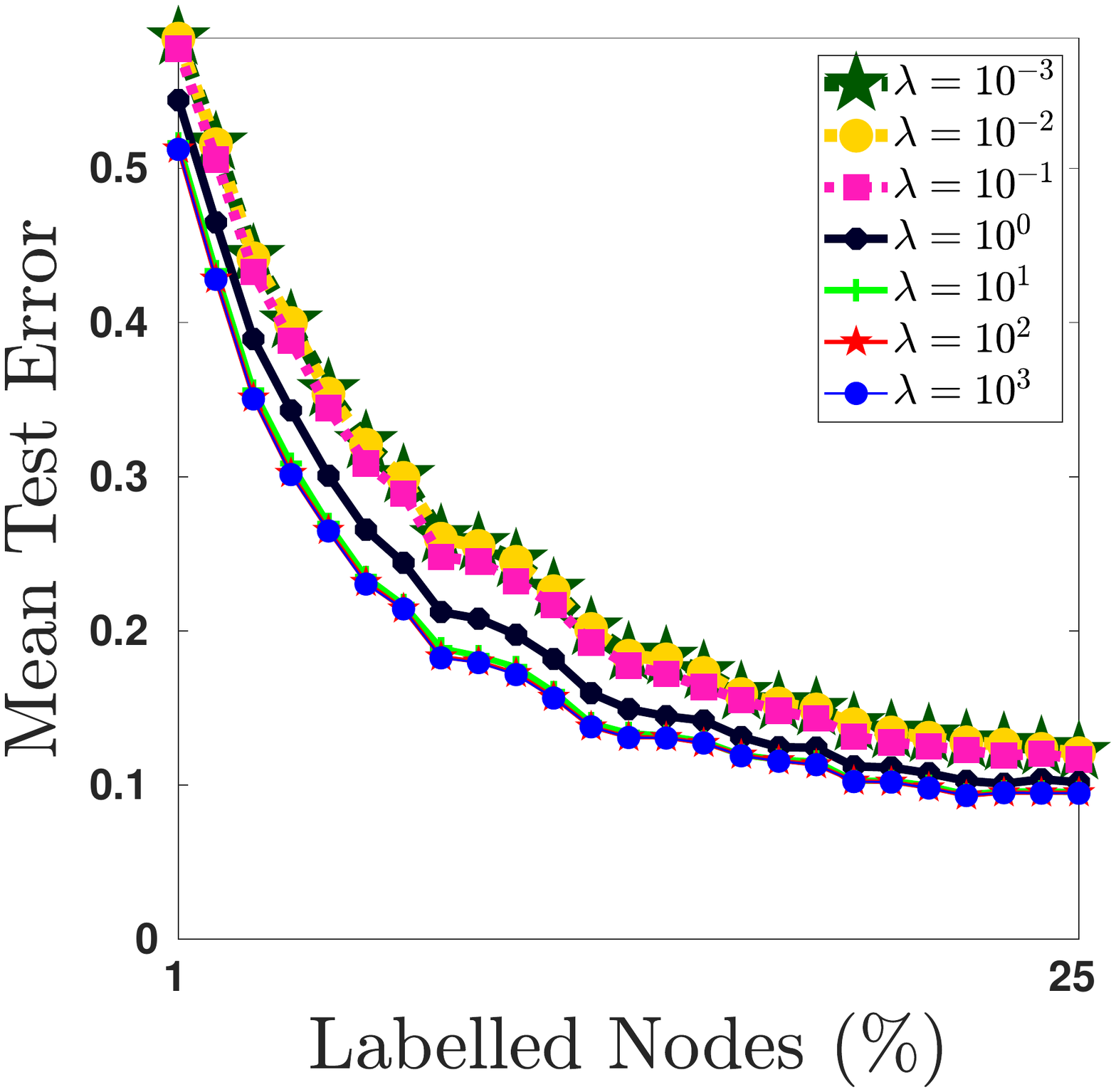}\hspace*{\fill}
  \caption{$L_{-10}$}
  \label{subfig:fig:SBM:diagonal_shift:fix_Wpos:L_{-10}}
  \end{subfigure}
  \setcounter{subfigure}{1}
  \caption{Dataset: BBC}
\end{subfigure}  
\begin{subfigure}[t]{1\linewidth}
  \begin{subfigure}[]{0.24\linewidth}
  \renewcommand\thesubfigure{\alph{subfigure}1}
  \includegraphics[width=1\linewidth, clip,trim=110 40 155 40]{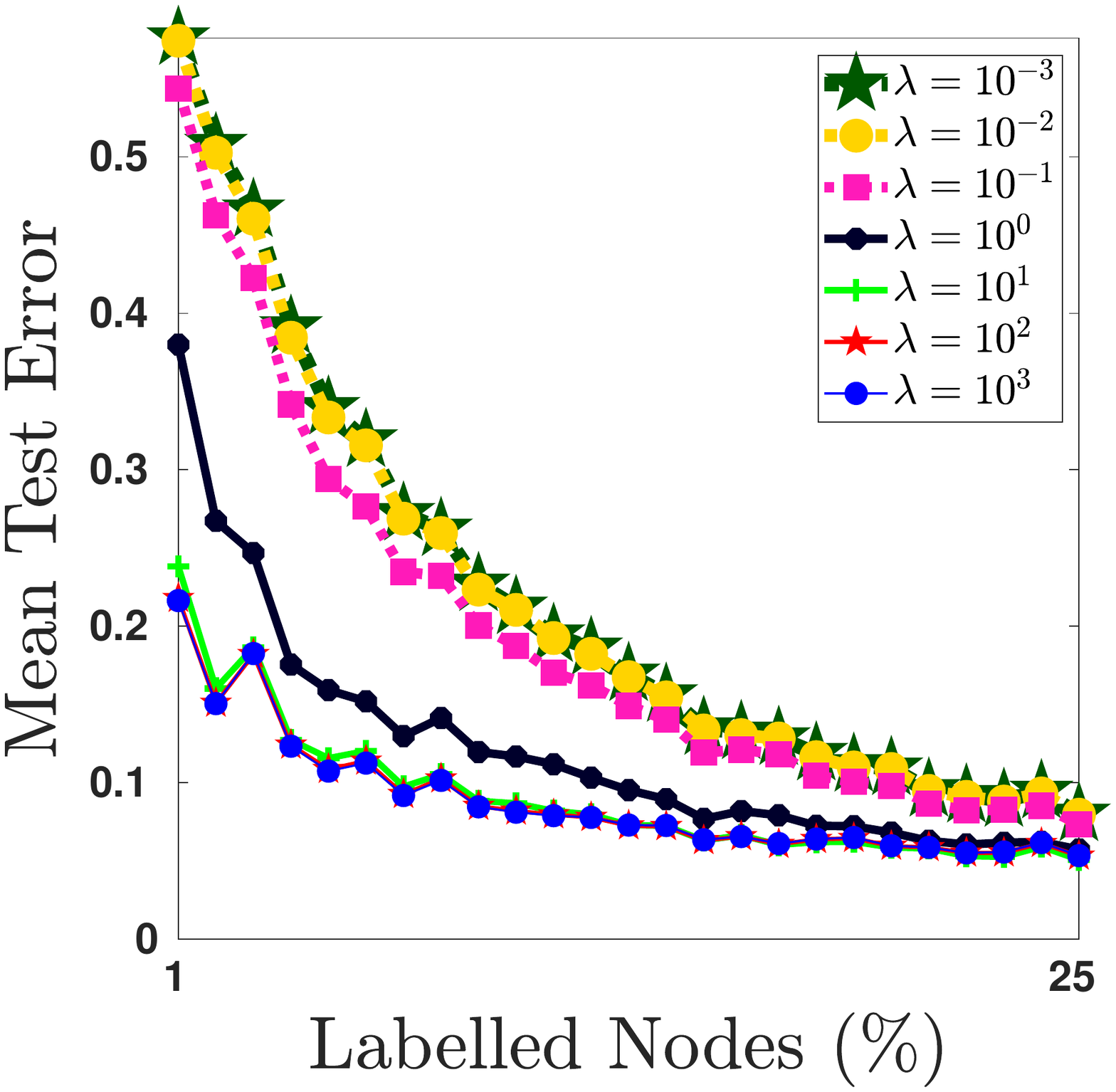}\hspace*{\fill}
  \caption{$L_{-1}$}
  \label{subfig:fig:SBM:diagonal_shift:fix_Wpos:L_{-1}}
  \end{subfigure}%
  \hfill
  \begin{subfigure}[]{0.24\linewidth}
  \addtocounter{subfigure}{-1}
  \renewcommand\thesubfigure{\alph{subfigure}2}
  \includegraphics[width=1\linewidth, clip,trim=110 40 155 40]{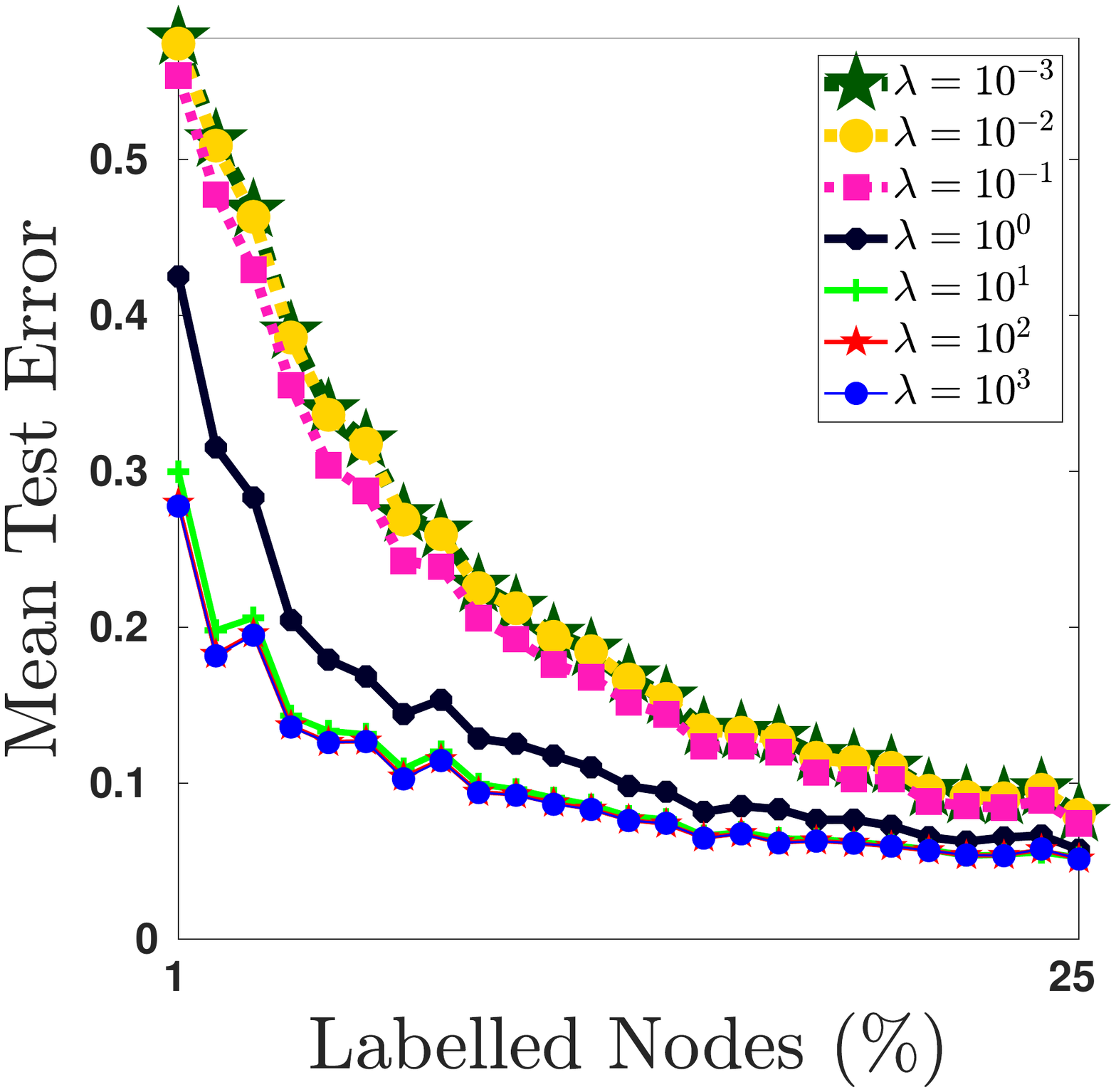}\hspace*{\fill}
  \caption{$L_{-2}$}
  \label{subfig:fig:SBM:diagonal_shift:fix_Wpos:L_{-2}}
  \end{subfigure}%
  \hfill
  \begin{subfigure}[]{0.24\linewidth}
    \addtocounter{subfigure}{-1}
  \renewcommand\thesubfigure{\alph{subfigure}3}
  \includegraphics[width=1\linewidth, clip,trim=110 40 155 40]{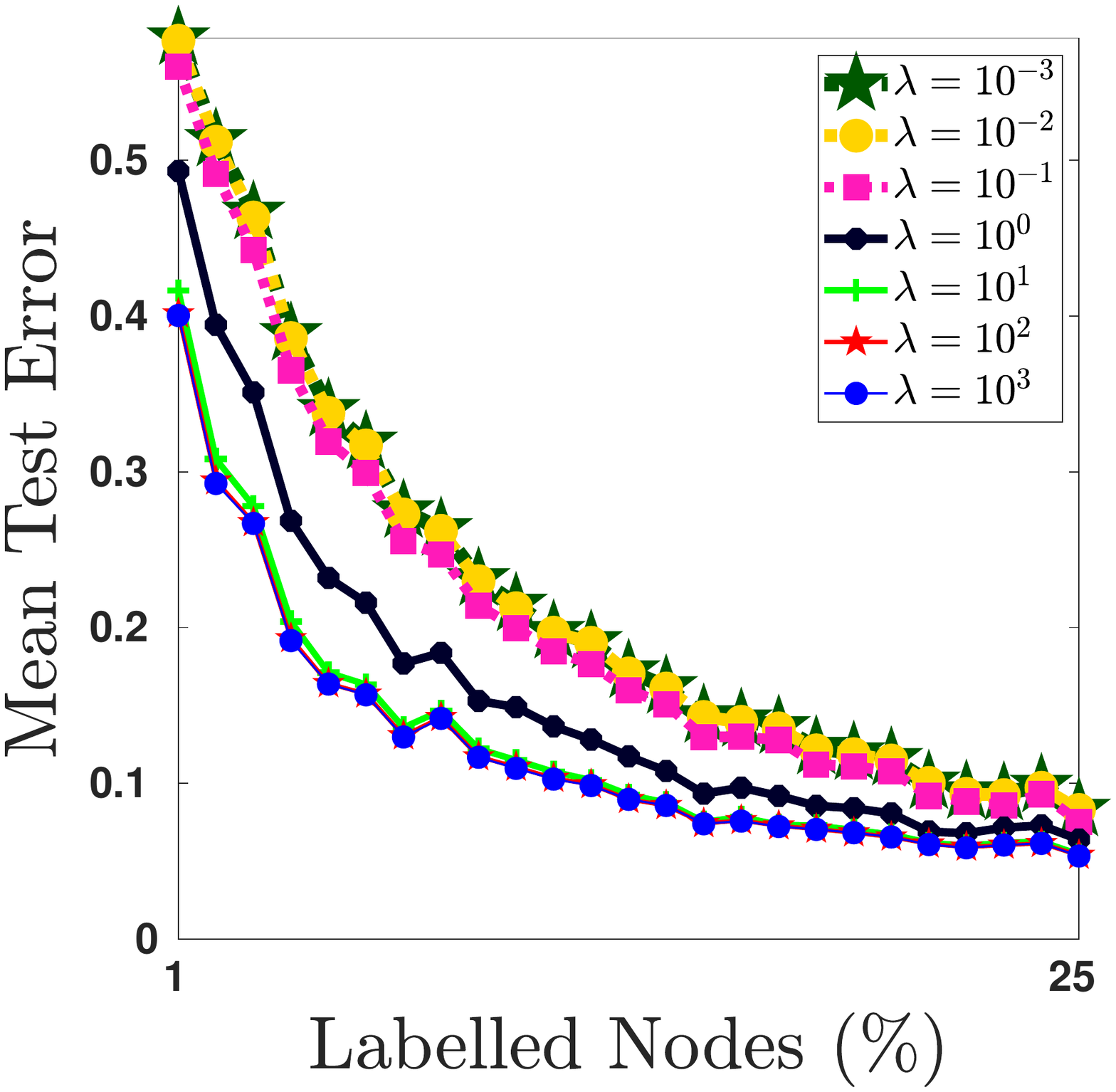}\hspace*{\fill}
  \caption{$L_{-5}$}
  \label{subfig:fig:SBM:diagonal_shift:fix_Wpos:L_{-5}}
  \end{subfigure}%
  \hfill
  \begin{subfigure}[]{0.24\linewidth}
    \addtocounter{subfigure}{-1}
  \renewcommand\thesubfigure{\alph{subfigure}4}
  \includegraphics[width=1\linewidth, clip,trim=110 40 155 40]{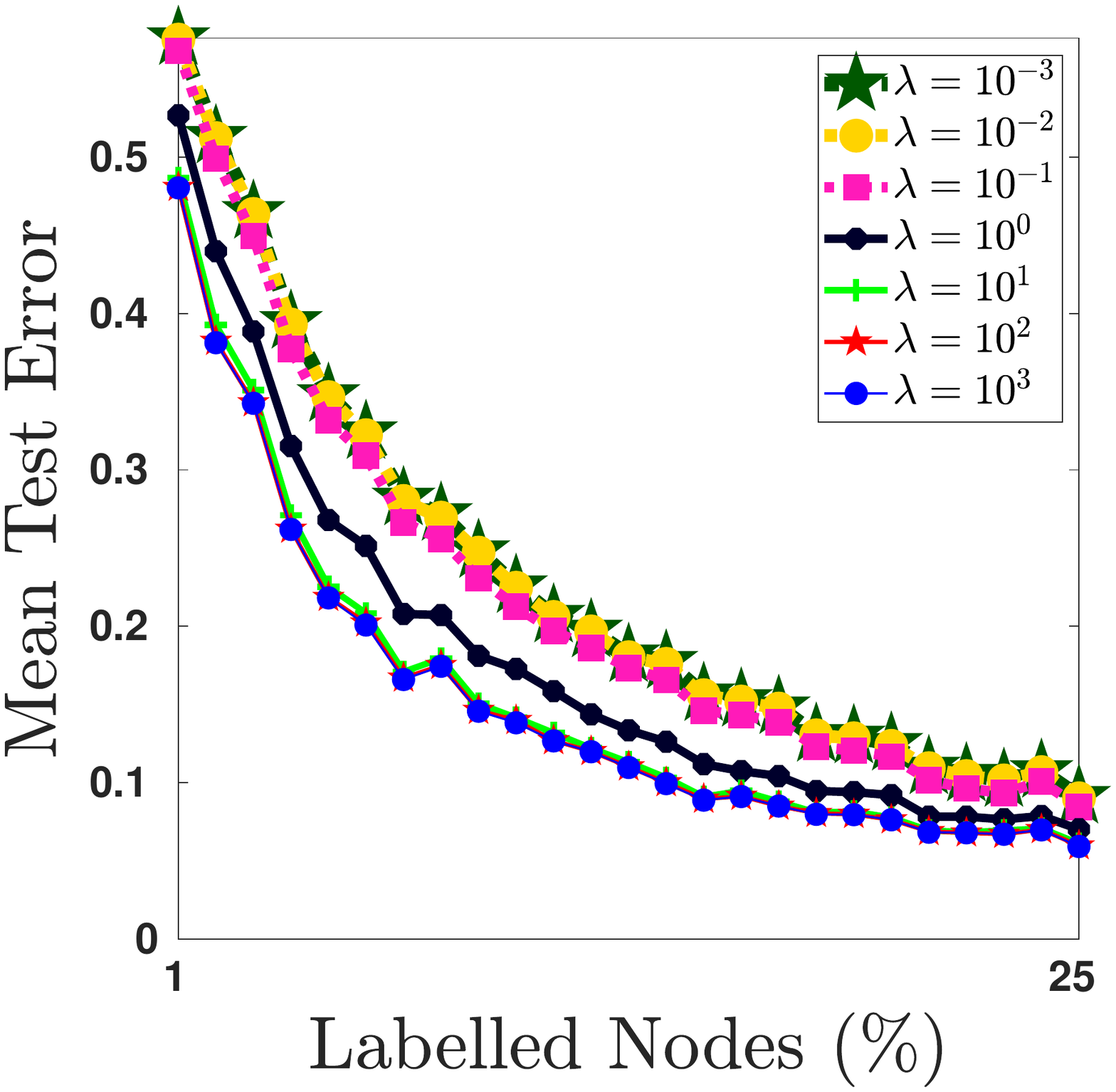}\hspace*{\fill}
  \caption{$L_{-10}$}
  \label{subfig:fig:SBM:diagonal_shift:fix_Wpos:L_{-10}}
  \end{subfigure}
  \setcounter{subfigure}{2}
  \caption{Dataset: BBCS}
\end{subfigure}  
\begin{subfigure}[t]{1\linewidth}
  \begin{subfigure}[]{0.24\linewidth}
  \renewcommand\thesubfigure{\alph{subfigure}1}
  \includegraphics[width=1\linewidth, clip,trim=110 40 155 40]{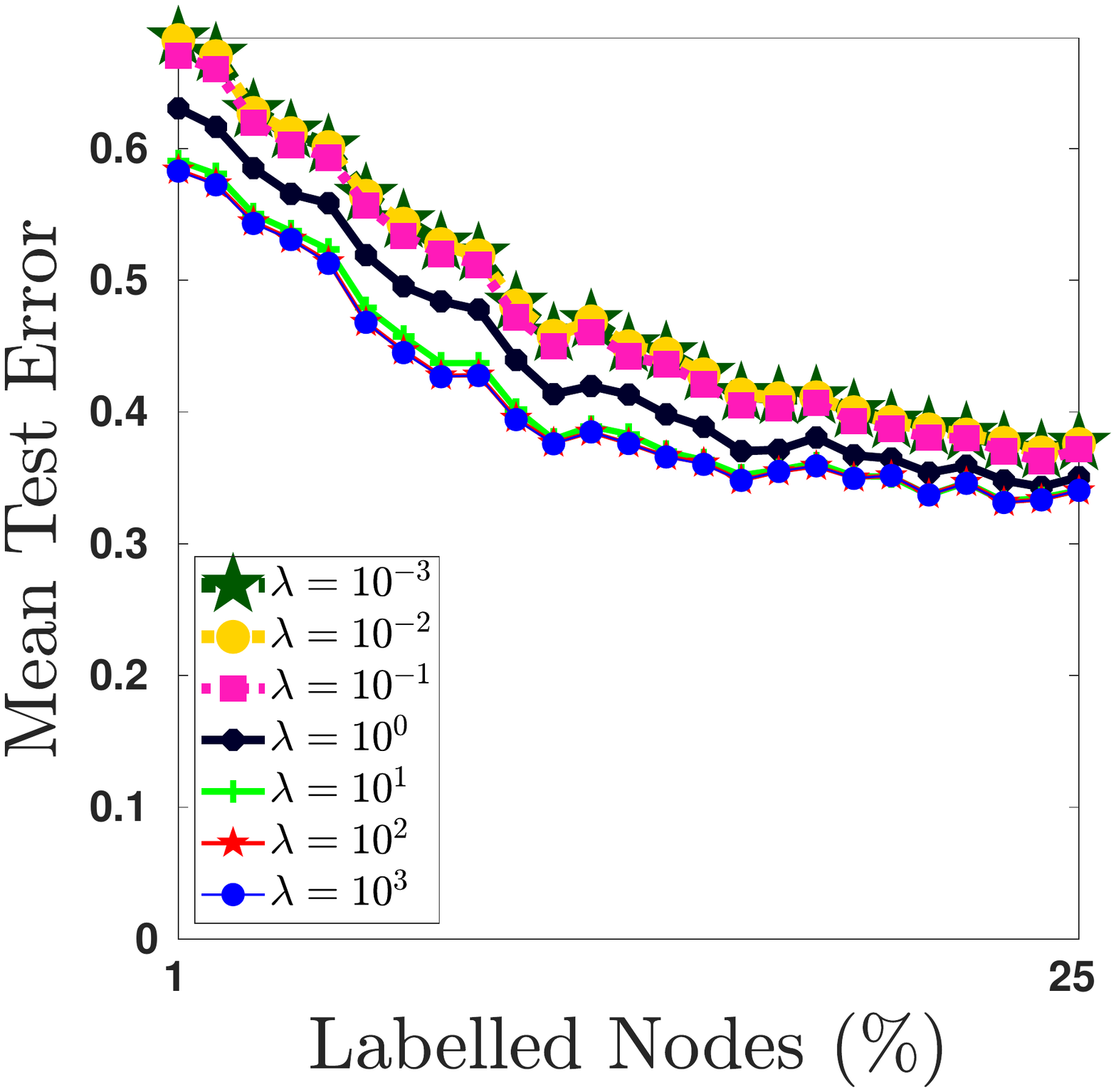}\hspace*{\fill}
  \caption{$L_{-1}$}
  \label{subfig:fig:SBM:diagonal_shift:fix_Wpos:L_{-1}}
  \end{subfigure}%
  \hfill
  \begin{subfigure}[]{0.24\linewidth}
  \addtocounter{subfigure}{-1}
  \renewcommand\thesubfigure{\alph{subfigure}2}
  \includegraphics[width=1\linewidth, clip,trim=110 40 155 40]{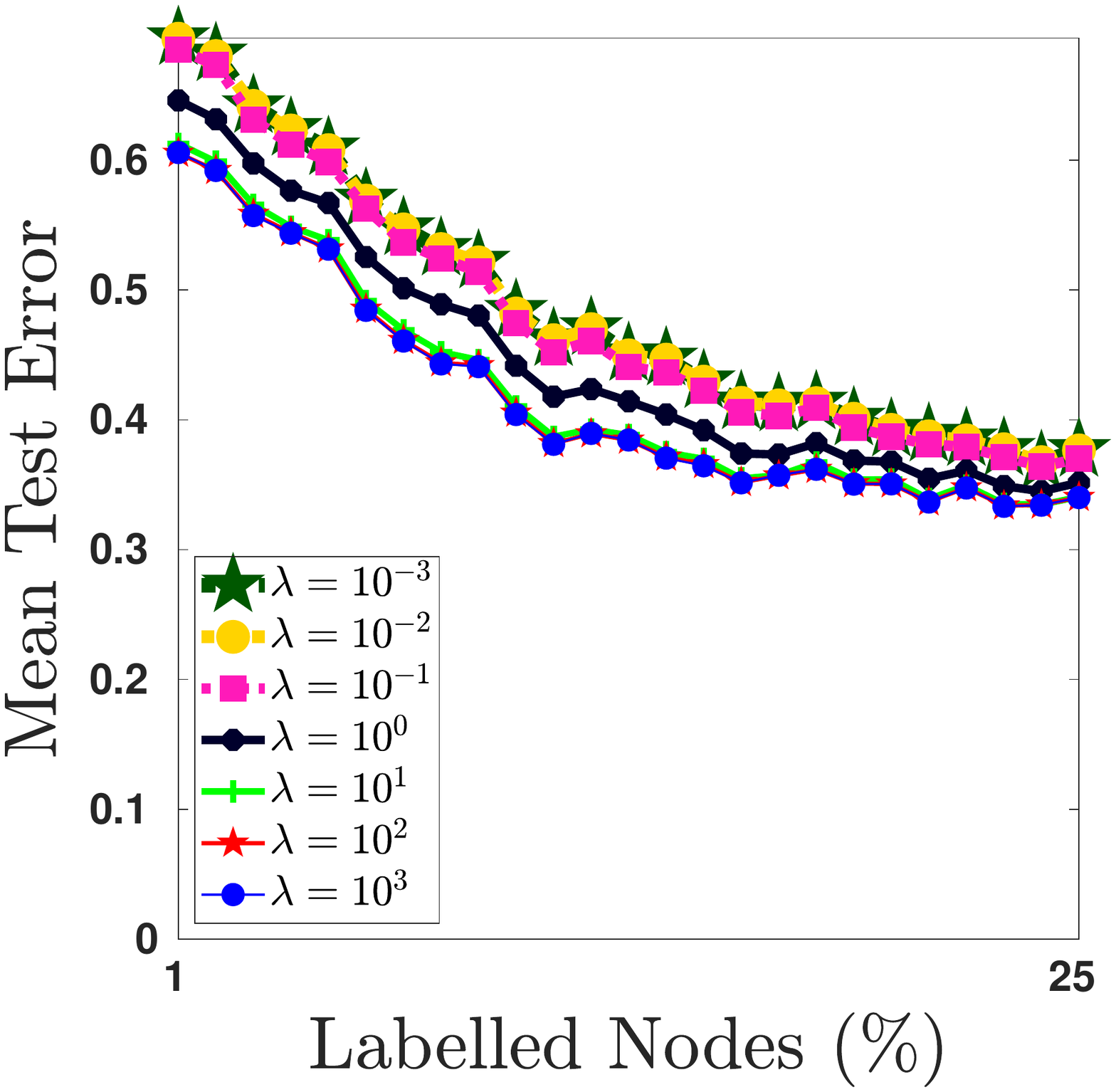}\hspace*{\fill}
  \caption{$L_{-2}$}
  \label{subfig:fig:SBM:diagonal_shift:fix_Wpos:L_{-2}}
  \end{subfigure}%
  \hfill
  \begin{subfigure}[]{0.24\linewidth}
    \addtocounter{subfigure}{-1}
  \renewcommand\thesubfigure{\alph{subfigure}3}
  \includegraphics[width=1\linewidth, clip,trim=110 40 155 40]{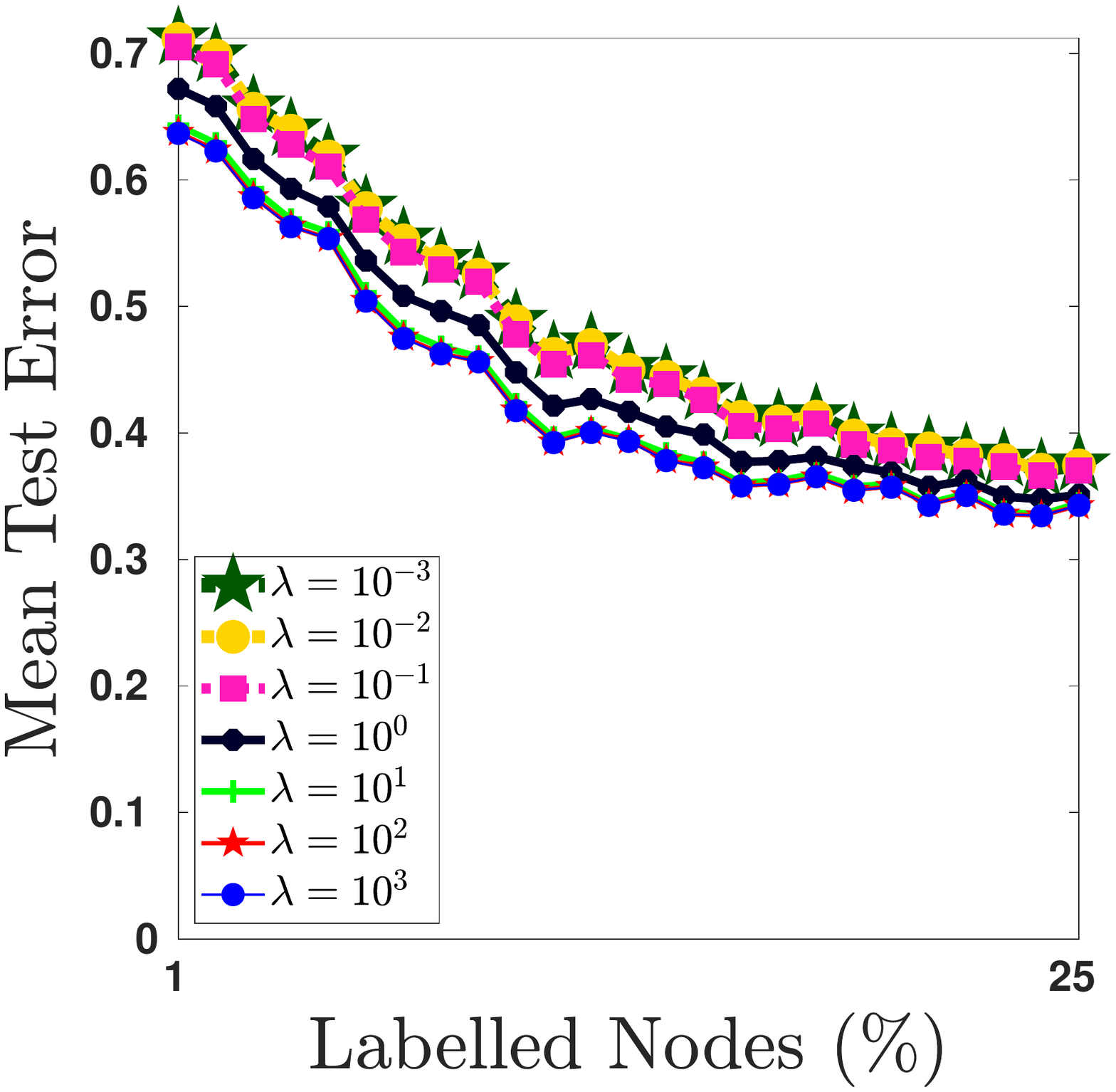}\hspace*{\fill}
  \caption{$L_{-5}$}
  \label{subfig:fig:SBM:diagonal_shift:fix_Wpos:L_{-5}}
  \end{subfigure}%
  \hfill
  \begin{subfigure}[]{0.24\linewidth}
    \addtocounter{subfigure}{-1}
  \renewcommand\thesubfigure{\alph{subfigure}4}
  \includegraphics[width=1\linewidth, clip,trim=110 40 155 40]{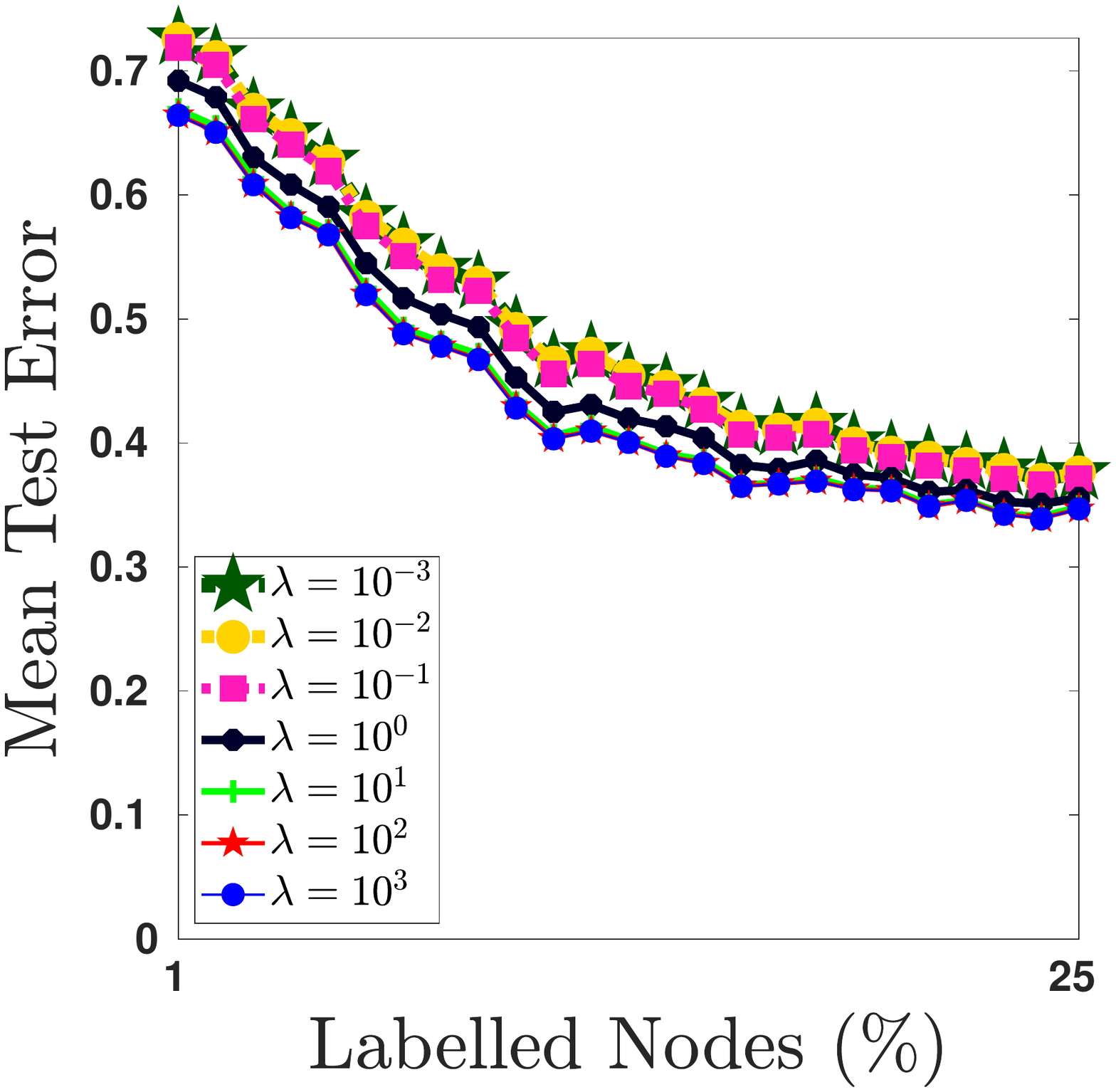}\hspace*{\fill}
  \caption{$L_{-10}$}
  \label{subfig:fig:SBM:diagonal_shift:fix_Wpos:L_{-10}}
  \end{subfigure}
  \setcounter{subfigure}{3}
  \caption{Dataset: Wikipedia}
\end{subfigure}  
\begin{subfigure}[t]{1\linewidth}
  \begin{subfigure}[]{0.24\linewidth}
  \renewcommand\thesubfigure{\alph{subfigure}1}
  \includegraphics[width=1\linewidth, clip,trim=110 40 155 40]{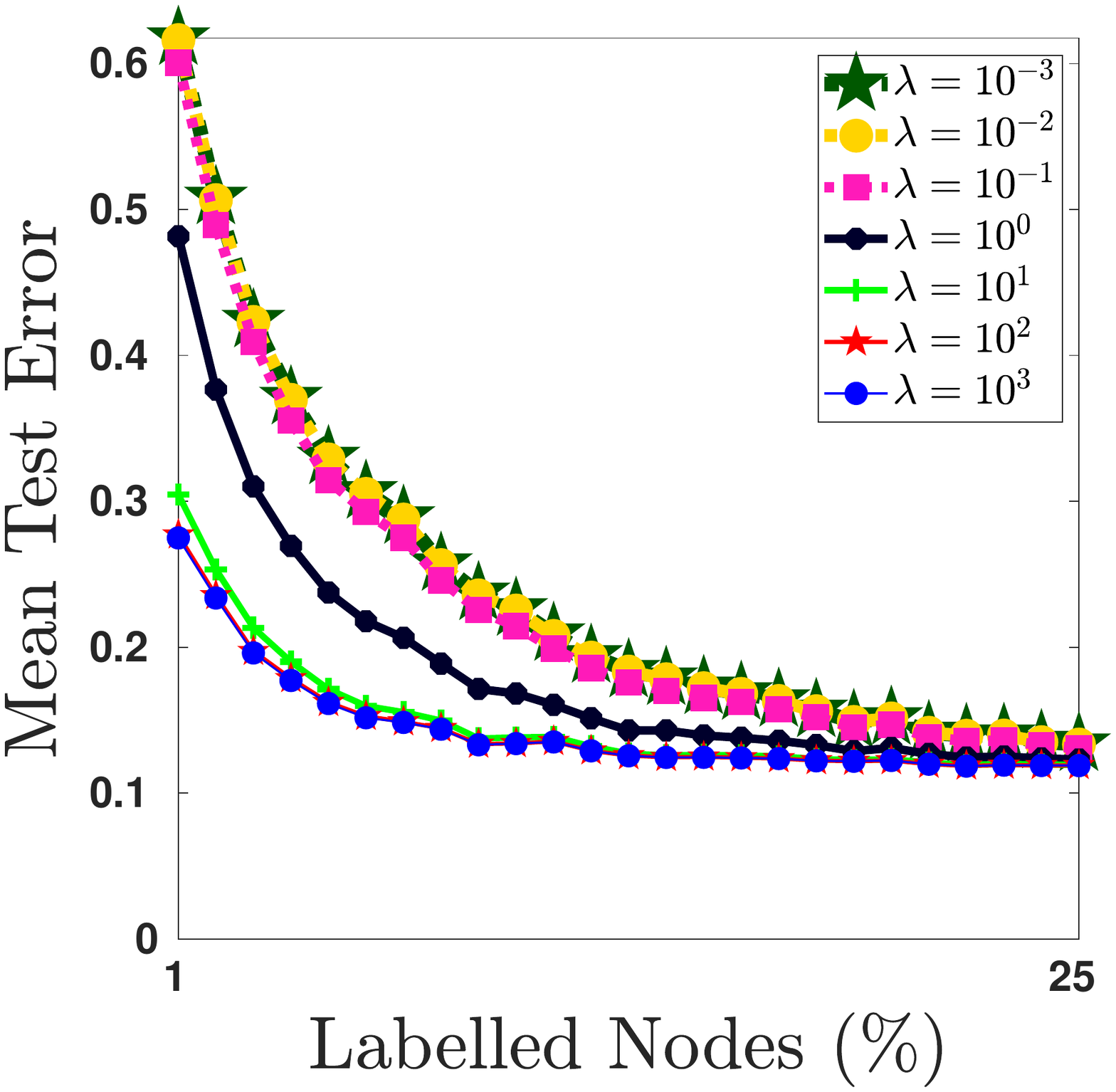}\hspace*{\fill}
  \caption{$L_{-1}$}
  \label{subfig:fig:SBM:diagonal_shift:fix_Wpos:L_{-1}}
  \end{subfigure}%
  \hfill
  \begin{subfigure}[]{0.24\linewidth}
  \addtocounter{subfigure}{-1}
  \renewcommand\thesubfigure{\alph{subfigure}2}
  \includegraphics[width=1\linewidth, clip,trim=110 40 155 40]{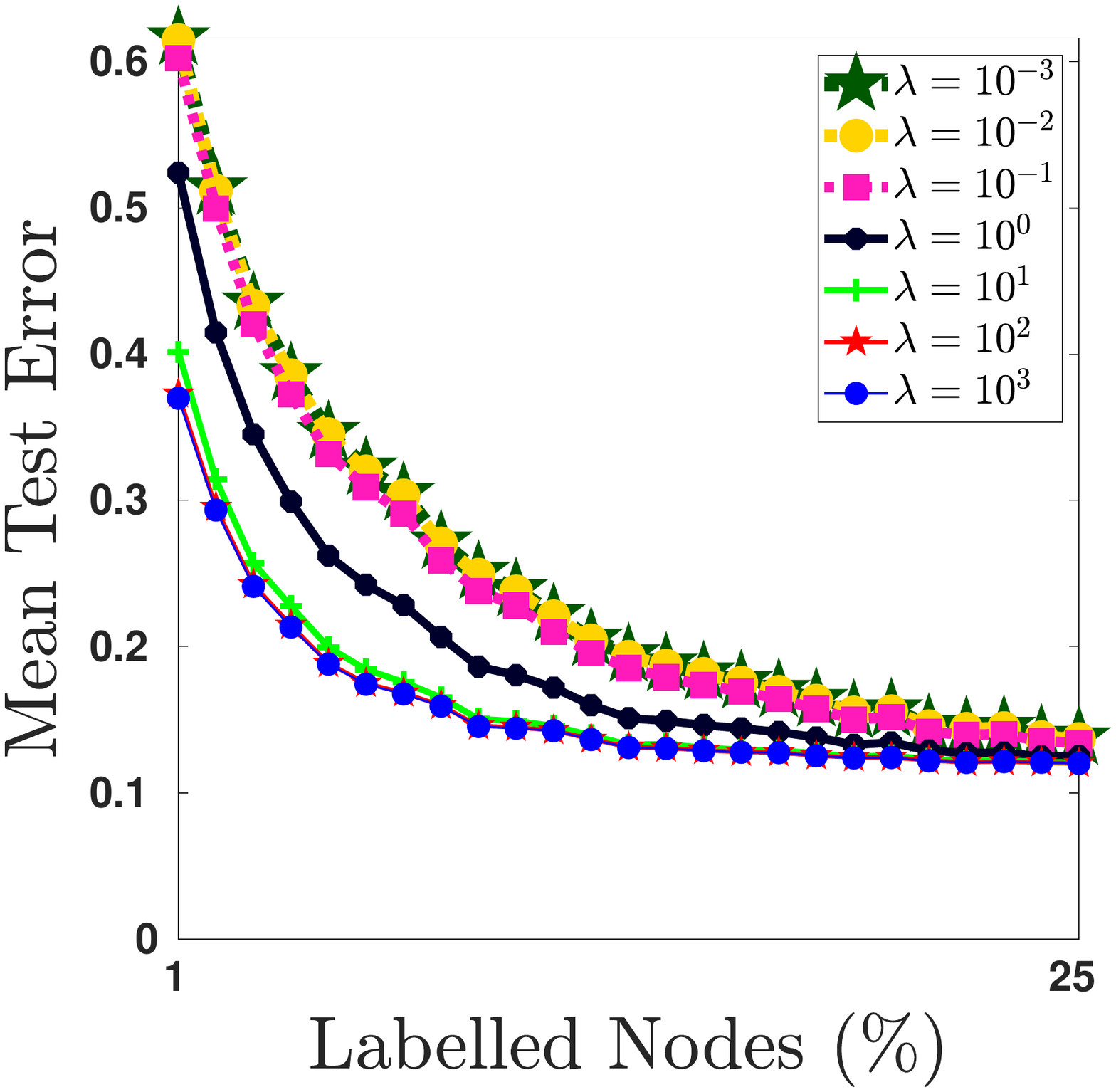}\hspace*{\fill}
  \caption{$L_{-2}$}
  \label{subfig:fig:SBM:diagonal_shift:fix_Wpos:L_{-2}}
  \end{subfigure}%
  \hfill
  \begin{subfigure}[]{0.24\linewidth}
    \addtocounter{subfigure}{-1}
  \renewcommand\thesubfigure{\alph{subfigure}3}
  \includegraphics[width=1\linewidth, clip,trim=110 40 155 40]{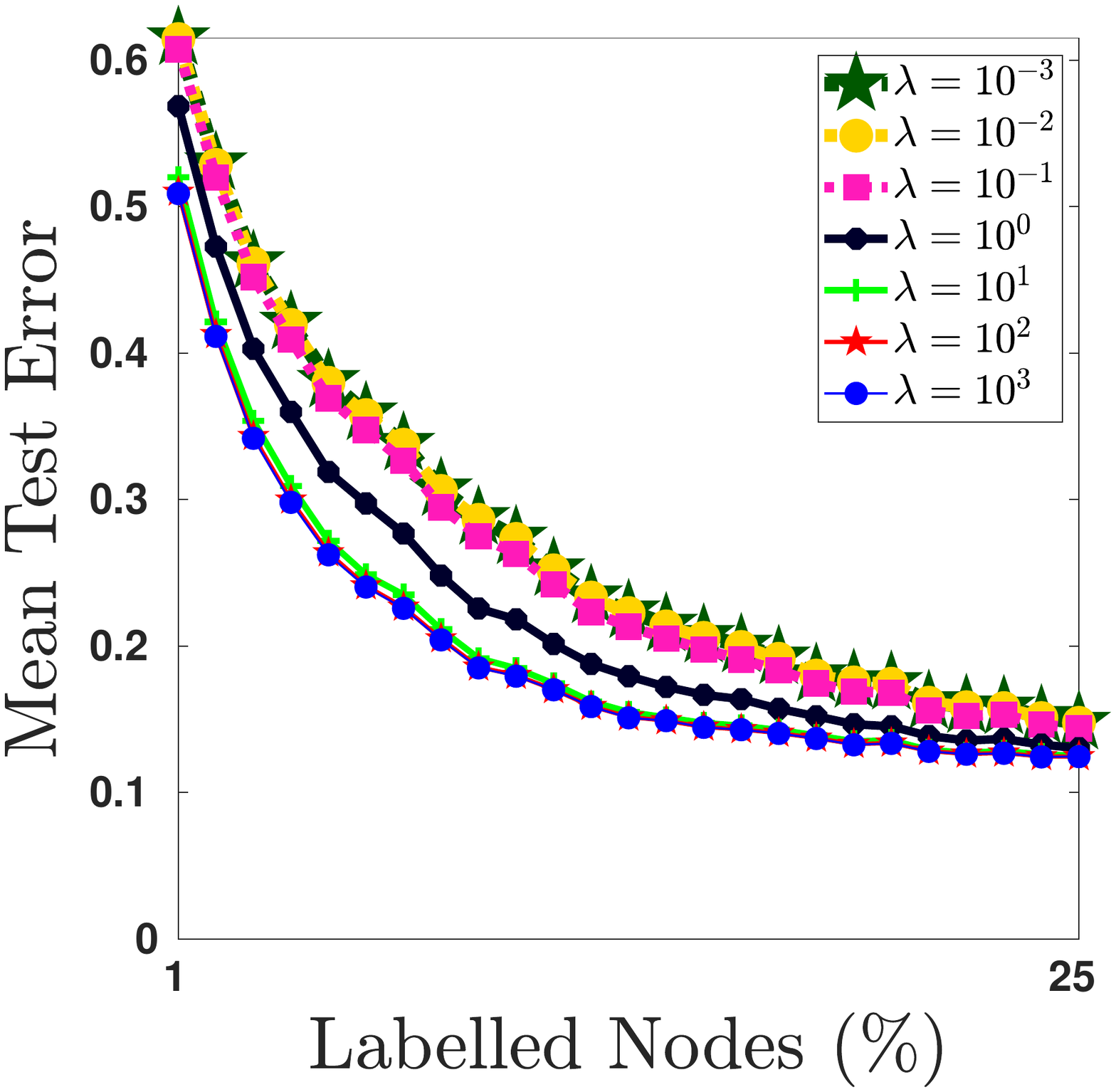}\hspace*{\fill}
  \caption{$L_{-5}$}
  \label{subfig:fig:SBM:diagonal_shift:fix_Wpos:L_{-5}}
  \end{subfigure}%
  \hfill
  \begin{subfigure}[]{0.24\linewidth}
    \addtocounter{subfigure}{-1}
  \renewcommand\thesubfigure{\alph{subfigure}4}
  \includegraphics[width=1\linewidth, clip,trim=110 40 155 40]{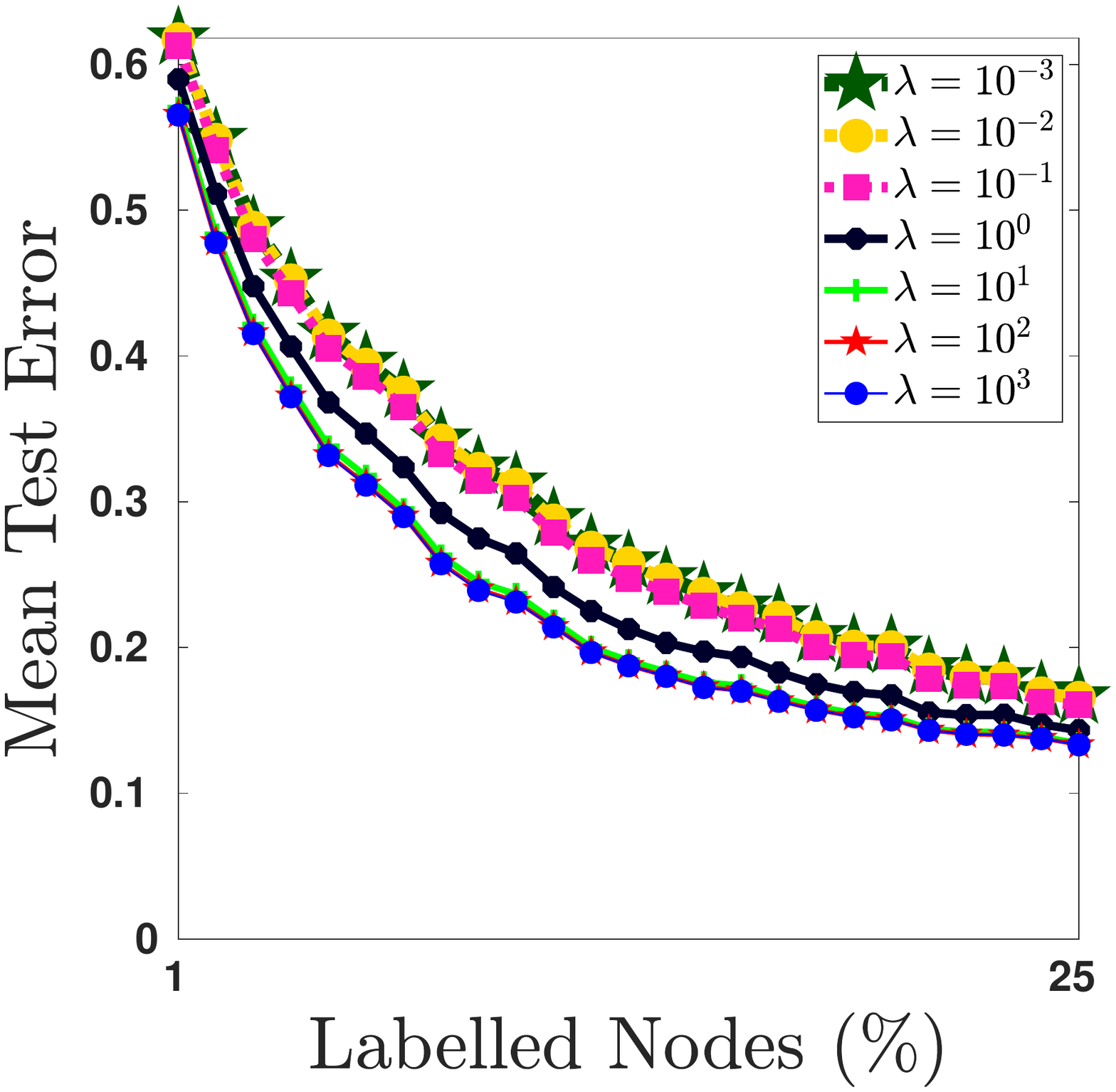}\hspace*{\fill}
  \caption{$L_{-10}$}
  \label{subfig:fig:SBM:diagonal_shift:fix_Wpos:L_{-10}}
  \end{subfigure}
  \setcounter{subfigure}{4}
  \caption{Dataset: UCI}
\end{subfigure}  
\end{figure*}
\newpage\clearpage
\begin{figure}[H]
\ContinuedFloat
\centering
\vskip.0em
%
\begin{subfigure}[t]{1\linewidth}
  \begin{subfigure}[]{0.24\linewidth}
  \renewcommand\thesubfigure{\alph{subfigure}1}
  \includegraphics[width=1\linewidth, clip,trim=110 40 155 40]{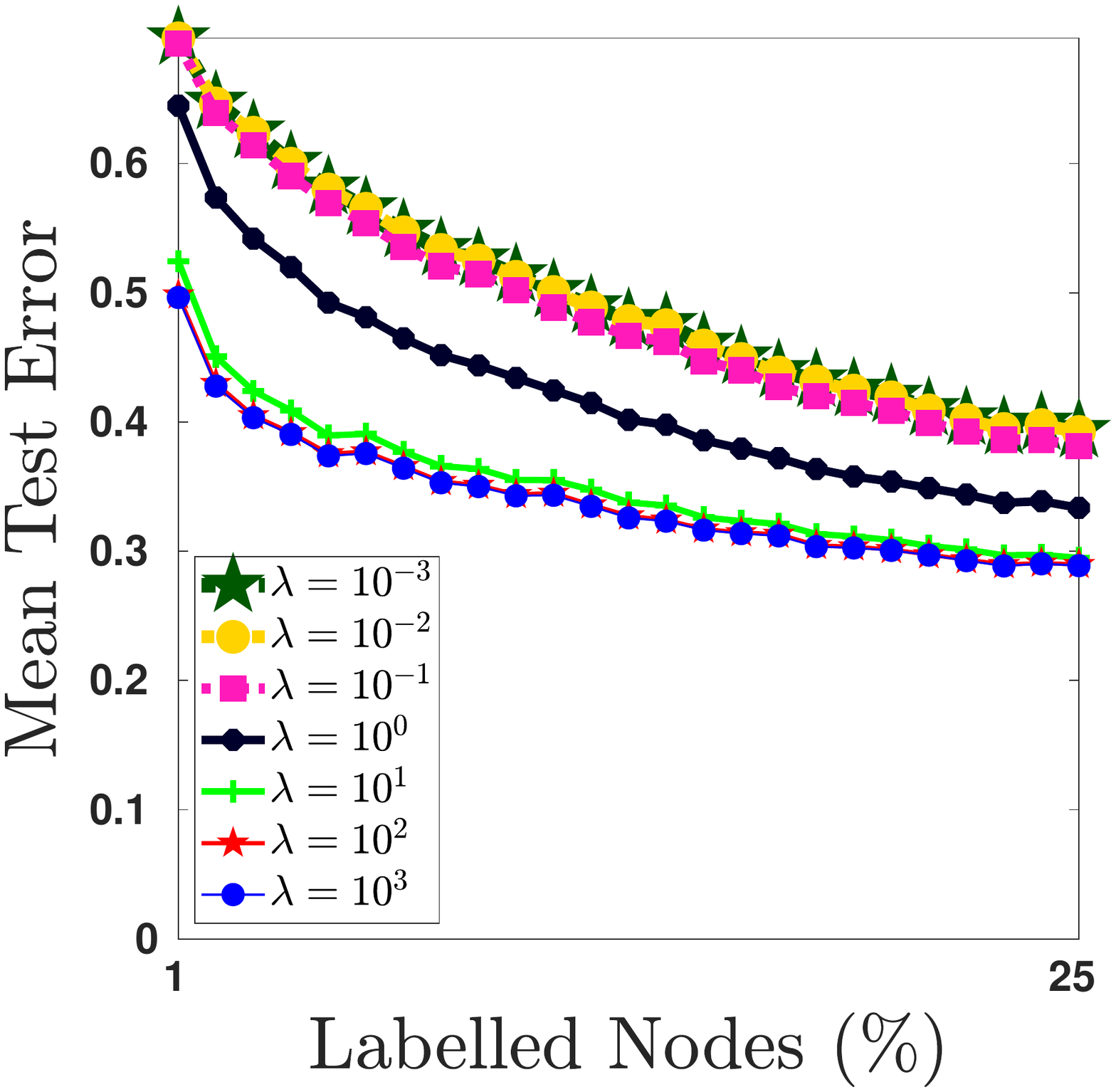}\hspace*{\fill}
  \caption{$L_{-1}$}
  \label{subfig:fig:SBM:diagonal_shift:fix_Wpos:L_{-1}}
  \end{subfigure}%
  \hfill
  \begin{subfigure}[]{0.24\linewidth}
  \addtocounter{subfigure}{-1}
  \renewcommand\thesubfigure{\alph{subfigure}2}
  \includegraphics[width=1\linewidth, clip,trim=110 40 155 40]{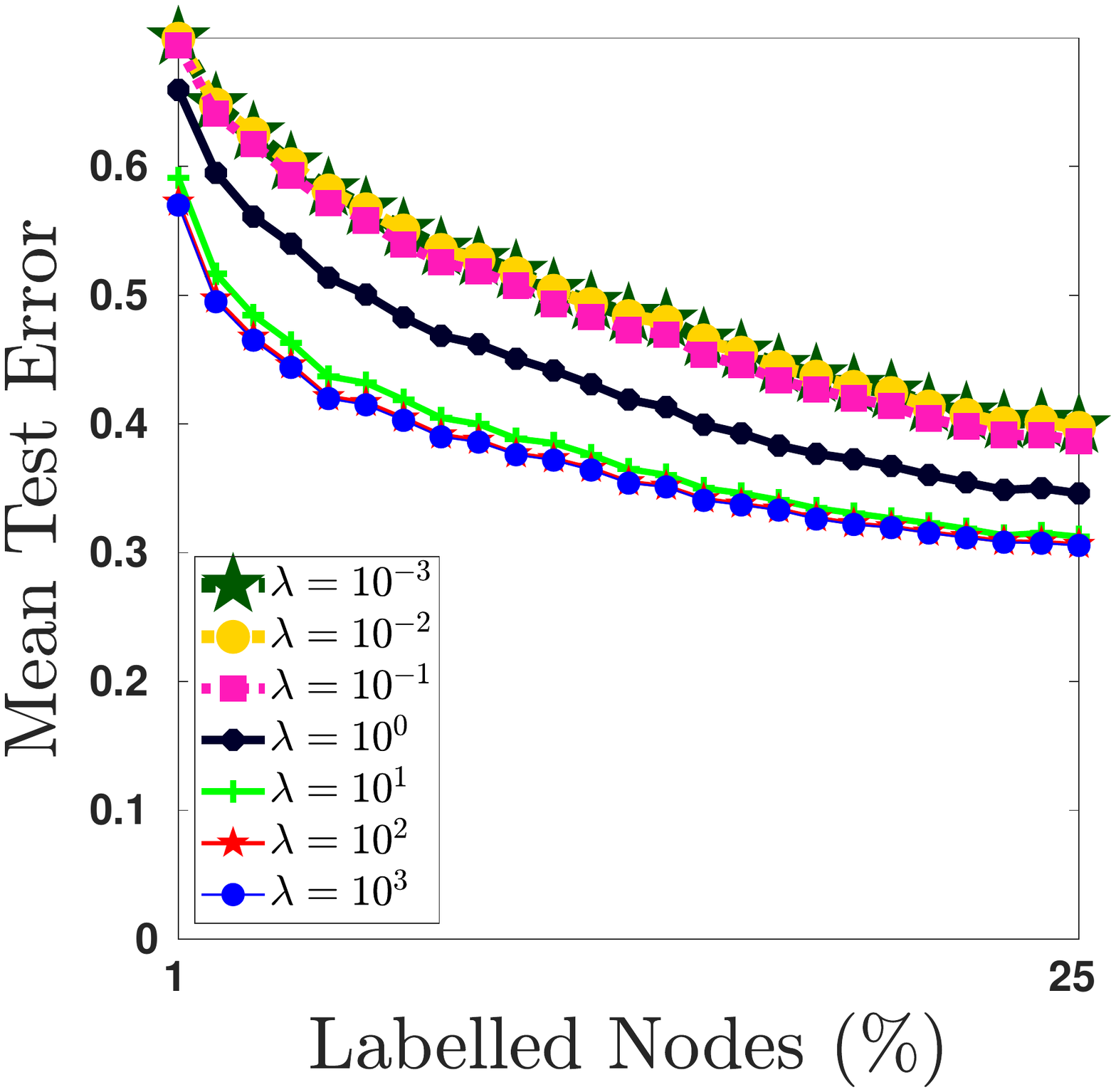}\hspace*{\fill}
  \caption{$L_{-2}$}
  \label{subfig:fig:SBM:diagonal_shift:fix_Wpos:L_{-2}}
  \end{subfigure}%
  \hfill
  \begin{subfigure}[]{0.24\linewidth}
    \addtocounter{subfigure}{-1}
  \renewcommand\thesubfigure{\alph{subfigure}3}
  \includegraphics[width=1\linewidth, clip,trim=110 40 155 40]{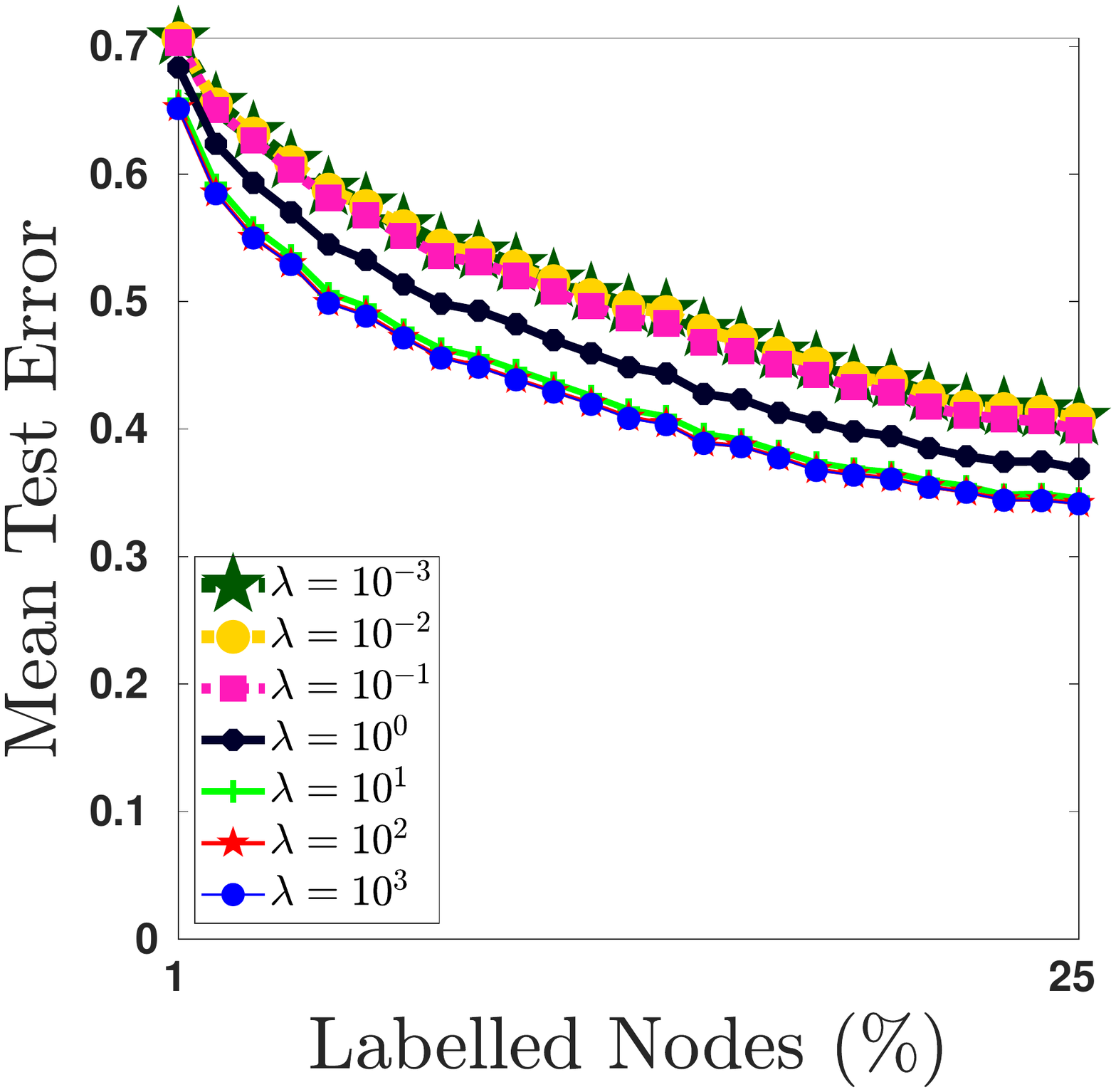}\hspace*{\fill}
  \caption{$L_{-5}$}
  \label{subfig:fig:SBM:diagonal_shift:fix_Wpos:L_{-5}}
  \end{subfigure}%
  \hfill
  \begin{subfigure}[]{0.24\linewidth}
    \addtocounter{subfigure}{-1}
  \renewcommand\thesubfigure{\alph{subfigure}4}
  \includegraphics[width=1\linewidth, clip,trim=110 40 155 40]{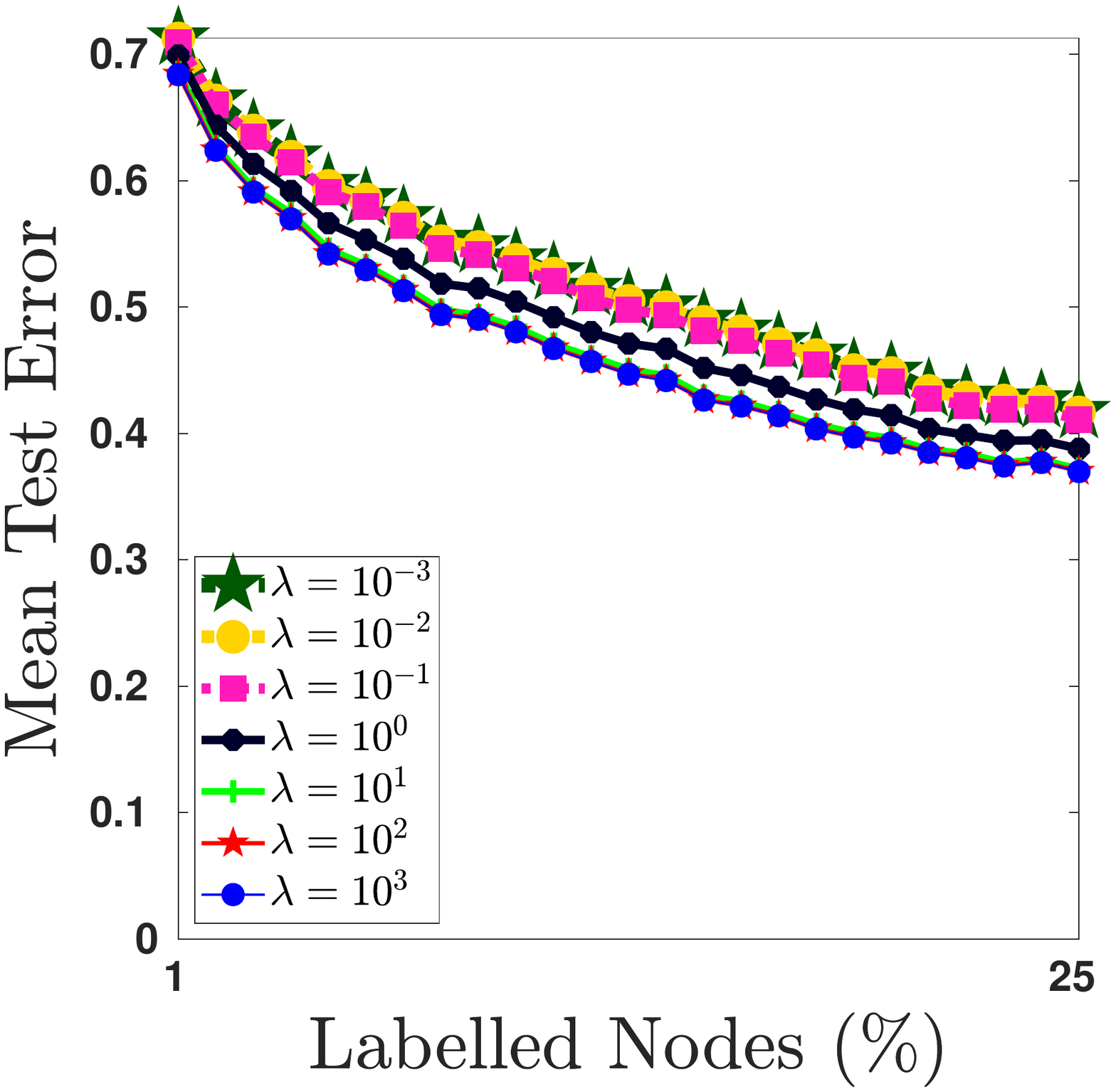}\hspace*{\fill}
  \caption{$L_{-10}$}
  \label{subfig:fig:SBM:diagonal_shift:fix_Wpos:L_{-10}}
  \end{subfigure}
  \setcounter{subfigure}{5}
  \caption{Dataset: Citeseer}
\begin{subfigure}[t]{1\linewidth}
  \begin{subfigure}[]{0.24\linewidth}
  \renewcommand\thesubfigure{\alph{subfigure}1}
  \includegraphics[width=1\linewidth, clip,trim=110 40 155 40]{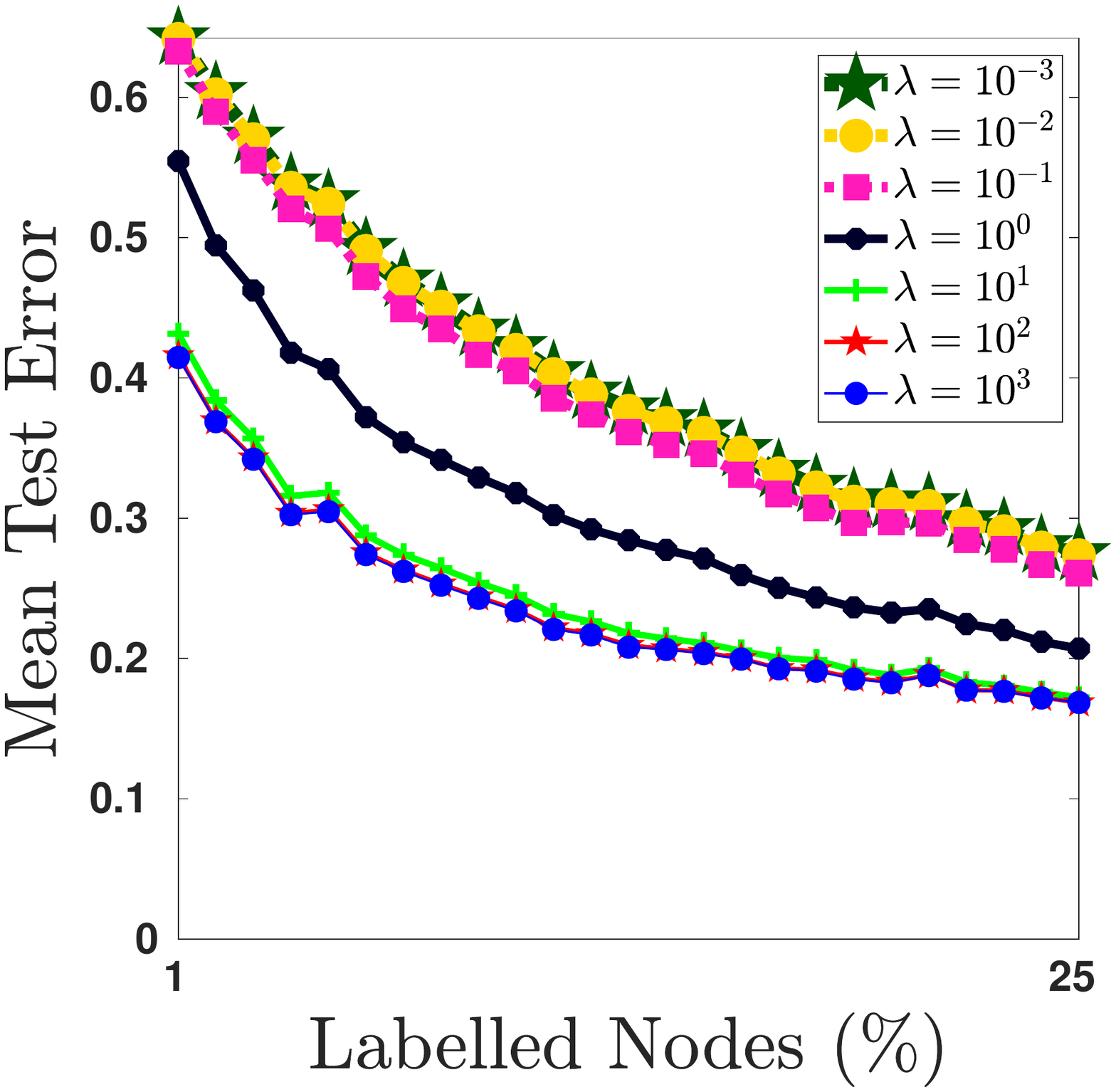}\hspace*{\fill}
  \caption{$L_{-1}$}
  \label{subfig:fig:SBM:diagonal_shift:fix_Wpos:L_{-1}}
  \end{subfigure}%
  \hfill
  \begin{subfigure}[]{0.24\linewidth}
  \addtocounter{subfigure}{-1}
  \renewcommand\thesubfigure{\alph{subfigure}2}
  \includegraphics[width=1\linewidth, clip,trim=110 40 155 40]{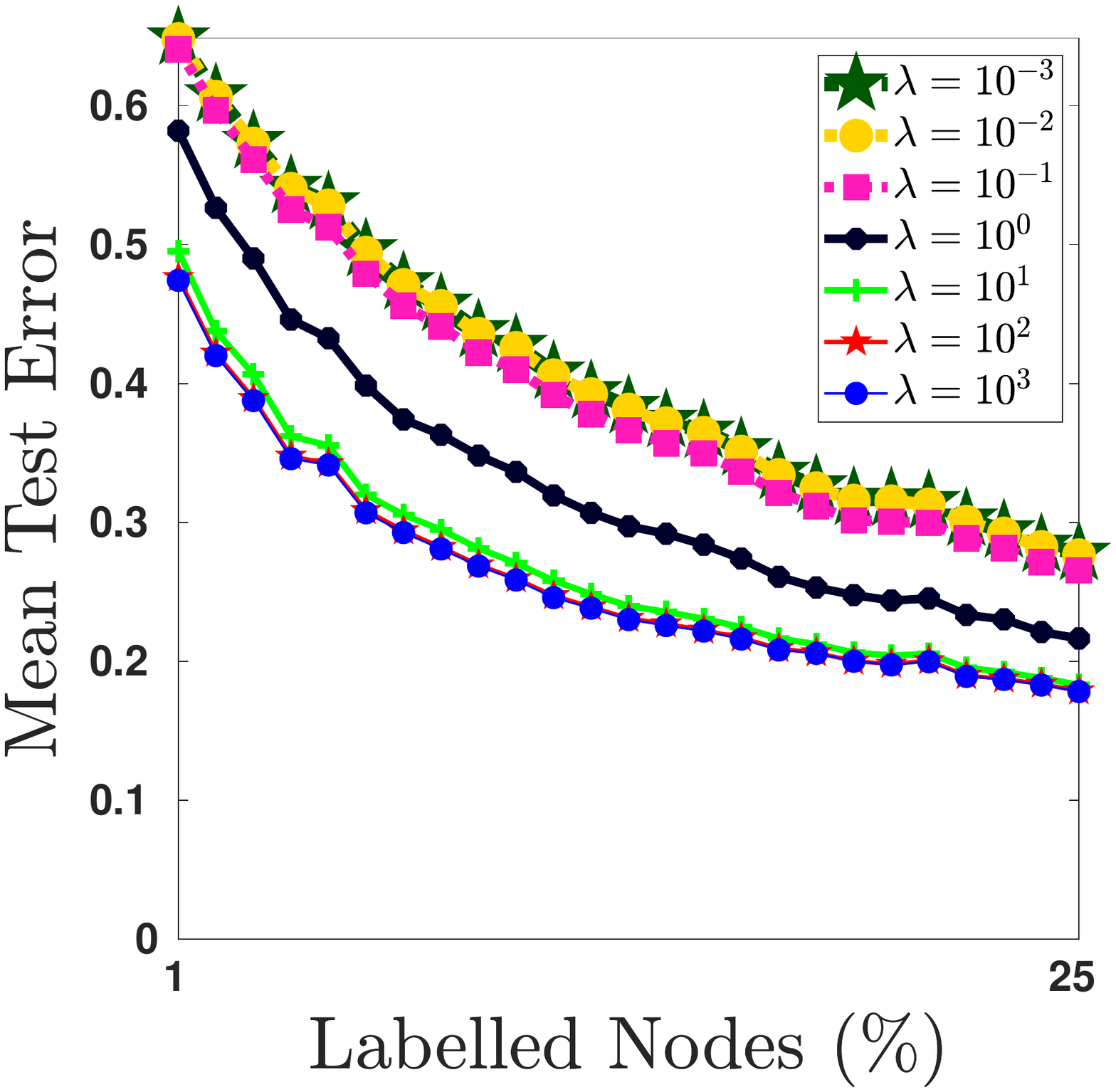}\hspace*{\fill}
  \caption{$L_{-2}$}
  \label{subfig:fig:SBM:diagonal_shift:fix_Wpos:L_{-2}}
  \end{subfigure}%
  \hfill
  \begin{subfigure}[]{0.24\linewidth}
    \addtocounter{subfigure}{-1}
  \renewcommand\thesubfigure{\alph{subfigure}3}
  \includegraphics[width=1\linewidth, clip,trim=110 40 155 40]{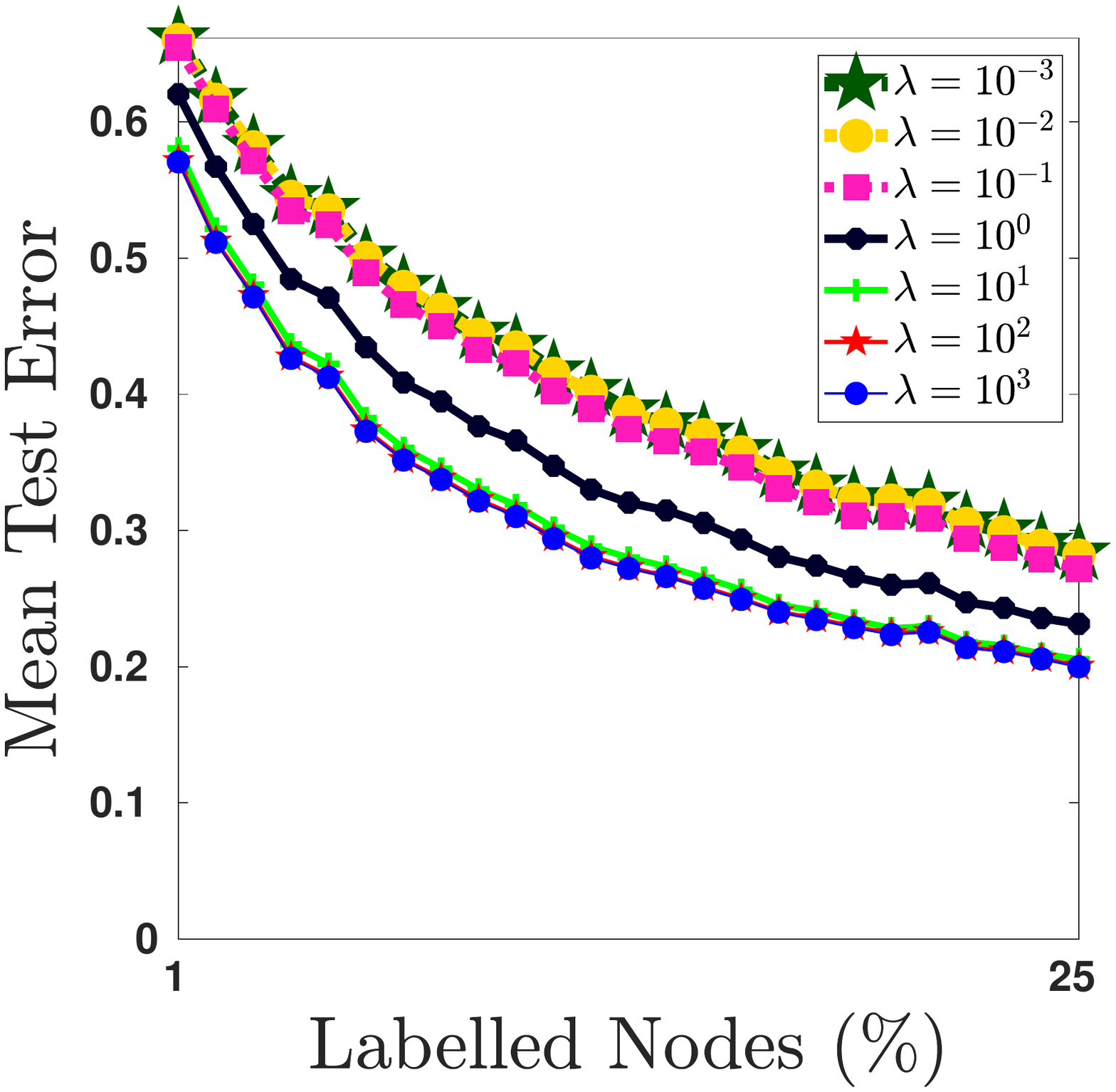}\hspace*{\fill}
  \caption{$L_{-5}$}
  \label{subfig:fig:SBM:diagonal_shift:fix_Wpos:L_{-5}}
  \end{subfigure}%
  \hfill
  \begin{subfigure}[]{0.24\linewidth}
    \addtocounter{subfigure}{-1}
  \renewcommand\thesubfigure{\alph{subfigure}4}
  \includegraphics[width=1\linewidth, clip,trim=110 40 155 40]{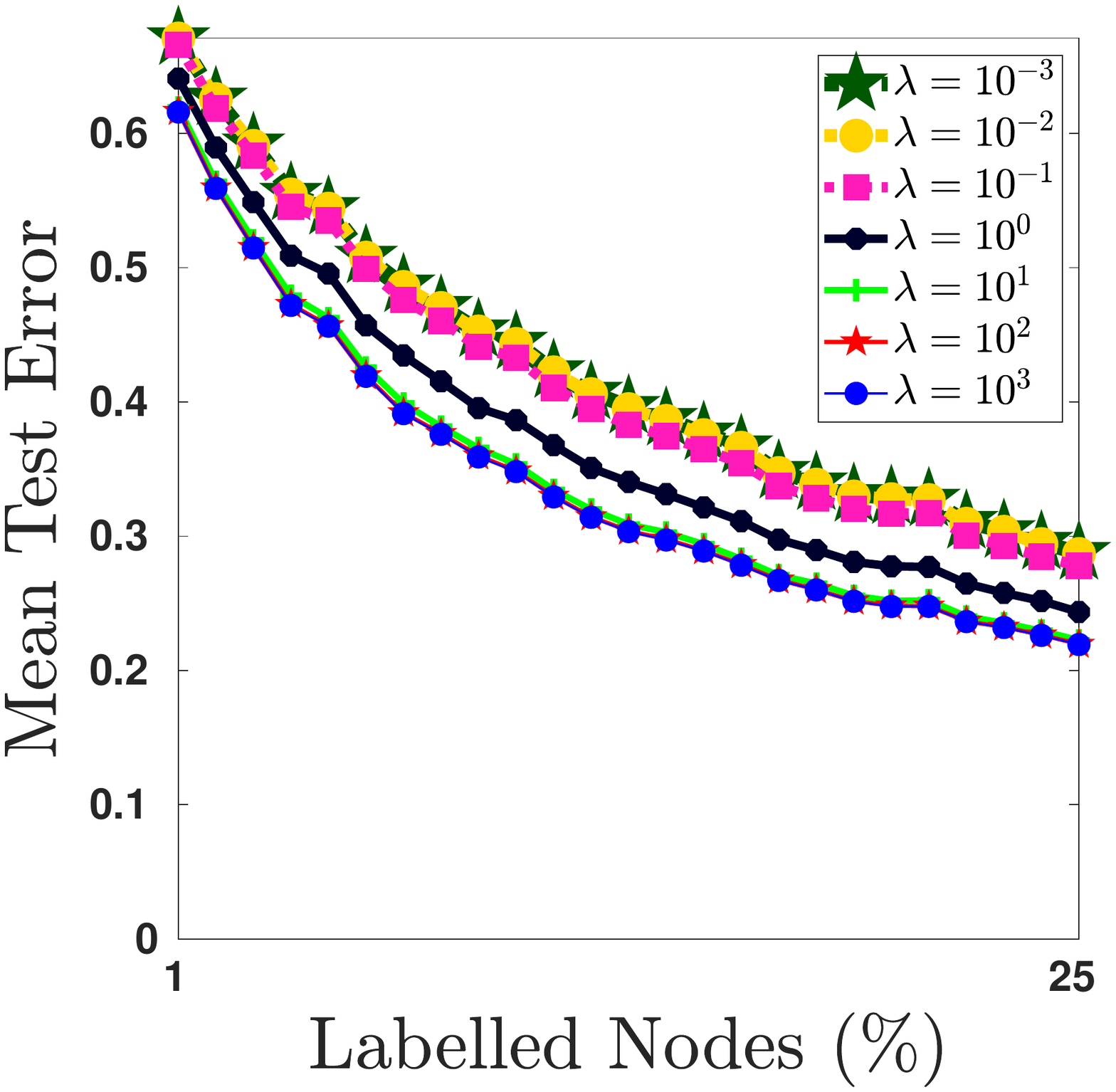}\hspace*{\fill}
  \caption{$L_{-10}$}
  \label{subfig:fig:SBM:diagonal_shift:fix_Wpos:L_{-10}}
  \end{subfigure}
  \setcounter{subfigure}{6}
  \caption{Dataset: Cora}
\end{subfigure}  
\end{subfigure}  
\begin{subfigure}[t]{1\linewidth}
  \begin{subfigure}[]{0.24\linewidth}
  \renewcommand\thesubfigure{\alph{subfigure}1}
  \includegraphics[width=1\linewidth, clip,trim=110 40 155 40]{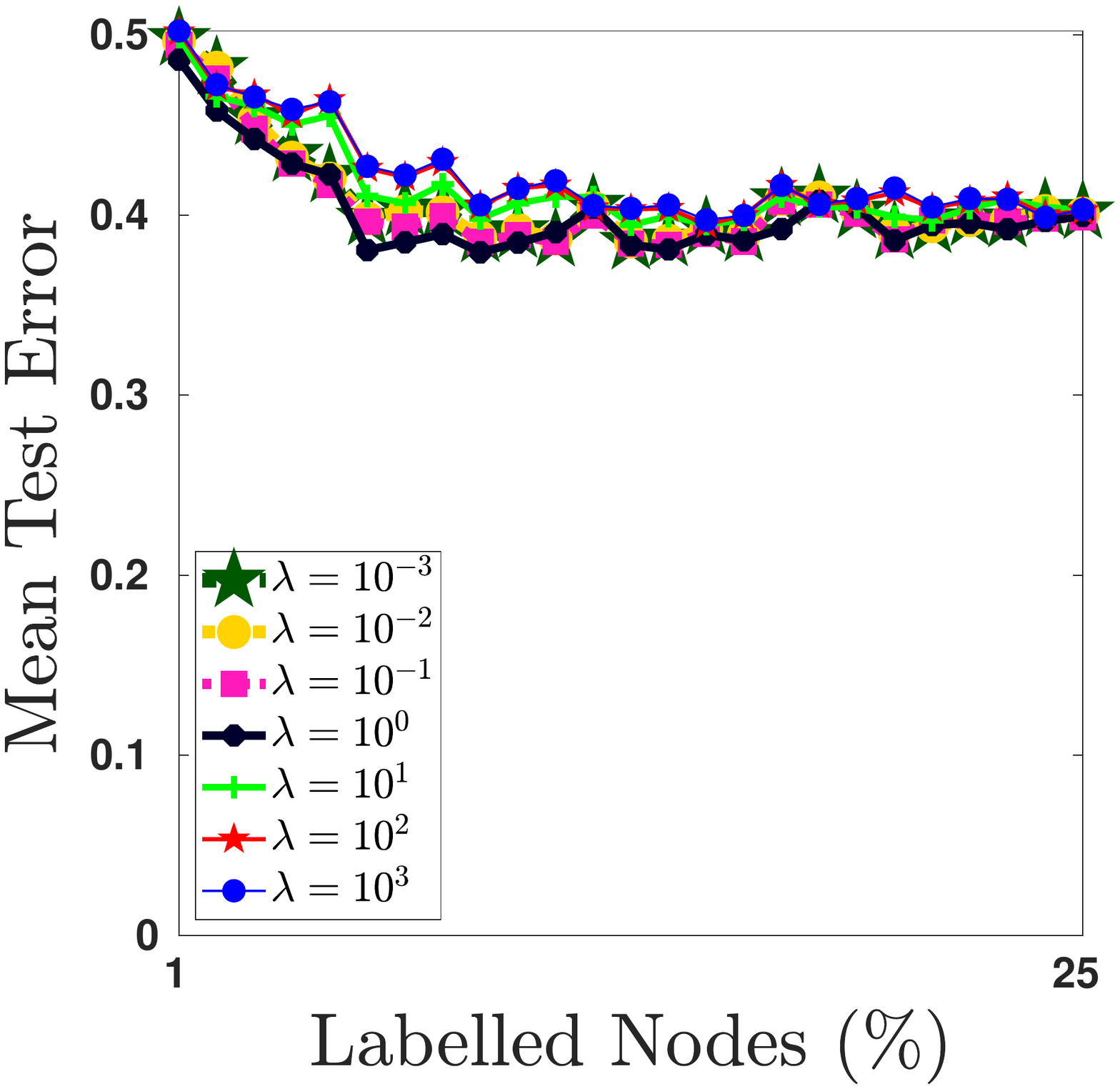}\hspace*{\fill}
  \caption{$L_{-1}$}
  \label{subfig:fig:SBM:diagonal_shift:fix_Wpos:L_{-1}}
  \end{subfigure}%
  \hfill
  \begin{subfigure}[]{0.24\linewidth}
  \addtocounter{subfigure}{-1}
  \renewcommand\thesubfigure{\alph{subfigure}2}
  \includegraphics[width=1\linewidth, clip,trim=110 40 155 40]{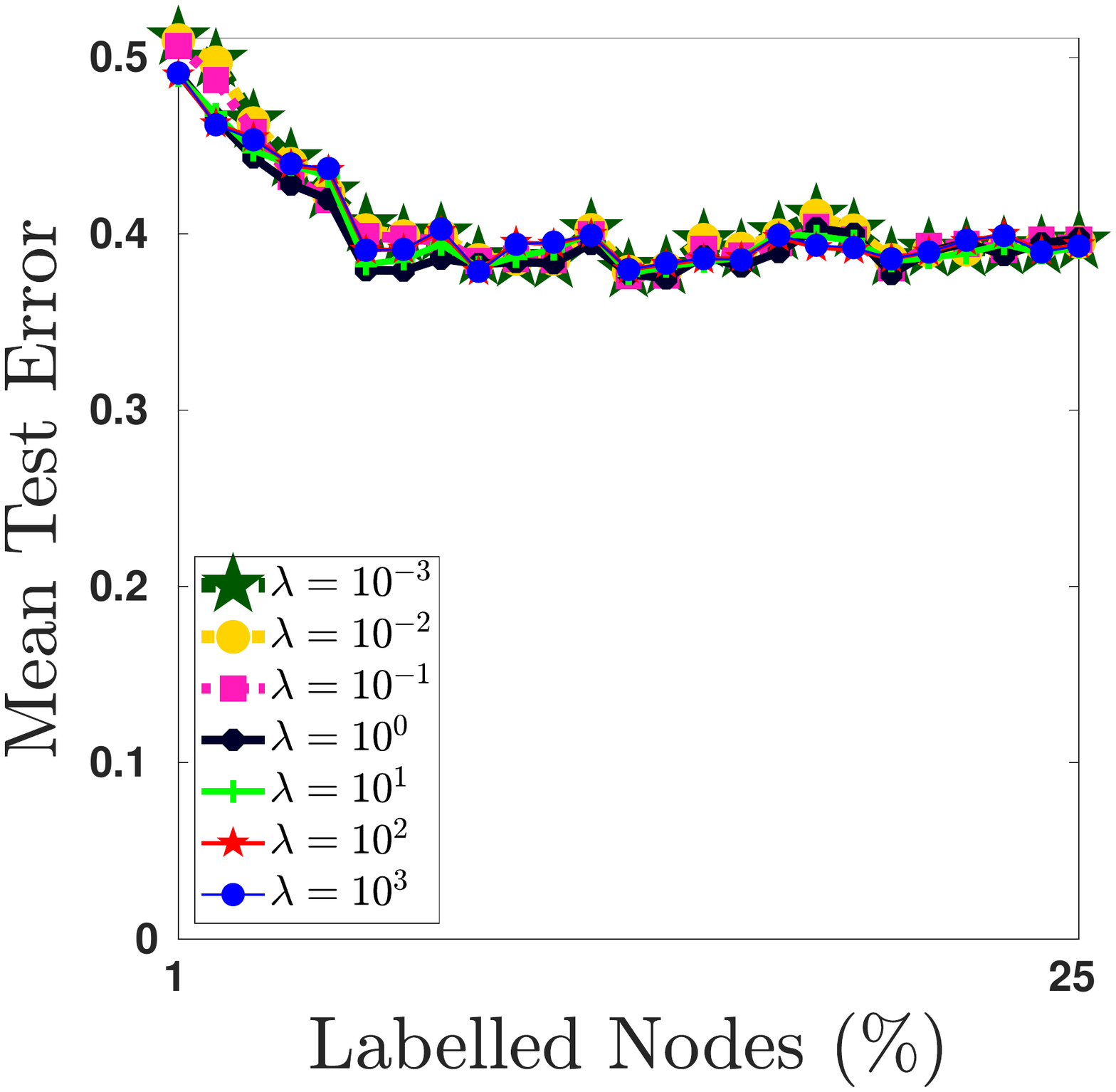}\hspace*{\fill}
  \caption{$L_{-2}$}
  \label{subfig:fig:SBM:diagonal_shift:fix_Wpos:L_{-2}}
  \end{subfigure}%
  \hfill
  \begin{subfigure}[]{0.24\linewidth}
    \addtocounter{subfigure}{-1}
  \renewcommand\thesubfigure{\alph{subfigure}3}
  \includegraphics[width=1\linewidth, clip,trim=110 40 155 40]{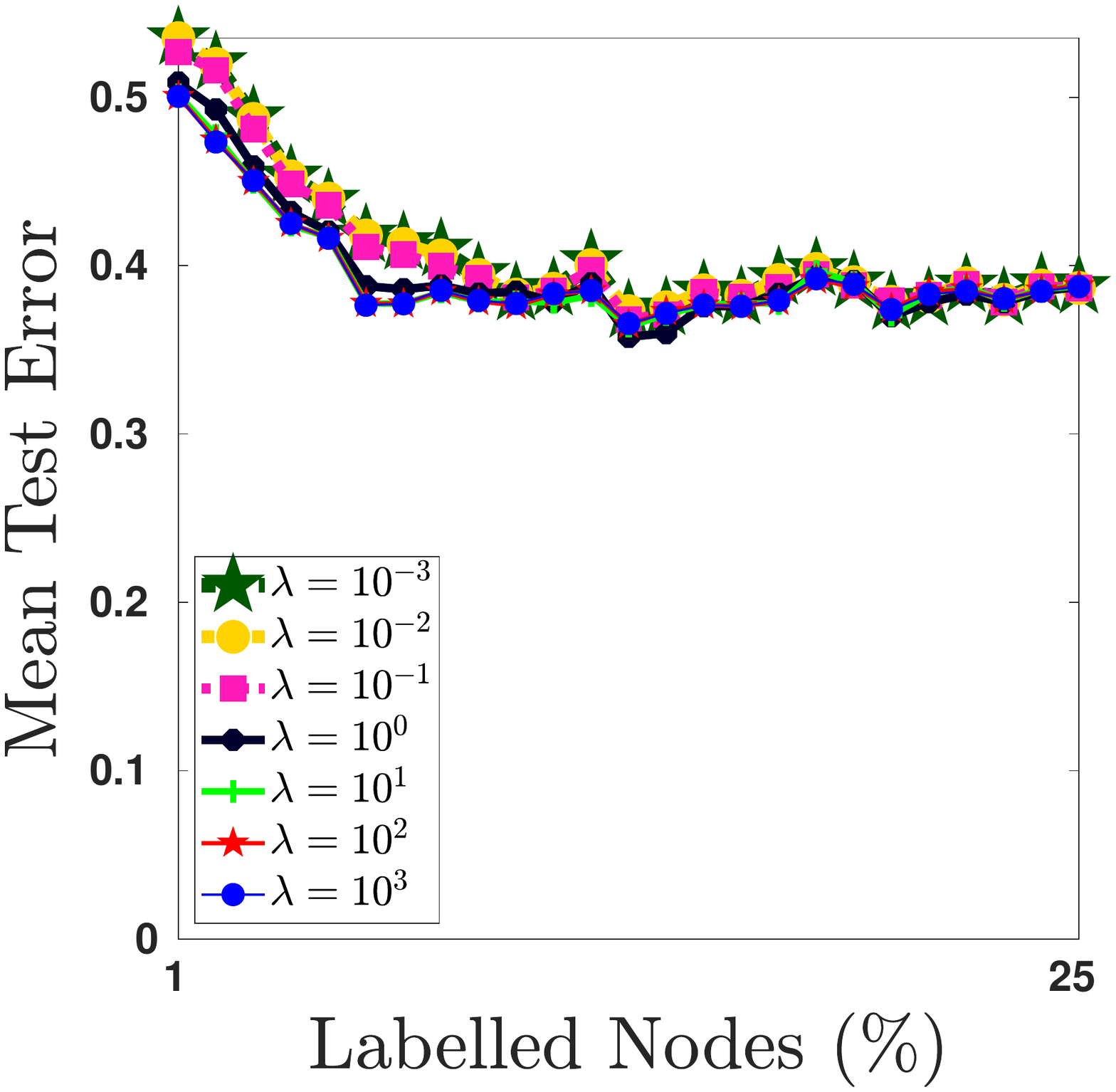}\hspace*{\fill}
  \caption{$L_{-5}$}
  \label{subfig:fig:SBM:diagonal_shift:fix_Wpos:L_{-5}}
  \end{subfigure}%
  \hfill
  \begin{subfigure}[]{0.24\linewidth}
    \addtocounter{subfigure}{-1}
  \renewcommand\thesubfigure{\alph{subfigure}4}
  \includegraphics[width=1\linewidth, clip,trim=110 40 155 40]{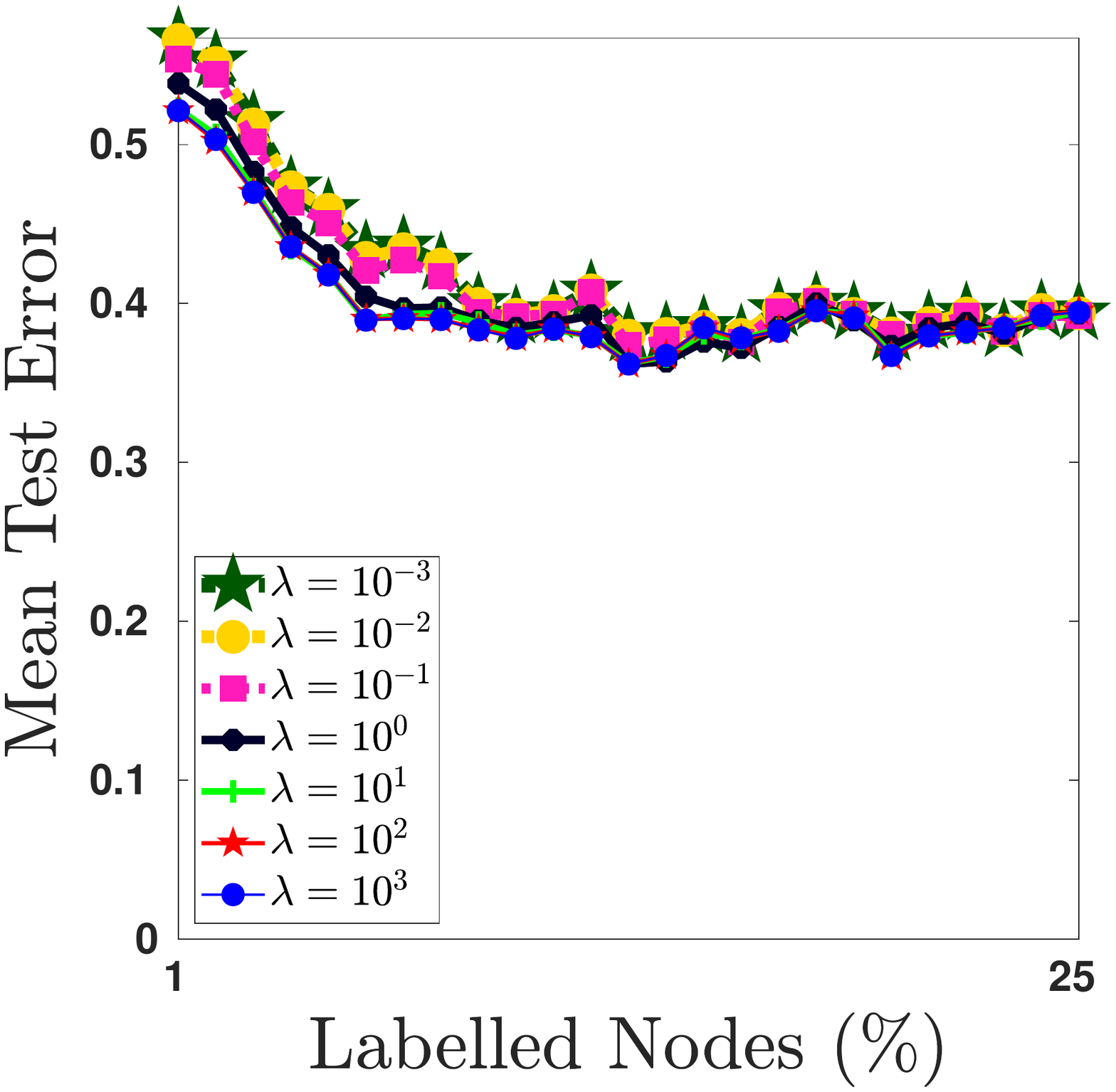}\hspace*{\fill}
  \caption{$L_{-10}$}
  \label{subfig:fig:SBM:diagonal_shift:fix_Wpos:L_{-10}}
  \end{subfigure}
  \setcounter{subfigure}{7}
  \caption{Dataset: WebKB}
\end{subfigure}  
%
\caption{
Mean test classification error on real world datasets for different values of $\lambda$. Details in Sec.~\ref{sec:OnLambda}.
}
\label{fig:RealDataSets:lambda_effect}
\end{figure}

\fi
\end{document}